\documentclass{article}
\usepackage[margin=1in]{geometry}
\usepackage{amsmath, amssymb, amsfonts, amsbsy, latexsym, epsfig}
\usepackage{hyperref}
\usepackage{cleveref}
\crefname{equation}{}{}
\crefname{enumi}{}{} 
\crefname{assumption}{Assumption}{Assumptions}
\Crefname{assumption}{Assumption}{Assumptions}
\DeclareMathOperator{\Tr}{tr}
\usepackage{amsthm}
\usepackage{MnSymbol}
\usepackage{url}
\usepackage{xcolor}
\usepackage{tikz}
\usepackage{graphicx} 
\usepackage{natbib}
\usepackage{booktabs}
\usepackage{bbm}
\usepackage{tablefootnote}
\usepackage{comment}
\usepackage[ruled,vlined]{algorithm2e}
\usepackage{mathtools}
\usepackage{enumitem}
\setlist{topsep=0pt, leftmargin=*}
\allowdisplaybreaks

\newcommand{\ip}[2]{\left\langle #1, #2 \right\rangle}

\newcommand{\cO}{\mathcal{O}}
\newcommand{\cN}{\mathcal{N}}

\newcommand{\cL}{\mathcal{L}}

\newcommand{\cM}{\mathcal{M}}
\newcommand{\cB}{\mathcal{B}}
\newcommand{\cA}{\mathcal{A}}

\newcommand{\PP}{\mathbb{P}}
\newcommand{\EE}{\mathbb{E}}
\newcommand{\RR}{\mathbb{R}}

\newcommand{\bbone}{\mathbbm{1}}

\newcommand{\cLwd}{\mathcal{L}_{\lambda}}
\newcommand{\bzero}{\textbf{0}}

\newcommand{\Leb}{\mathrm{Leb}}

\newcommand{\ve}{\mathrm{vec}}
\newcommand{\var}{\mathrm{Var}}

\newcommand{\diam}{\mathrm{diam}}

\newtheorem{lemma}{Lemma}
\newtheorem{definition}{Definition}
\newtheorem{theorem}{Theorem}

\newtheorem{assumption}{Assumption}

\newtheorem{corollary}{Corollary}
\newtheorem{proposition}{Proposition}

\title{A Theoretical Analysis of Generalization Dynamics in Neural Networks under Gradient Descent with Weight Decay
}
\author{Yuqing Wang\\[0.2em]
Johns Hopkins University\\[0.2em]
\texttt{ywan1050@jh.edu}
\and
Ioannis G. Kevrekidis\\[0.2em]
Johns Hopkins University\\[0.2em]
\texttt{yannisk@jhu.edu}
\and
Mikhail Belkin\\[0.2em]
University of California San Diego\\[0.2em]
\texttt{mbelkin@ucsd.edu}}
\date{}

\begin{document}

\maketitle

\begin{abstract}
Understanding generalization remains a central challenge in machine learning because it requires jointly considering data, architecture, and training dynamics. In this paper, we develop a theoretical framework that characterizes how these factors jointly shape generalization performance throughout training. More precisely, we study a broad class of neural networks trained under the $\ell^2$ loss by gradient descent (GD) with weight decay, and prove the convergence of GD to a neighbourhood of the global minimizers of the empirical loss. By partitioning the space based on the input data, we then decompose the population error into data error, optimization error, and prediction variation error, and bound them separately. In particular, for the prediction variation error, which measures the oscillations of the learned function, we propose (local) approximate homogeneity and derive explicit cellwise and layerwise bounds for its evolution along the training trajectory. These bounds yield two important implications: a necessary condition of improved generalization explains differences in layerwise generalization behavior; a sufficient condition describes delayed generalization and provides a theoretical characterization of grokking.

\end{abstract}

\begingroup
\tableofcontents
\endgroup
\clearpage

\section{Introduction}

Generalization refers to the ability of a learned model to perform well on unseen data and is a central goal in machine learning. Classical theories study generalization through the complexity of hypothesis classes, including uniform convergence, Rademacher complexity, norm-based capacity measures, and algorithmic stability \citep{vapnik1971uniform,dudley1967sizes,bartlett2002rademacher,bartlett1998sample,neyshabur2015norm,bartlett2017spectrally,zhang2017understanding,nagarajan2019uniform,bousquet2002stability,hardt2016train,Bou02,Gol17}. Other major perspectives include PAC-Bayes and compression bounds \citep{mcallester1999pac,dziugaite2017computing,Ney17b,Aro18,Per20}, neural tangent kernel and kernel-regime analyses \citep{jacot2018neural,du2019gradient,Bie19,Bor20,Cao19b,arora2019fine,lee2019wide}, mean-field and random-feature theories \citep{Mei18,Sir18,Chi18,Mei19b,rahimi2008uniform,Mei19}, information-theoretic perspective \citep{Tis15,Shw17,Xu17}, and analytical approaches \citep{kawaguchi2018generalization}. In modern overparameterized learning, however, the hypothesis class alone is often insufficient to distinguish solutions with different test performance. This has shifted attention toward the interaction among the data, the selected solution, and more importantly, the training process that produces it.

Training dynamics is closely tied to generalization performance, and many empirical findings demonstrate its diverse effects. One of the best known phenomena is grokking, which exhibits a prolonged period of poor test performance after the model has fit the training data, followed by a delayed and often abrupt improvement in generalization \citep{power2022grokking,chughtai2023toy,tan2023understanding,notsawo2023predicting,fan2024deep,blanc2020implicit,humayun2024deep,merrill2023tale,zhu2024critical,wang2024grokking,xulet}. Beyond such temporal behavior at the network level, generalization performance can also differ across layers. In particular, neural networks evolve unevenly across depth: shallower layers may stabilize earlier and learn simpler structures, whereas deeper layers may evolve over longer time scales towards more complex structures \citep{raghu2017svcca,chen2023which}. These observations motivate a finer analysis of generalization that incorporates both the training trajectory and the architecture, rather than focusing solely on the final learned model.

Much theoretical work seeks to connect generalization with training dynamics. Several studies characterize grokking in simplified models or specific regimes \citep{lyu2023dichotomy,mohamadi2024you,kumar2024grokking,xu2023benign,boursier2025theoretical,xu2026grok}. Stability-based analysis links learning rates during training to the generalization of two-layer ReLU networks \citep{qiao2024stable}. Finite-step dynamics have also been studied for single-index models from a high-dimensional perspective \citep{ba2022high}. Nevertheless, these results typically rely on specialized architectures, data models, or limiting descriptions of training, leaving the evolution of the population error along a general neural-network trajectory largely unresolved. This motivates our first question:
\begin{center}
    \emph{How can the dynamics of generalization performance, including its layerwise behavior, be characterized\\ for a broad class of neural networks and tasks?}
\end{center}

While the first question concerns how generalization evolves during training, a complementary line of research asks which properties of the data, architecture, and learned solution make good generalization possible. One insight comes from benign overfitting, which demonstrates that a model can interpolate the training data perfectly while still achieving strong test performance \citep{belkin2019reconciling,nakkiran2020deep,bartlett2020benign,tsigler2023benign,hastie2022surprises}. 

The geometry of the learned solution provides another prominent perspective. Flat interpolators are often associated with strong test performance \citep{hochreiter,keskar2017on,Jiang*2020Fantastic,chiang2023loss,DBLP:conf/iclr/ForetKMN21}. However, flatness alone cannot characterize generalization: flat interpolators may still perform poorly on test data, while sharp minimizers may generalize well \citep{pmlr-v70-dinh17b,pmlr-v202-andriushchenko23a,wen2023how,schliserman2025flatminimageneralizationinsights}.

Architectural properties add another layer to this picture. For single-index and multi-index learning problems, quantities such as information exponent \citep{arous2021online}, generative exponent \citep{damian2024computational}, and noise sensitivity exponent \citep{defilippis2026noise} describe how the activation function influences statistical or computational learnability. These results demonstrate that architectural design can directly affect generalization, but a systematic understanding of how architecture interacts with data quality and training dynamics remains missing. This leads to our second question:
\begin{center}
    \emph{What properties can guarantee good generalization?}
\end{center}

Motivated by these questions, we develop a theoretical framework for neural network generalization dynamics that integrates the roles of data, architecture, and optimization. The framework shows that good generalization requires the architecture and data to induce sufficiently strong (local) approximate homogeneity, the optimization trajectory to enter a neighbourhood of the set of global minimizers, and training to remain in this region for sufficiently long. We consider the broad family of neural networks studied in \citet{wang2026convergencegradientdescentgeneral}, trained under the $\ell^2$ loss by gradient descent with weight decay. Our main contributions are as follows.
\begin{itemize}
    \item For the optimization dynamics, we establish that the regularized objective decreases along gradient descent under mild assumptions (Theorem~\ref{thm:general_convergence_loss_decay}). We further show that gradient descent enters a neighbourhood of the set of global minimizers of the empirical loss, where both the empirical loss and its gradient are explicitly controlled by the learning rate (Theorem~\ref{thm:convergence_to_global_min}).
    \item For the generalization dynamics, we introduce a partition of the input domain adapted to the training samples (Definition~\ref{def:partition_pi_null}) and decompose the population error into three components:
    \begin{itemize}
        \item Data error $E_{\rm data}$: The effect of data is represented by the data error $E_{\rm data}$, which is bounded by a stochastic term determined by label noise and population cell masses, and a deterministic term measuring how well the sampled inputs approximate the ground truth function within the partition cells (Theorem~\ref{thm:gen_data_error}).
        \item Optimization error $E_{\rm opt}$: The optimization error $E_{\rm opt}$ connects the empirical loss guarantee with a population-weighted error bound (Corollary~\ref{cor:gen_optimization_error}). Its dependence on the cell masses also identifies when loss reweighting can remove the imbalance between empirical and population weights.
        \item Prediction variation error $E_{\rm PV}$: The effects of architecture and training dynamics are captured by the prediction variation error $E_{\rm PV}$. This error is controlled by layerwise quantities whose evolution depends on (local) approximate homogeneity (Definitions~\ref{def:approx_homogeneity} and~\ref{def:local_approx_homogeneity}), the gradient of empirical loss, and the regularized objective (Theorem~\ref{thm:gen_NN_error}). The corresponding cellwise and layerwise prediction variations provide a further characterization of the dynamics of each layer (Lemma~\ref{lem:V_decay}).
    \end{itemize}
    \item This framework yields the following layerwise and temporal consequences for generalization dynamics:
    \begin{itemize}
    \item A necessary condition for the layerwise bound to guarantee a relative decay helps explain why different layers may begin to generalize at different times and learn components of the target function with different levels of complexity (Corollary~\ref{cor:necessary_condi_generalization}).
    \item A sufficient condition characterizes delayed generalization as a regime in which the optimization error becomes small before the prediction variation has sufficient time to decay. It also interprets a mechanism of grokking through a small approximate homogeneity error, a small empirical gradient and a small regularized objective in the neighbourhood of global minimizers, and a sufficiently long residence time in that region
     (Corollary~\ref{cor:sufficient_condi_generalization}). 

    \end{itemize}
\end{itemize}

\paragraph{Notations.} We use $\|\cdot\|$ for the Euclidean norm of vectors. For a function $\psi:\RR^m\to\RR^n$, $\nabla\psi$ denotes
the Jacobian with respect to all variables, and $\nabla_w\psi$ denotes the
Jacobian with respect to the variable $w$. For a set $A$, we write $A^o$ for
its interior, $\bar A$ for its closure, $\diam(A)$ for its diameter, and
$\bbone_A$ for its indicator function. We denote by $B_w(r)$ the open Euclidean
ball centered at $w$ with radius $r$. For a measure $\nu$ and a measure $\mu$,
the notation $\nu\ll\mu$ means that $\nu$ is absolutely continuous with respect
to $\mu$, and $\Leb_d$ denotes Lebesgue measure on $\RR^d$. If
$\pi(x)dx$ is a probability measure on $\Omega$, then
$\pi(A)=\int_A\pi(x)dx$ for measurable $A\subseteq\Omega$. For layer-indexed parameters, $\theta_{\ell-1:0}$ denotes the collection $(\theta_0,\cdots,\theta_{\ell-1})$ of parameter blocks before the $\ell$th layer. We use $\cO(\cdot)$, $\Omega(\cdot)$, and $\Theta(\cdot)$ in their
standard asymptotic senses.

\section{Optimization Theory}
\label{sec:optimization_theory}

In this section, we develop the optimization theory underlying our generalization analysis. We first introduce the setting, including the neural network model and the GD dynamics with weight decay, and then recall the general convergence mechanism of \citet{wang2026convergencegradientdescentgeneral}, which we adapt to the regularized objective (Section~\ref{subsec:optimization_preliminaries}). We next establish decay of the regularized objective (Section~\ref{subsec:general_loss_decay}) and show that GD enters a neighbourhood of the set of global minimizers of the empirical loss (Section~\ref{subsec:convergence_global_min}), thus providing the loss and gradient controls used in Section~\ref{sec:generalization_theory}.

We consider the following family of neural networks studied by \citet{wang2026convergencegradientdescentgeneral}:
\begin{align}
\label{eqn:NN}
    u_{0,i}&=x_i,\notag\\
    u_{\ell+1,i}&=u_{\ell,i}+{\varphi}_\ell(\theta_\ell;u_{\ell,i}),\qquad \ell=0,\cdots,L-1,\\
    f(\theta;x_i)&={\varphi}_L(\theta_L;u_{L,i}).\notag
\end{align}
Here $\theta=(\theta_0,\cdots,\theta_L)$ denotes the collection of trainable parameters. For each sample $(x_i,y_i)$, define the $\ell^2$ loss by
\begin{align*}
    l(\theta;x_i,y_i):=\frac{1}{2}\|y_i-f(\theta;x_i)\|^2.
\end{align*}
Then the empirical loss is
\begin{align*}
    \cL(\theta):=\cL(\theta;\{x_i,y_i\}_{i=1}^N)=\frac{1}{N}\sum_{i=1}^N l(\theta;x_i,y_i).
\end{align*}

GD applied to the empirical loss does not generally guarantee convergence to a global minimizer, although such convergence is often needed to obtain strong generalization guarantees. Therefore, we consider the empirical loss with the weight decay term $\frac{\lambda}{2}\|\theta\|^2$ and study GD applied to the resulting regularized objective
\begin{align*}
    \cL_{\lambda}(\theta):=\cL(\theta)+\frac{\lambda}{2}\|\theta\|^2=\frac{1}{N}\sum_{i=1}^N l(\theta;x_i,y_i)+\frac{\lambda}{2}\|\theta\|^2,
\end{align*}
where the weight decay parameter $\lambda> 0$.

The GD update is defined as follows
\begin{align*}
    \theta^{k+1}=\theta^k-\eta\nabla_\theta\cLwd(\theta^k).
\end{align*}
Note that the weight decay term enters the training objective, but is not included in the population error used to evaluate generalization. The generalization error is still computed from the original network output $f(\theta;x)$ and the ground truth function $g$.

\subsection{Preliminaries from general GD convergence}
\label{subsec:optimization_preliminaries}

Before introducing our convergence results, we briefly recall two notions from \citet{wang2026convergencegradientdescentgeneral} that will be adapted to the regularized objective $\cL_\lambda$. These notions are used to establish Lipschitz smoothness along the GD trajectory. The first is \citet[Definition~3]{wang2026convergencegradientdescentgeneral}.

\begin{definition}[Poly-smoothness]
A function $\psi:\Omega\subseteq\RR^n\to\RR^d$ in $C^1$ satisfies polynomial generalized smoothness if
\begin{align*}
    \|\nabla \psi(w)-\nabla \psi(w')\|\le S(\|w_{\max,1}\|,\cdots,\|w_{\max,n_w}\|)\|w-w'\|,\qquad \forall w,w'\in\Omega,
\end{align*}
where $\|w_{\max,i}\|=\max\{\|w_i\|,\|w_i'\|\}$ and $S(\cdot)$ is a polynomial with positive coefficients.
\end{definition}

This property is naturally satisfied by many activation and normalization functions used in neural network architectures, is preserved under basic operations, and is straightforward to verify along discrete GD dynamics.

The second one is \citet[Definition~5]{wang2026convergencegradientdescentgeneral}.

\begin{definition}[$(w,w^*,r,\rho,\epsilon)$-dissipativity]
A function $\psi\in C^1$ satisfies the $(w,w^*,r,\rho,\epsilon)$-dissipative condition near $w$ if there exist a stationary point $w^*$ and a constant $r>\|w-w^*\|$ such that, for every $\widetilde w\in B_{w^*}(r)\backslash\{v:\|\nabla\psi(v)\|\le\epsilon\}$,
\begin{align*}
    \nabla\psi(\widetilde w)^\top(\widetilde w-w^*)\ge-\rho\|\nabla\psi(\widetilde w)\|^2,
\end{align*}
where $\epsilon\ge0$ and $\rho\in\RR$ may depend on $w,w^*,\epsilon$ but is independent of $\widetilde w$.
\end{definition}
This condition ensures that, outside certain region, the gradient points in a direction that prevents the trajectory from escaping to infinity and is satisfied by all $C^1$ functions with a finite $\rho$; see \citet[Proposition~2]{wang2026convergencegradientdescentgeneral}. Together with the descent lemma, it keeps the GD trajectory inside a ball, on which poly-smoothness function can then be bounded.

The convergence proof in \citet{wang2026convergencegradientdescentgeneral} combines two components. First, the dissipative condition, together with poly-smoothness yields a uniform Lipschitz smoothness constant along the trajectory; see \citet[Lemma~6]{wang2026convergencegradientdescentgeneral}. Second, analyticity of the neural network ensures the nondegeneracy of the network Jacobian almost everywhere, yielding an iterate-dependent PL-type inequality with a strictly positive, but not necessarily uniform, lower frame bound; see \citet[Lemma~7]{wang2026convergencegradientdescentgeneral}. Measure-zero arguments then connect these components by excluding degenerate parameter values, data configurations, and learning rates. This argument applies to broad classes of neural network architectures and absolutely continuous input distributions. However, it does not by itself ensure convergence to a global minimizer. This limitation motivates our analysis of GD dynamics with weight decay below.

\subsection{General loss decay of the regularized objective}
\label{subsec:general_loss_decay}

In this section, we adapt the convergence framework of \citet{wang2026convergencegradientdescentgeneral} to establish decay of the regularized objective $\cL_\lambda$ along the GD dynamics. Our assumptions are imposed on the empirical loss $\cL$ and the network architecture, whereas the effect of the quadratic regularizer is treated explicitly through the dynamics.

We begin with a mild absolute-continuity condition on the joint distribution of the training inputs.
\begin{assumption}[Input data]
\label{assump:data}
    Let $\pi(x)dx$ be the population input distribution. Let $x_1,\cdots,x_N\in\RR^d$ be the training inputs and let $y_1,\cdots,y_N\in\RR^d$ be the corresponding outputs. Assume
    $$(x_1,\cdots,x_N)\,{\sim}\,\pi_N(z_1,\cdots,z_N)dz_1\cdots dz_N,$$
    where $\pi_N(z_1,\cdots,z_N)\in L^\infty$ is supported on an open set $\Omega^N\subseteq\RR^{Nd}$, and $\pi_N\ll \Leb_{Nd}$. 
\end{assumption}
The marginals of the joint training input distribution $\pi_N(X)dX$ are allowed to be different from the population distribution $\pi(x)dx$ used later for generalization. As in \citet{wang2026convergencegradientdescentgeneral}, the inputs are not required to be i.i.d.; the convergence argument only uses absolute continuity to exclude degenerate inputs. The optimization results in this section do not impose additional assumptions on the output vectors $y_i$.

We next require the center of the quadratic regularizer to be compatible with the empirical-loss landscape.
\begin{assumption}
\label{assump:0_stationary_point}
    Assume $\bzero$ is a stationary point of the empirical loss $\cL(\theta)$.
\end{assumption}
The point $\bzero$ is the minimizer of the quadratic regularization term. This assumption aligns the center of weight decay with a stationary point of the empirical loss and can be easily satisfied by neural network architectures whose activation functions vanish at $\bzero$. It will later be used in the dissipativity argument to control the direction of $\nabla\cLwd(\theta)$ and $\theta$.

In addition to the preceding assumptions, we collect the remaining model assumptions as follows.
\begin{assumption}[Model assumptions \citep{wang2026convergencegradientdescentgeneral}]
\label{assump:model_assumptions}
    Assume the model satisfies the following conditions:
    \begin{enumerate}
        \item \textbf{Real-analyticity.} Each block $\varphi_\ell(\theta_\ell;u)$ is real-analytic in its parameters and input variable on the relevant domain.
        \item \textbf{Poly-smoothness.} Each block $\varphi_\ell$, together with $\nabla_{\theta_\ell}\varphi_\ell$ and $\nabla_u\varphi_\ell$, is polynomially bounded, and satisfies polynomial generalized continuity and polynomial generalized smoothness in $(\theta_\ell,u)$.
        \item \textbf{Nonlinear architecture.} There exists a layer $\bar{\ell}\in\{0,\cdots,L-1\}$ with $\dim(\theta_{\bar{\ell}})>Nd$ such that, one can find a collection of hidden states $(u_1,\cdots,u_N)$ for which the stacked Jacobian
        $\begin{pmatrix}
            \nabla_{\theta_{\bar{\ell}}}\tilde f_{\bar{\ell}}(\theta;u_1)\\
            \vdots\\
            \nabla_{\theta_{\bar{\ell}}}\tilde f_{\bar{\ell}}(\theta;u_N)
        \end{pmatrix}$
        has full row rank $Nd$, where $\tilde f_{\bar{\ell}}$ denotes the network output map starting from layer $\bar{\ell}$.
    \end{enumerate}
\end{assumption}
This assumption contributes to the two components used in the convergence proof: poly-smoothness provides a finite smoothness constant on bounded parameter regions, while real-analyticity and the nonlinear architecture condition imply the generic full-rank property behind the iterate-dependent PL-type inequality. See Section~\ref{subsec:optimization_preliminaries} and \cite{wang2026convergencegradientdescentgeneral}.

With these data, stationary-point, and model assumptions in place, we can state the loss decay result for the regularized objective.

\begin{theorem}
    \label{thm:general_convergence_loss_decay}
    Suppose GD is applied to the neural network~\eqref{eqn:NN} under the following requirements, with probability 1 over the joint distribution $\pi_N$ of the input data:
\begin{enumerate}
    \item \textbf{Data and Model assumptions.} Assumption~\ref{assump:data}, Assumption~\ref{assump:0_stationary_point}, and Assumption~\ref{assump:model_assumptions} hold.
    \item \textbf{Initialization.} The initialization $\theta^0$ belongs to $\RR^{\dim \theta}$ except for a measure-zero set.
    \item \textbf{Dissipative condition.} The loss $\cL(\theta)$ satisfies the $(\theta^0,\bzero,R_\lambda,\rho,\epsilon_\cL)$ dissipative condition with $\theta^*=\bzero$
\begin{align*}
   i.e.,\quad\ip{\nabla_\theta\cL(\theta)}{\theta}\ge -\rho\|\nabla_\theta \cL(\theta)\|^2
\end{align*}
where $\theta\in B_\bzero (R_{\lambda})\backslash\{\theta:\|\nabla\cL\|\le\epsilon_\cL\}$ and $\rho>0$, with $\rho_{\lambda}=\frac{\rho}{1-2\rho\lambda}$, $0<\delta<2$, and
\begin{align*}
    R_{\lambda}=\sqrt{ \|\theta^0\|^2+\frac{2(2\rho_{\lambda}+1)}{\delta+\lambda(2\rho_{\lambda}+1)}\cL(\theta^0)}.
\end{align*}
\item \textbf{Weight decay.} The weight decay parameter $0<\lambda< \frac{1}{2\rho}.$ 
    \item \textbf{Learning rate.} The learning rate $\eta$ is chosen from the interval
    $$
        0<\eta\le \min\left\{\frac{2-\delta}{\mathsf{L}_\lambda},\frac{1}{\lambda},1\right\},
    $$
    excluding at most finitely many exceptional values,
    where
    $$\mathsf{L}_\lambda=S\left(\cdots,R_{\lambda}+\sup_{\theta\in B_\bzero(R_{\lambda})}\|\nabla_{\theta_i}\cLwd(\theta)\|,\cdots\right)+\lambda$$
and $S(\cdot)$ is the polynomial corresponding to the poly-smoothness of $\cL(\theta)$.

\end{enumerate}
Then, under some arbitrarily small adjustment on the scale of $\varphi_\ell$ for $\ell=0,\cdots,L-1$, we have
$$
    \cLwd(\theta^k)\le \prod_{s=0}^{k-1}\left(1-\eta\left(1-\frac{\eta\,\mathsf{L}_\lambda}{2}\right)\frac{2\mu_{\lambda,s,X}}{N}\right)\cLwd(\theta^0), \text{ for all }k\ s.t.\ \|\theta^k\|\ge \frac{2\epsilon_\cL}{\lambda}
$$
where $\mu_{\lambda,s,X}>0$ depends on $\theta^s$ and $x_1,\cdots,x_N$.
\end{theorem}

The above theorem is the loss decay estimate for the GD dynamics with weight decay. The quantity $\mu_{\lambda,s,X}$ is the analogue of the PL constant: it turns the current value of $\cL_\lambda$ into a descent rate at the current iterate. The constant $\mathsf{L}_\lambda$ is the local smoothness constant on the controlled ball $B_\bzero(R_\lambda)$, and the learning-rate restriction is exactly based on the corresponding descent condition. The remaining technical details in the statement, including the measure-zero exceptional set of initializations, the exclusion of finitely many learning rates, the probability-one statement over the input distribution, and the arbitrarily small adjustment of block scales, are inherited from the convergence framework of \citet{wang2026convergencegradientdescentgeneral}, where the underlying genericity and measure-zero arguments are discussed in detail.

The stopping region in Theorem~\ref{thm:general_convergence_loss_decay} differs from that in \citet{wang2026convergencegradientdescentgeneral}. Their result establishes loss decay when $\|\nabla\cL(\theta^k)\|>\epsilon_\cL$ and proves that GD eventually enters the small-gradient region $\{\theta:\|\nabla\cL(\theta)\|\le\epsilon_\cL\}$. In contrast, Theorem~\ref{thm:general_convergence_loss_decay} establishes decay of the regularized objective whenever $\|\theta^k\|\ge 2\epsilon_\cL/\lambda$, namely, while the trajectory remains outside the neighbourhood $B_\bzero(2\epsilon_\cL/\lambda)$ of $\bzero$. By Assumption~\ref{assump:0_stationary_point}, $\bzero$ is a stationary point of $\cL$ and hence also of $\cLwd$. The radius of this neighbourhood can be made small by choosing $\epsilon_\cL$ sufficiently small relative to $\lambda$. Consequently, the theorem guarantees continued decay of the regularized objective while the trajectory remains outside this neighbourhood. However, it does not guarantee decay of the empirical loss $\cL$. Convergence to a neighbourhood of the set of global minimizers of $\cL$ is established separately in Subsection~\ref{subsec:convergence_global_min}, particularly in Theorem~\ref{thm:convergence_to_global_min}.

To complete the loss decay argument, we next show how dissipativity of the empirical loss yields the control required for the regularized objective.

\subsubsection{Transferring dissipativity from \texorpdfstring{$\cL$}{L} to \texorpdfstring{$\cL_\lambda$}{L lambda}}

Theorem~\ref{thm:general_convergence_loss_decay} assumes dissipativity of the empirical loss $\cL$, not of the regularized loss $\cL_\lambda$. This choice is deliberate. The empirical loss is the object determined by the data and the architecture, so its dissipativity can be checked or interpreted more directly in practice. In contrast, assuming dissipativity directly for $\cL_\lambda$ would mix a property of the data and model landscape with the particular algorithmic choice of weight decay.

The consequence is that, under the GD dynamics with weight decay, the trajectory may enter the small-gradient region of $\cL$, where this dissipative condition is not assumed. In that region, the quadratic term still provides radial control as long as the parameter norm is not too small. The next lemma combines the dissipative region of $\cL$ with this large-norm, small-gradient regime.
\begin{lemma}[Dissipativity of the regularized objective]
    \label{lem:ip_nabla_cLwd_theta_lower_bound_mainbody}
    Suppose Assumption~\ref{assump:0_stationary_point} holds and $\cL$ satisfies the dissipative condition in Theorem~\ref{thm:general_convergence_loss_decay} with constant $\rho>0$. Let $0<\lambda<\frac{1}{2\rho}$ and $\rho_\lambda=\frac{\rho}{1-2\rho\lambda}$.
    Then, for all $\theta\in B_\bzero(R_{\lambda})\backslash B_\bzero(2\epsilon_\cL/\lambda)$, we have
    \begin{align*}
     \ip{\nabla_\theta\cLwd(\theta)}{\theta}\ge -\rho_{\lambda}\|\nabla_\theta \cLwd(\theta)\|^2+\frac{\lambda}{2}\|\theta\|^2.
\end{align*}
\end{lemma}
This lemma provides the dissipativity estimate used in the convergence analysis of the GD dynamics with weight decay. In particular, inside the small-gradient region, the condition $\|\theta\|\ge 2\epsilon_\cL/\lambda$ ensures that the weight-decay drift dominates the possible inward component of $\nabla\cL$.

\subsection{Convergence to a neighbourhood of global minimizers of empirical loss}
\label{subsec:convergence_global_min}

In this section, based on Theorem~\ref{thm:general_convergence_loss_decay}, we translate the decay of the regularized objective $\cL_\lambda$ into convergence to a neighbourhood of the set of global minimizers of the empirical loss $\cL$. To this end, we use the geometry of the set of global minimizers inside the controlled ball. Define the set
\begin{align*}
    \cM=\{\theta\,|\,\cL(\theta)=0,\theta\in B_\bzero(R_{\lambda})\}.
\end{align*}
As is discussed in Section~\ref{subsec:general_loss_decay}, the decay of $\cL_\lambda$ alone does not imply that the empirical loss $\cL$ becomes small, since the decrease may instead come from the quadratic regularization term. To recover control of $\cL$, we impose a geometric condition requiring part of the set of global minimizers to separate the initialization from the small-norm region where the decay estimate in Theorem~\ref{thm:general_convergence_loss_decay} stops applying. Consequently, a trajectory moving from the initialization toward this region must enter a neighbourhood of the set of global minimizers of $\cL$.

To capture the discrete nature of GD, we define the maximal one-step displacement for a learning rate $\eta$
\begin{align*}
    \epsilon_\eta= \eta \sup_{\theta\in B_\bzero(R_{\lambda})}\|\nabla\cLwd(\theta)\|.
\end{align*}
Then every update within the controlled ball satisfies
\begin{align*}
    \|\theta^{k+1}-\theta^k\|=\eta\|\nabla\cLwd(\theta^k)\|\le\epsilon_\eta,
\end{align*}
namely, $\epsilon_\eta$ is a uniform upper bound on the displacement of one GD step.

With this one-step control, we impose the following separation condition.
\begin{assumption}[Global minimum and initialization]
    \label{assump:global_min}
    Assume the initialization $\theta^0$ satisfies $$\|\theta^0\|>2\epsilon_\cL/\lambda+\epsilon_\eta.$$ 
    Assume that there exists a nonempty compact subset $\cM_0\subseteq\cM$ such that $\cM_0$ separates $\theta^0$ from $B_\bzero(2\epsilon_\cL/\lambda+\epsilon_\eta)$ inside $ B_\bzero(R_{\lambda})$. More precisely, there exist two disjoint nonempty sets $\cB_1,\cB_2\subseteq  B_\bzero(R_{\lambda})\backslash\cM_0$ that are relatively open in $B_\bzero(R_{\lambda})\backslash\cM_0$ such that
    $$ B_\bzero(R_{\lambda})\backslash\cM_0=\cB_1\cup \cB_2,\qquad
    \theta^0\in \cB_1,\qquad
    B_\bzero(2\epsilon_\cL/\lambda+\epsilon_\eta)\subseteq \cB_2.$$
\end{assumption}
Assumption~\ref{assump:global_min} is a topological separation condition. Under this condition, the set $\cB_1$ is the region containing the initialization, while $\cB_2$ contains the enlarged small-norm region. Removing $\cM_0$ separates these two regions, so any continuous path between them must intersect $\cM_0$. Since GD is discrete, the enlargement by $\epsilon_\eta$ and the one-step bound above ensure that the trajectory cannot pass directly from $\cB_1$ to $\cB_2$ without entering an $\epsilon_\eta$-neighbourhood of $\cM_0$.

To translate this separation geometry into a descent estimate, we next introduce a constant relating the norm of $\theta$ to its distance from the set of global minimizers. Define this distance by
\[
    \operatorname{dist}(\theta,\cM)=\inf_{\theta^*\in\cM}\|\theta-\theta^*\|
\]
and then define
\begin{align*}
    \zeta=\zeta(\lambda,\epsilon_\cL,R_{\lambda})
    =\sqrt{\frac{2}{\lambda}}\max_{\frac{2\epsilon_\cL}{\lambda}\le \|\theta\|\le R_{\lambda}} \frac{\operatorname{dist}(\theta,\cM)}{\|\theta\|}.
\end{align*}
The constant $\zeta$ measures how far a parameter can be from $\cM$, relative to its norm, outside the small ball $B_\bzero(2\epsilon_\cL/\lambda)$. It enters the gradient lower bound that drives the descent estimate below.

\begin{theorem}[Convergence to a neighbourhood of the global minimizer set]
    \label{thm:convergence_to_global_min}
    Consider the same assumptions in Theorem~\ref{thm:general_convergence_loss_decay}, and Assumption~\ref{assump:global_min}. Assume additionally
    \begin{align*}
        \eta<\frac{\|\theta^0
        \|-2\epsilon_\cL/\lambda}{\sup_{\theta\in B_\bzero(R_{\lambda})}\|\nabla\cLwd(\theta)\|}.
    \end{align*}
    There exist $K_1\in\mathbb{N}$ and $K_2\in\mathbb{N}\cup\{+\infty\}$ with $K_1\le K_2$ such that, for all $k\le K_2$, we have
    \begin{align*}
        \cLwd(\theta^k)\le\left(1- \eta\left(1-\frac{\eta\, \mathsf{L}_\lambda}{2}\right)\frac{\lambda}{\mathsf{L}_\cM \zeta^2+2}\right)^{k}\cLwd(\theta^{0}).
    \end{align*}

    Moreover, for all $K_1\le k\le K_2$ we have
    \begin{align*}
        &\cL(\theta^k)\le  \frac{1}{2}\mathsf{L}_{\cM_0} \epsilon_\eta^2=\cO(\eta^2),\\
   & \|\nabla\cL(\theta^k)\|\le \mathsf{L}_{\cM_0}\epsilon_\eta=\cO(\eta),
    \end{align*}
    where
    \begin{align*}
        \mathsf{L}_{\cM}
        &=\sup_{\theta\in B_\bzero(R_{\lambda})}\sup_{\theta^*\in\cM}
        S(\cdots,\max\{\|\theta_i\|,\|\theta^*_i\|\},\cdots),\\
        \mathsf{L}_{\cM_0}
        &=\sup_{\theta\in B_\bzero(R_{\lambda})}\sup_{\theta^*\in\cM_0}
        S(\cdots,\max\{\|\theta_i\|,\|\theta^*_i\|\},\cdots).
    \end{align*}
\end{theorem}

In the above theorem, the indices $K_1$ and $K_2$ mark the time interval during which the trajectory stays inside the $\epsilon_\eta$-neighbourhood of the subset $\cM_0$ of global minimizers. The decay factor has two parts: $\eta(1-\eta\mathsf{L}_\lambda/2)$ is the standard GD descent factor coming from local smoothness on $B_\bzero(R_\lambda)$, while $\lambda/(\mathsf{L}_\cM\zeta^2+2)$ is the lower-bound constant for $\|\nabla\cL_\lambda\|^2$ near the set of global minimizers.

Consequently, GD is guaranteed to reach a neighbourhood of the set of global minimizers. Without additional training strategies, such as a learning-rate schedule or injected stochasticity, the theorem does not claim convergence to an arbitrarily small neighbourhood. Rather, for the fixed learning rate $\eta$, it gives $\cL(\theta^k)=\cO(\eta^2)$ and $\|\nabla\cL(\theta^k)\|=\cO(\eta)$ in this time interval.

\section{Generalization Theory}
\label{sec:generalization_theory}
We now turn from optimization to generalization through the analysis of the population loss. The main idea is to compare the network and the ground truth through cellwise representatives determined by the training data. This produces three terms: a data error (Section~\ref{subsec:data_error}), an optimization error (Section~\ref{subsec:optimization_error}), and a prediction variation error (Section~\ref{subsec:prediction_variation_error}). 

The construction starts by assigning each training input to a local cell in the input domain. This requires the training inputs to be distinct. Under the assumption that their joint distribution is absolutely continuous, this holds almost surely, as stated in the following lemma.
\begin{lemma}
    \label{lem:non_overlapping_input}
    Under Assumption~\ref{assump:data}, with probability 1 over the joint distribution $\pi_N$, $x_i\ne x_j$ for all $i\ne j$. 
\end{lemma}

On this probability-one event, each sample can serve as the representative of one cell. More precisely, we introduce the following notion of a partition.

\begin{definition}[Partition up to $\pi$-null sets]
    \label{def:partition_pi_null}
    Given pairwise distinct training inputs $x_1,\cdots,x_N\in\Omega$, we say that $\{\Omega_i\}_{i=1}^N$ is a partition of $\Omega$ up to $\pi$-null sets adapted to $\{x_i\}_{i=1}^N$ if each $\Omega_i\subseteq\Omega$ is measurable and the following conditions hold:
    \begin{enumerate}
        \item $\Omega=\bigcup_{i=1}^N\Omega_i$;
        \item $x_i\in\Omega_i^o$ and $\pi(\Omega_i)>0$ for every $i=1,\cdots,N$;
        \item $\pi(\Omega_i\cap\Omega_j)=0$ for every $i\ne j$.
    \end{enumerate}
\end{definition}
Since all population quantities below are integrated with respect to $\pi(x)dx$, it is enough to require the intersections of distinct cells to have $\pi$-measure zero, so common boundaries do not affect these quantities. For the remainder of this section, we fix such a partition $\{\Omega_i\}_{i=1}^N$, while the theory is stated for any choice satisfying Definition~\ref{def:partition_pi_null}. One canonical example is the Voronoi partition generated by $\{x_i\}_{i=1}^N$, although other choices may better reflect the geometry of the input distribution or the regularity of the target function; see Section~\ref{subsec:data_error}.

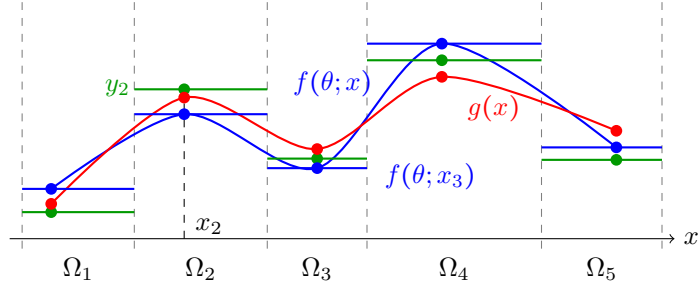
\begin{figure}
    \centering
    \begin{tikzpicture}[scale=1.1]

	    \draw[->] (-0.5,0) -- (7.5,0) node[right] {$x$};

			    \foreach \x in {-0.35,1.0,2.6,3.8,5.9,7.35}
			        \draw[dashed, gray] (\x,-0.12) -- (\x,2.85);
			    \node[below] at (0.325,-0.12) {$\Omega_1$};
			    \node[below] at (1.8,-0.12) {$\Omega_2$};
			    \node[below] at (3.2,-0.12) {$\Omega_3$};
			    \node[below] at (4.85,-0.12) {$\Omega_4$};
			    \node[below] at (6.625,-0.12) {$\Omega_5$};
		    \draw[dashed, black] (1.6,0) -- (1.6,1.9);

		    \coordinate (P1) at (0,0.6);
		    \coordinate (P2) at (1.6,1.5);
			    \coordinate (P3) at (3.2,0.85);
			    \coordinate (P4) at (4.7,2.35);
			    \coordinate (P5) at (6.8,1.1);
		    \coordinate (Q1) at (0,0.42);
		    \coordinate (Q2) at (1.6,1.70);
		    \coordinate (Q3) at (3.2,1.08);
		    \coordinate (Q4) at (4.7,1.95);
		    \coordinate (Q5) at (6.8,1.30);
		    \coordinate (R1) at (0,0.32);
		    \coordinate (R2) at (1.6,1.80);
		    \coordinate (R3) at (3.2,0.965);
		    \coordinate (R4) at (4.7,2.15);
		    \coordinate (R5) at (6.8,0.95);

	    \draw[blue, thick, smooth] plot coordinates {
		        (P1) (P2) (P3) (P4) (P5)
		    };
	    \draw[red, thick, smooth] plot coordinates {
		        (Q1) (Q2) (Q3) (Q4) (Q5)
		    };

					    \foreach \xleft/\xright/\y in {
					        -0.35/1.0/0.6,
					        1.0/2.6/1.5,
						        2.6/3.8/0.85,
					        3.8/5.9/2.35,
				        5.9/7.35/1.1
					    }
						        \draw[blue, thick] (\xleft,\y) -- (\xright,\y);

					    \foreach \xleft/\xright/\y in {
					        -0.35/1.0/0.32,
					        1.0/2.6/1.80,
						        2.6/3.8/0.965,
						        3.8/5.9/2.15,
					        5.9/7.35/0.95
					    }
						        \draw[green!60!black, thick] (\xleft,\y) -- (\xright,\y);

					    \foreach \r in {R1,R2,R3,R4,R5}
					        \fill[green!60!black] (\r) circle (2pt);
					    \foreach \q in {Q1,Q2,Q3,Q4,Q5}
					        \fill[red] (\q) circle (2pt);
					    \foreach \p in {P1,P2,P3,P4,P5}
					        \fill[blue] (\p) circle (2pt);
	                        \node[blue, above right] at (3.9,0.48) {$f(\theta;x_3)$};
	                         \node[blue, above right] at (2.8,1.6) {$f(\theta;x)$};
							    \node[red, below right] at (4.9,1.85) {$g(x)$};
							    \node[green!60!black, above] at (0.8,1.60) {$y_2$};
							    \node[black, below] at (1.9,0.35) {$x_2$};

\end{tikzpicture}
\caption{A schematic illustration of the population error decomposition. The blue and red smooth curves represent the network output $f(\theta;x)$ and the ground truth $g(x)$, respectively. The blue and green piecewise-constant functions represent $\bar f_\theta(x)$ and $\bar y(x)$, respectively; on each cell $\Omega_i$, they take the values $f(\theta;x_i)$ and $y_i$.}
\label{fig:population_error_decomposition}
\end{figure}

We now use the fixed partition to decompose the population error. For each cell $\Omega_i$, the training point $x_i$ serves as the representative input. Then, we define the corresponding piecewise constant cellwise network and label representative functions by
\begin{align*}
    \bar f_\theta(x):=\sum_{i=1}^N f(\theta;x_i)\bbone_{x\in\Omega_i},
    \qquad
    \bar y(x):=\sum_{i=1}^N y_i\bbone_{x\in\Omega_i}.
\end{align*}
In Figure~\ref{fig:population_error_decomposition}, $\bar f_\theta$ is shown by the blue horizontal pieces and $\bar y$ by the green horizontal pieces. The population error is then compared along the chain $f(\theta;\cdot)\to \bar f_\theta\to \bar y\to g$. 
Therefore, by Jensen's inequality, the population loss decomposes into three errors:
\begin{align*}
    \EE_{x\sim\pi}\|f(\theta;x)-g(x)\|^2
    &\le
    \underbrace{3\,\EE_{x\sim\pi}
    \left\|f(\theta;x)-\bar f_\theta(x)\right\|^2}_{E_{\rm PV}(\theta)}
    +\underbrace{3\,\EE_{x\sim\pi}
    \left\|\bar f_\theta(x)-\bar y(x)\right\|^2}_{E_{\rm opt}(\theta)}
    +\underbrace{3\,\EE_{x\sim\pi}
    \left\|\bar y(x)-g(x)\right\|^2}_{E_{\rm data}}.
\end{align*}
The three terms are the squared distances between consecutive functions in the chain above. The data error $E_{\rm data}$ is independent of training once the samples and partition are fixed. The optimization error $E_{\rm opt}$ is a weighted empirical error evaluated at the training data, which can be seen more clearly from the following cellwise decomposition. The prediction variation error $E_{\rm PV}$ measures how much the trained network varies within each cell and is the key term for characterizing generalization performance through architecture and optimization dynamics.

Since the cells form a partition up to $\pi$-null sets, the three errors can equivalently be decomposed into local cell contributions:
\begin{align*}
    E_{\rm PV}(\theta)
    &=3 \sum_{i=1}^N\int_{\Omega_i}\pi(x)dx\cdot
    \EE_{\Omega_i}\|f(\theta;x)-f(\theta;x_i)\|^2,\\
    E_{\rm opt}(\theta)
    &=3\sum_{i=1}^N\int_{\Omega_i}\pi(x)dx\cdot\|f(\theta;x_i)-y_i\|^2,\\
    E_{\rm data}
    &=3\sum_{i=1}^N\int_{\Omega_i}\pi(x)dx\cdot \EE_{\Omega_i}\|y_i-g(x)\|^2,
\end{align*}
where $\EE_{\Omega_i}[\cdot]:=\EE_{x\sim\pi}[\cdot\mid x\in\Omega_i]$ is the expectation under the population input distribution conditioned on the event $x\in\Omega_i$. The weights $\int_{\Omega_i}\pi(x)dx$ are the population masses of the cells. They depend both on the partition and on the true input distribution. Hence the same local error can contribute differently to the population loss depending on where it occurs: errors on high-mass cells are emphasized, while errors on low-density cells are downweighted. Note that this partition is not introduced only for theoretical purposes; it also represents how the training inputs are sampled relative to the true distribution. This is especially useful when the sampling distribution differs from the true input distribution. For instance, if the input data are sampled from a multi-modal distribution different from the true distribution, then a partition with balanced masses may not exist.

In the next three subsections, we bound $E_{\rm data}$, $E_{\rm opt}$, and $E_{\rm PV}$ separately.

\subsection{Data error}
\label{subsec:data_error}

In this section, we study the data error $E_{\rm data}$, which is controlled by the noise and the shape of ground truth $g$. The data error is independent of the training dynamics and the network architecture. Once the data and partition are fixed, it measures how well the observed label at each input $x_i$ approximates the ground truth function throughout $\Omega_i$.
To distinguish label noise from within-cell data error, we impose the following standard additive noise assumption.

\begin{assumption}[Output data]
    \label{assump:output_data}
    Assume the output data $$y_i=g(x_i)+\sigma_{\rm data}\xi_i\in\RR^d,$$ where $\sigma_{\rm data}\ge 0$, $\xi_i\overset{i.i.d.}{\sim}\cN(0,I)$ are independent of the training inputs, and $g$ is the ground truth function. 
\end{assumption}

This assumption is used only to make the contributions of ground truth data and noise explicit; it is not required for the population error decomposition. We impose no smoothness, parametric form, or regularity condition on $g$, apart from the basic measurability and integrability. Without this additive noise model, one may instead bound $E_{\rm data}$ directly from the label distribution or other problem-specific information. 

Under the stated model, the two contributions are separated by the following theorem.

\begin{theorem}
    \label{thm:gen_data_error}
Under Assumption~\ref{assump:data} and~\ref{assump:output_data}, with probability $1$ over the joint input distribution $\pi_N$, and with probability at least $1-\delta$ over the output noise,
\begin{align*}
   E_{\rm data}\le \underbrace{6\,\sigma_{\rm data}^2 \left(d+ 2 \sqrt{d\sum_{i=1}^N\left(\int_{\Omega_i}\pi(x)dx\right)^2\log\frac{1}{\delta}}+ 2\log\frac{1}{\delta}\max_{1\le i\le N}\int_{\Omega_i}\pi(x)dx \right)}_{E_{\rm data,I}}+\underbrace{6\,\EE_{x\sim\pi(x)}\left\|\sum_{i=1}^N g(x_i) \bbone_{x\in\Omega_i}-g(x)\right\|^2}_{E_{\rm data,II}}.
\end{align*}
    
\end{theorem}

The first term $E_{\rm data,I}$ captures the Gaussian label noise: the output noise enters through $\sigma_{\rm data}$, while the input distribution and partition enter through the cell masses $\int_{\Omega_i}\pi(x)dx$. Among nonnegative masses with total mass one, the balanced choice minimizes this term:
\begin{align*}
    \min_{\{\Omega_i\}}E_{\rm data,I}
    =6\,\sigma_{\rm data}^2 \left(d+ 2 \sqrt{\frac{d}{N}\log\frac{1}{\delta}}+ 2\frac{1}{N}\log\frac{1}{\delta} \right)
\end{align*}
when $\int_{\Omega_i}\pi(x)dx=\int_{\Omega_j}\pi(x)dx$ for all $i,j=1,\cdots,N$. This minimization concerns only the noise contribution. The deterministic term $E_{\rm data,II}$ measures how well the piecewise constant function $\sum_{i=1}^N g(x_i)\bbone_{x\in\Omega_i}$ approximates $g$ under the population distribution; it depends jointly on the inputs, the partition, and the within-cell variation of $g$. Thus, the partition creates a tradeoff: balanced cell masses reduce the first term, whereas geometry-adapted cells reduce the deterministic term; additionally, this tradeoff should be considered under the full population error, since the same partition also enters $E_{\rm opt}$ and $E_{\rm PV}$.

\subsection{Optimization error}
\label{subsec:optimization_error}

In this section, we bound the optimization error $E_{\rm opt}$ through the optimization dynamics and the population cell masses. The optimization dynamics have already been analyzed in Section~\ref{sec:optimization_theory}. In particular, Theorem~\ref{thm:convergence_to_global_min} bounds the empirical loss once the GD dynamics with weight decay reach a neighbourhood of the global minimizer set of $\cL$. Here we translate that training guarantee into a bound on $E_{\rm opt}$, which weights the residual at each data point by its population cell mass. This requires only an additional control by the largest cell mass and gives the following estimate.

\begin{corollary}
    \label{cor:gen_optimization_error}
    Under the same assumptions as Theorem~\ref{thm:convergence_to_global_min}, for every $K_1\le k\le K_2$,
     \begin{align*}
    E_{{\rm opt}}(\theta^k)\le \left(3 N \max_{1\le i\le N}\int_{\Omega_i}\pi(x)dx\right)\mathsf{L}_{\cM_{0}} \epsilon_\eta^2.
\end{align*}
 \end{corollary}
This bound separates the effects of training and population weighting. The training contribution $\frac{1}{2}\mathsf{L}_{\cM_{0}} \epsilon_\eta^2$ comes from Theorem~\ref{thm:convergence_to_global_min} and quantifies the empirical loss achieved by the GD dynamics with weight decay. The cell-mass factor $6N \max_{1\le i\le N}\int_{\Omega_i}\pi(x)dx$ is the additional control needed to convert the equally weighted empirical loss into the population-weighted optimization error; it depends on the partition and the population input distribution rather than on training. This separation motivates a population-aligned training objective through loss reweighting.

\subsubsection{Optimization error and loss reweighting}
The preceding bound in Corollary~\ref{cor:gen_optimization_error} suggests that optimizing the empirical loss alone may not be enough to control the population-weighted optimization error; the conversion from empirical optimization to population error depends on how the population mass is distributed across the partition cells. Next, we compare the balanced and imbalanced cell masses, and discuss the insights of loss reweighting.

We first consider the balanced case. Since the cell masses sum to one, the imbalance factor satisfies
$$N \max_{1\le i\le N}\int_{\Omega_i}\pi(x)dx\ge 1.$$
Equality holds when there exists a partition up to $\pi$-null sets (Definition~\ref{def:partition_pi_null}) such that $\int_{\Omega_i}\pi(x)dx=\int_{\Omega_j}\pi(x)dx$ for all $i,j=1,\cdots,N$. In this case, the population and empirical weights coincide:
\begin{align*}
    E_{{\rm opt}}(\theta^k)=6\cL(\theta^k).
\end{align*}
Thus the standard empirical loss is already aligned with $E_{\rm opt}$, and Theorem~\ref{thm:convergence_to_global_min} gives $E_{{\rm opt},K}\le 3\mathsf{L}_{\cM_{0}}\epsilon_\eta^2$. 

However, balancing the cell masses may increase the deterministic data error $E_{\rm data,II}$ by producing cells poorly adapted to the geometry of $g$. Therefore, even when such a balanced partition exists, it need not be optimal for the total population error. A partition that achieves a smaller total error may therefore be imbalanced. A data point from a high-mass cell describes a larger portion of the population but can still receive weight $1/N$ in the standard empirical loss. This mismatch motivates the reweighted empirical loss
\begin{align*}
    \cL_{\rm re}(\theta^k)=\frac{1} {2}\sum_{i=1}^N\int_{\Omega_i}\pi(x)dx\cdot \|f(\theta^k;x_i)-y_i\|^2.
\end{align*}
With this definition, the optimization term is exactly
\begin{align*}
         E_{{\rm opt}}(\theta^k)= 6\cL_{\rm re}(\theta^k).
    \end{align*}
Thus minimizing $\cL_{\rm re}$ directly controls the weighted residuals in the population error decomposition. If a training strategy reduces $\cL_{\rm re}$ to the same scale as the standard empirical loss, then the optimization contribution avoids the imbalance factor $N\max_{1\le i\le N}\int_{\Omega_i}\pi(x)dx$. Consequently, reweighting aligns the training objective more directly with the population error, which is consistent with its use in imbalanced learning \citep{cui2019class,lin2017focal,park2021influence,cao2019learning,menon2020long,li2021autobalance}.

Loss reweighting is not the only way to reduce this contribution. Even without it, sufficiently long training of the standard loss can sometimes compensate for the imbalance factor by driving $\cL(\theta^k)$ to a smaller value. This is compatible with observations that, in some regimes, the standard loss can approach the performance of reweighted losses when optimized for a sufficiently long time \citep{byrd2019effect,sagawa2020investigation,kini2021label,xu2021understanding}.

\subsection{Prediction variation error}
\label{subsec:prediction_variation_error}

In this section, we study the remaining term $E_{\rm PV}$, which is the most important term in generalization. This term measures how much the network prediction varies within each partition cell relative to its value at the training input, and captures the effects of architecture and training dynamics that are not reflected in $E_{\rm data}$ or $E_{\rm opt}$. We first introduce (local) approximate homogeneity to quantify the behavior of nonlinear blocks under weight decay (Section~\ref{subsubsec:approximate_homogeneity}). Next, we establish the evolution of prediction variation error $E_{\rm PV}$ under the GD dynamics with weight decay, separating contraction from the approximate homogeneity, empirical gradient, and regularized objective contributions (Section~\ref{subsubsec:prediction_variation_dynamics}). We then refine this analysis to cellwise and layerwise prediction variations (Section~\ref{subsubsec:layerwise_prediction_variation}). Based on this refinement, we derive a necessary condition for relative layerwise error decay, which characterizes differences in learning behavior across layers (Section~\ref{subsubsec:necessary_generalization}). Finally, we establish a sufficient condition showing that a small approximate homogeneity error, a small empirical gradient and regularized objective in the neighbourhood of global minimizers, and a sufficiently long time in that region guarantee layerwise decay, providing a mechanism for delayed generalization and grokking (Section~\ref{subsubsec:sufficient_generalization}).

\subsubsection{Approximate homogeneity}
\label{subsubsec:approximate_homogeneity}

Homogeneity records how a function changes under rescaling. In this section, we introduce basic definitions and properties of homogeneous functions, including exact homogeneity, Euler's homogeneous function theorem, approximate homogeneity, and local approximate homogeneity. These notions allow us to quantify when and how a nonlinear block of a neural network behaves, either globally or locally, like a function with a well-defined scaling degree.

We first recall the standard notion of homogeneity.

\begin{definition}[Homogeneity]
    \label{def:homogeneity}
    A function $\psi:\Omega\subseteq\RR^m\to\RR^n$ is homogeneous of degree $s$ if 
    \begin{align*}
        \psi(tx)=t^s\psi(x),\text{ for }x,tx\in\Omega,\ t\in\RR.
    \end{align*}
    If the above equation holds for $t\ge 0$, then $\psi$ is positively homogeneous of degree $s$.
\end{definition}

This definition says that, when the input is scaled by $t$, then the output is scaled by $t^s$. The following theorem connects this property to a differential equation.

\begin{theorem}[Euler's homogeneous function theorem]
    If a function $\psi:\Omega\subseteq\RR^m\to\RR^n$ is positively homogeneous of degree $s$ and $\psi$ is $C^1(\Omega^\mathrm{o})$, then it satisfies Euler's equation
    \begin{align*}
        s\, \psi(x)=\nabla \psi(x)\, x,\quad\text{ for all }x\in\Omega^{\mathrm{o}}.
    \end{align*}
\end{theorem}
Euler's equation can be viewed as an infinitesimal statement of homogeneity. It is obtained by differentiating the scaling identity $\psi(tx)=t^s\psi(x)$ with respect to $t$ at $t=1$. The vector $\nabla\psi(x)x$ consists of the directional derivatives of the components of $\psi$ in the direction $x$, so the theorem says that exact homogeneity is equivalent to these directional derivatives being exactly $s$ times the corresponding output components. 

In practice, however, modern neural network blocks are often not exactly homogeneous. Inspired by \citet{cimpean2011hyers}, we introduce an approximate version of homogeneity by measuring the residual $s\psi(x)-\nabla\psi(x)x$ from Euler's equation.
\begin{definition}[Approximate homogeneity]
    \label{def:approx_homogeneity}
    Let $\psi(\theta,u):\Omega_\theta\times\Omega_u\subseteq \mathbb R^k\times\mathbb R^m\to \mathbb R^n$ be $C^1$. We say that
         $\psi(\theta,u)$ is $\epsilon$-approximately homogeneous of degree $s$ with respect to $u$ if 
         \begin{align*}
             \sup_{(\theta,u)\in \Omega_\theta\times\Omega_u}\|s\, \psi(\theta,u)-\nabla_u \psi(\theta,u)\, u\|\le \epsilon.
         \end{align*}
\end{definition}
This definition requires the Euler residual with respect to the input variable $u$ to be uniformly bounded over the region $\Omega_\theta\times\Omega_u$, rather than requiring the exact scaling relation $\psi(\theta,tu)=t^s\psi(\theta,u)$. The parameter $\epsilon$ quantifies how far the block is from being homogeneous of degree $s$ in this componentwise directional-derivative sense. When $\epsilon=0$, the definition recovers the exact homogeneity with respect to $u$.

In some situations, a single global degree may be too restrictive. Hence, we localize the preceding condition by allowing the effective degree to depend on the region where the function is evaluated.

\begin{definition}[$(\epsilon,D_\theta,D_u)$-local approximate homogeneity]
    \label{def:local_approx_homogeneity}
Let $\psi(\theta,u):\Omega_\theta\times\Omega_u\subseteq \mathbb R^k\times\mathbb R^m\to \mathbb R^n$ be $C^1$. We say that 
$\psi(\theta,u)$ is $(\epsilon,D_\theta,D_u)$-locally approximately homogeneous with respect to $\theta$ if, for every product subset $\widetilde\Omega_\theta\times\widetilde\Omega_u$ satisfying $\diam(\widetilde\Omega_\theta)\le D_\theta$ and $\diam(\widetilde\Omega_u)\le D_u$, there exists $s_{\widetilde\Omega_\theta,\widetilde\Omega_u}\ge 0$ such that
\[
\sup_{(\theta,u)\in \widetilde\Omega_\theta\times\widetilde\Omega_u}
\left\|
s_{\widetilde\Omega_\theta,\widetilde\Omega_u}\psi(\theta,u)
-
\nabla_\theta\psi(\theta,u)\theta
\right\|
\le \epsilon .
\]
\end{definition}
The local version imposes a uniform Euler residual condition on every product region whose $\theta$-diameter is at most $D_\theta$ and whose $u$-diameter is at most $D_u$. The degree $s_{\widetilde\Omega_\theta,\widetilde\Omega_u}$ is allowed to vary from region to region, so the definition measures whether the function has a locally consistent scaling behavior rather than a single global degree. Smaller values of $D_\theta$ and $D_u$ make the condition more local; smaller values of $\epsilon$ make the local approximation of homogeneity more accurate.

\subsubsection{Dynamics of the prediction variation error}
\label{subsubsec:prediction_variation_dynamics}

In this section, we bound the change in the prediction variation error $E_{\rm PV}$ along the GD dynamics. After the data error and the optimization error are controlled, i.e., after the empirical loss reaches its global minimum, good generalization requires the learned network to have small within-cell oscillation: when the data are good enough to represent the ground truth function, for most $x\in\Omega_i$, the prediction $f(\theta;x)$ should not be too far from the representative prediction $f(\theta;x_i)$. Thus, when the input partition and the output data are sufficiently accurate, the remaining prediction variation is governed by how much the network varies inside each cell. The goal of this section is to show that this oscillation can be reduced along the trajectory generated by the GD dynamics with weight decay when the architecture is (locally) approximately homogeneous and the training is sufficient.

To make the contribution of each layer explicit, we introduce additional structure into the neural network.
\begin{assumption}[Architecture]
    \label{assump:architecture_generalization}
    Assume 
    \begin{align*}
        \varphi_L(\theta_L;u)=\theta_Lu\quad\text{ and }\quad\varphi_\ell(\theta_\ell;u)=\theta_{\ell,\rm out}\sigma_\ell(\theta_{\ell,\rm in};\tau_\ell(u)),
    \end{align*}
    where $\theta_{\ell,\rm out}\in\RR^{d\times m_\ell}$, $\theta_{\ell,\rm in}\in\RR^{m_\ell\times d}$; $\sigma_\ell(\cdot)$ and $\tau_\ell(\cdot)$ are analytic.
\end{assumption}
The assumption of a linear final layer is necessary for studying layerwise behavior. The block structure separating the inner and outer weights is commonly used in practice. It will be used to identify their distinct roles in generalization but is not necessary for quantifying the generalization error.

Before studying the dynamics, we first decompose the prediction variation into layerwise contributions in the following lemma.
\begin{lemma}[Layerwise decomposition of prediction variation error]
    \label{lem:prediction_variation_layer_decomposition}
    Under Assumption~\ref{assump:architecture_generalization}, the network output satisfies
    \begin{align*}
        f(\theta;x)=\theta_Lx+\sum_{\ell=0}^{L-1}\theta_L\varphi_\ell(\theta_\ell;u_\ell(\theta_{\ell-1:0};x)).
    \end{align*}
    Consequently, the prediction variation satisfies
    \begin{align*}
        E_{\rm PV}(\theta)
        &\le E_{{\rm PV,layer}}(\theta):= 3(L+1)\left(V_{\rm pre}(\theta)+\sum_{\ell=0}^{L-1}V_\ell(\theta)\right),
    \end{align*}
    where the pre-layer and layerwise prediction variations are, respectively,
    \begin{align*}
        V_{\rm pre}(\theta)&:=\sum_{i=1}^N\left(\int_{\Omega_i}\pi(x)dx\right)\,\EE_{\Omega_i}\|\theta_Lx-\theta_Lx_i\|^2,\\
        V_\ell(\theta)&:=\sum_{i=1}^N\left(\int_{\Omega_i}\pi(x)dx\right)\,\EE_{\Omega_i}\|\theta_L\varphi_\ell(\theta_\ell;u_\ell(\theta_{\ell-1:0};x))-\theta_L\varphi_\ell(\theta_\ell;u_\ell(\theta_{\ell-1:0};x_i))\|^2,
    \end{align*}
    for $\ell=0,\cdots,L-1$.
\end{lemma}

When the last layer is linear, the neural network can be decomposed into a sum of the contributions from individual layers. As with $\bar f_\theta$ and $\bar y$, each layerwise variation can equivalently be represented using a piecewise-constant function that takes the value of the corresponding layer contribution at $x_i$ on $\Omega_i$. Thus, $V_{\rm pre}$ and $V_\ell$ measure the mean squared changes in these contributions when an input in a cell is replaced by its representative. Lemma~\ref{lem:prediction_variation_layer_decomposition} then shows that controlling every layerwise variation is sufficient to control $E_{\rm PV}$. 

Next, we introduce a structural condition that allows these layerwise quantities to be bounded along the GD dynamics with weight decay. In particular, we impose (local) approximate homogeneity at the parameter values and hidden states encountered within each partition cell.
\begin{assumption}[Approximate homogeneity]
    \label{assump:approximate_homogeneity}
    Assume that, for $\ell=0,\cdots,L-1$:
    \begin{enumerate}
        \item $\varphi_\ell(\theta_\ell;u)$ is $\epsilon_{\rm ah}$-approximately homogeneous of degree $1$ with respect to $u$ in $B_\bzero(R_\lambda)\times\Omega$.
        \item $\varphi_\ell(\theta_\ell;u)$ is $(\epsilon_{\rm ah},D_\theta,D_u)$-locally approximately homogeneous with respect to $\theta_{\ell,\rm in}$ and $\theta_{\ell,\rm out}$, where
        \begin{align*}
            D_u\ge\max_{1\le i\le N}\sup_{\theta\in B_\bzero(R_\lambda)}\diam(u_\ell(\theta;\{x\in\Omega_i\})).
        \end{align*}
    \end{enumerate}
\end{assumption}
This assumption is essential for quantifying the dynamics of the prediction variation error. One may also impose finer assumptions under which the error parameter $\epsilon_{\rm ah}$ varies across variables and layers. The first condition controls the response of a block to its hidden input, while the second controls its local response to parameter rescaling. The quantity $\epsilon_{\rm ah}$ measures the resulting approximation error and depends on both the architecture and the data partition. The requirement on $D_u$ links data geometry to architecture: smaller diameters of the hidden nodes inside the cells generally lead to smaller approximate homogeneity errors, while more heterogeneous data require blocks with stronger local homogeneity.

Under these conditions, the layerwise prediction variation contracts up to three interpretable residual terms.

\begin{theorem}
    \label{thm:gen_NN_error}
    Under the assumptions of Theorem~\ref{thm:general_convergence_loss_decay}, Assumption~\ref{assump:architecture_generalization}, and Assumption~\ref{assump:approximate_homogeneity}, for $k\ge0$,
    \begin{align*}
        E_{\rm PV,layer}(\theta^{K_0+k})
        \le&\left(1-2\eta\lambda\right)^kE_{\rm PV,layer}(\theta^{K_0})
        +\underbrace{3L(L+1)(L+3)C_{\rm ah}\epsilon_{\rm ah}}_{E_{\rm PV,I}}
        +\underbrace{\frac{3(L+1)^2S_V}{2-\eta\mathsf{L}_\lambda}\cLwd(\theta^{K_0})}_{E_{\rm PV,II}}\\
        &+\underbrace{\frac{3(L+1)C_{\rm grad}}{\lambda}\left(\frac{1}{2}+\sum_{\ell=0}^{L-1}\frac{1}{\ell+2}\right)
        \max_{K_0\le j\le K_0+k}\|\nabla\cL(\theta^j)\|}_{E_{\rm PV,III}},
    \end{align*}
    and
    \begin{align*}
        E_{\rm PV}(\theta^{K_0+k})\le E_{\rm PV,layer}(\theta^{K_0+k}),
    \end{align*}
    where $C_{\rm ah},C_{\rm grad},S_V>0$ are independent of $k,K_0$.
\end{theorem}
Theorem~\ref{thm:gen_NN_error} characterizes the dynamics of the prediction variation error, and the complete version is Theorem~\ref{thm:gen_NN_error_full}. The first term is the geometric contraction generated by weight decay. The contraction rate of each layer is indeed distinct; see the complete version in Theorem~\ref{thm:gen_NN_error_full} and Section~\ref{subsubsec:layerwise_prediction_variation}. The three remaining terms quantify the gap from exact decay. This is natural because exact decay would imply that the learned function is eventually close to piecewise constant on the partition cells. The term $E_{\rm PV,I}$ is the error arising from (local) approximate homogeneity. As discussed above, its magnitude depends on the architecture and the data partition through $\epsilon_{\rm ah}$, and on the depth $L$ and some bounds over $B_\bzero(R_\lambda)$ encoded by $C_{\rm ah}$. The term $E_{\rm PV,II}$ is small when the regularized objective $\cLwd(\theta)$ has already decreased sufficiently by time $K_0$. The term $E_{\rm PV,III}$ is controlled primarily by the empirical loss gradient $\nabla\cL(\theta)$. Hence, it is small when the GD iterates are near a stationary point. However, making the optimization error small additionally requires proximity to the set of global minimizers. Therefore, both errors are small when the iterates are near a global minimizer whose norm of $\theta$ is small. Note that, unlike the regularized loss $\cLwd$, neither the prediction variation error $E_{\rm PV}$ nor the empirical loss $\cL$ is guaranteed to decay monotonically. This is consistent with behavior observed in practice; see also one-step characterization in the proof of Lemma~\ref{lem:individual_variance_decay}.

\subsubsection{Layerwise and cellwise characterization of prediction variation error}
\label{subsubsec:layerwise_prediction_variation}

As established in Theorem~\ref{thm:gen_NN_error}, the prediction variation error is controlled by the layerwise quantity $E_{\rm PV,layer}$. To understand how this bound is assembled, we now examine the prediction variation separately within each partition cell and at each layer. This refinement makes explicit how local approximate homogeneity and layer depth affect the decay, while the population cell masses determine how the local variations contribute to the global quantity. At the end of this section, we discuss two consequences of this refined analysis: the distinct dynamics of the bias and variance components, as well as the different roles of the inner and outer weights. 

To make this refinement precise, we begin by isolating the prediction variation associated with each cell--layer pair. For each $i=1,\cdots,N$ and $\ell=0,\cdots,L-1$, define
     \begin{align*}
        V_{\ell,\Omega_i}(\theta):=\EE_{\Omega_i}\|\theta_L\varphi_\ell(\theta_\ell;u_\ell(\theta_{\ell-1:0};x))-\theta_L\varphi_\ell(\theta_\ell;u_\ell(\theta_{\ell-1:0};x_i))\|^2.
     \end{align*}
The aggregated layerwise variation $V_{\ell}$ is then the population-mass-weighted sum of these cellwise quantities
\begin{align*}
    V_{\ell}(\theta)=\sum_{i=1}^N\left(\int_{\Omega_i}\pi(x)dx\right)\,  V_{\ell,\Omega_i}(\theta).
\end{align*}
With these cellwise and layerwise quantities, we can now characterize their evolution along training in the following lemma.

\begin{lemma}
    \label{lem:V_decay}
     Under Assumption~\ref{assump:model_assumptions}, \ref{assump:architecture_generalization}, and~\ref{assump:approximate_homogeneity}, with probability 1 over the joint distribution $\pi_N$, for $k\ge 0$,
	    \begin{align*}
	     V_{\ell,\Omega_i}(\theta^{K_0+k})\le&  \left(1-2\left(\sum_{r=0}^{\ell}s_{i,r,\rm in}+\ell+2\right)\eta\lambda \right)^{k}V_{\ell,\Omega_i}(\theta^{K_0})+2C_{{\rm ah},V_{\ell,\Omega_i}}(\ell+2)\epsilon_{\rm ah}\\
	    &\quad +\frac{C_{{\rm grad},V_{\ell,\Omega_i}}}{\lambda(\ell+2)}\max_{K_0\le j\le K_0+k}\|\nabla\cL(\theta^j)\| + \frac{S_{V_{\ell,\Omega_i}}}{2-\eta\,\mathsf{L}_\lambda}\cLwd(\theta^{K_0}), 
        \end{align*}
        where $s_{i,r,\rm in}\ge 0$ is the local approximate homogeneity degree on $\Omega_i$; $C_{{\rm ah},V_{\ell,\Omega_i}},C_{{\rm grad},V_{\ell,\Omega_i}},S_{V_{\ell,\Omega_i}}>0$ are some universal constants. 
    Moreover, let 
    \begin{align*}
        s_{r,\rm in}:=\min_{1\le i\le N}s_{i,r,\rm in}, \quad&C_{{\rm ah},V_\ell}:=\max_{1\le i\le N} C_{{\rm ah},V_{\ell,\Omega_i}},\\
         C_{{\rm grad},V_{\ell}}:=\max_{1\le i\le N}C_{{\rm grad},V_{\ell,\Omega_i}},\quad&S_{V_{\ell}}:=\max_{1\le i\le N}S_{V_{\ell,\Omega_i}}.
    \end{align*}
    Then
        \begin{align*}
	     V_{\ell}(\theta^{K_0+k})\le&  \left(1-2\left(\sum_{r=0}^{\ell}s_{r,\rm in}+\ell+2\right)\eta\lambda \right)^{k}V_{\ell}(\theta^{K_0}) +2C_{{\rm ah},V_\ell}(\ell+2)\,\epsilon_{\rm ah}\\
	    &\quad +\frac{C_{{\rm grad},V_{\ell}}}{\lambda(\ell+2)}\max_{K_0\le j\le K_0+k}\|\nabla\cL(\theta^j)\| + \frac{S_{V_{\ell}}}{2-\eta\,\mathsf{L}_\lambda}\cLwd(\theta^{K_0}).
	\end{align*}
	\end{lemma}
Lemma~\ref{lem:V_decay} shows that each cellwise quantity $V_{\ell,\Omega_i}$ and the aggregated quantity $V_\ell$ contract at rates determined by the local homogeneity degrees and weight decay. The remaining three terms have clear origins similar to Theorem~\ref{thm:gen_NN_error}: $2C_{{\rm ah},V_\ell}(\ell+2)\epsilon_{\rm ah}$ is the approximate homogeneity error, the gradient term vanishes as GD approaches a stationary point of $\cL$, and the $\cL_\lambda(\theta^{K_0})$ term decays as the regularized objective decreases. See the complete version in Lemma~\ref{lem:V_M_decay}.

The cellwise quantity $V_{\ell,\Omega_i}$ also admits an exact bias--variance decomposition. By adding and subtracting the cellwise mean of the $\ell$th layer contribution, we obtain
\begin{align*}
    V_{\ell,\Omega_i}(\theta)
    =&\underbrace{\EE_{\Omega_i}\left\|\theta_L\varphi_\ell(\theta_\ell;u_\ell(\theta_{\ell-1:0};x))-\EE_{\Omega_i}\theta_L\varphi_\ell(\theta_\ell;u_\ell(\theta_{\ell-1:0};x))\right\|^2}_{\text{variance }V_{\ell,\Omega_i,\rm var}(\theta)}\\
    &+\underbrace{\left\|\EE_{\Omega_i}\theta_L\varphi_\ell(\theta_\ell;u_\ell(\theta_{\ell-1:0};x))-\theta_L\varphi_\ell(\theta_\ell;u_\ell(\theta_{\ell-1:0};x_i))\right\|^2}_{\text{bias }M_{\ell,\Omega_i}(\theta)},
\end{align*}
where $V_{\ell,\Omega_i,\rm var}$ measures fluctuations around the cellwise mean, and $M_{\ell,\Omega_i}$ measures the squared discrepancy between this mean and the layer contribution at the input $x_i$. Weighting the cellwise decomposition by the population cell masses yields the layerwise decomposition $V_\ell=V_{\ell,\rm var}+M_\ell$. At both the cellwise and layerwise levels, the variance and bias components exhibit different dynamics. Their bounds have the same geometric contraction rate, but the guaranteed approximate homogeneity remainder for $M_\ell$ is twice that for $V_{\ell,\rm var}$ because the centered variance benefits from cancellation around the cellwise mean. Consequently, the bias component can retain a larger residual error. See Section~\ref{subsubsec:bias_variance_decay} in the appendix for details.

Additionally, inner and outer weights play different roles in the layerwise dynamics. Since $\theta_{\ell,\rm out}$ acts linearly after the nonlinear node, it creates no direct approximate homogeneity error in $V_{\ell,\Omega_i}$, whereas $\theta_{\ell,\rm in}$ acts inside $\sigma_\ell$ and contributes to the error controlled by $\epsilon_{\rm ah}$. Nevertheless, an earlier outer weight $\theta_{r,\rm out}$ with $r<\ell$, can still influence $V_{\ell,\Omega_i}$ through the hidden nodes of later layers. See Appendix~\ref{subsubsec:inner_outer_weights} for details.

\subsubsection{Necessary condition for generalization: layerwise variation thresholds}
\label{subsubsec:necessary_generalization}

In this section, we extract a layerwise necessary condition from the decay estimate and discuss its implications for layerwise performance. The condition identifies when the bound can guarantee a fixed relative decay of the layerwise variation. Its dependence on depth also helps interpret how later layers learn more complex components of the target function and why the decay of later layers may occur later in training.

The following corollary formalizes the resulting necessary condition.

\begin{corollary}[Necessary condition for generalization]
    \label{cor:necessary_condi_generalization}
 Under Assumption~\ref{assump:model_assumptions}, \ref{assump:architecture_generalization}, and~\ref{assump:approximate_homogeneity}, let $0<c_{K_0}\le1$. With probability 1 over the joint distribution $\pi_N$, if, for some $K_2\ge K_0$,
 $$ V_{\ell}(\theta^{K_2})\le c_{K_0} V_{\ell}(\theta^{K_0}),$$
 then the following condition must be satisfied
    $$  V_{\ell}(\theta^{K_0}) > \frac{2C_{{\rm ah},V_\ell}}{c_{K_0}} (\ell+2) \epsilon_{\rm ah}. $$
\end{corollary}
The corollary should be read as a necessary condition for the estimate in Lemma~\ref{lem:V_decay} to guarantee layerwise generalization. If the initial variation $V_\ell(\theta^{K_0})$ is below the approximate homogeneity error, then the upper bound cannot prove a relative decay by the factor $c_{K_0}$. Since this error scales like $(\ell+2)$, the lower bound becomes larger for later layers, which comes from the larger gap from exact contraction in Lemma~\ref{lem:V_decay}.

The depth-dependent lower bounds also allow later layers to retain larger variation even after the optimization terms have become small. This is consistent with the observation that shallower layers tend to learn less complex components of the target function, which may have smaller variation, whereas deeper layers tend to learn more complex components, which could have larger variation~\citep{chen2023which}. Note that this necessary condition does not claim that every trained network must exhibit strictly ordered layerwise errors.

The threshold also affects when decay can be guaranteed. Before the estimate can guarantee a fixed relative decay, the initial oscillation $V_\ell(\theta^{K_0})$ must dominate the approximate homogeneity error proportional to $(\ell+2)\epsilon_{\rm ah}$. Since this error grows with $\ell$, later layers require larger initial oscillation. Consequently, the guaranteed decay of $V_\ell$ may occur at a later stage for deeper layers \citep{raghu2017svcca,chen2023which}. As discussed above, this result provides a general qualitative interpretation and does not imply strict conclusions about layerwise behavior in neural networks; finer bounds would be needed to study these phenomena separately.

\subsubsection{Sufficient condition for generalization: grokking}
\label{subsubsec:sufficient_generalization}

In this section, we derive a sufficient condition for improved generalization through the delayed decay of $V_\ell$ and relate it to grokking. We show that a small approximate homogeneity error, a small empirical gradient and a small regularized objective in the neighbourhood of global minimizers, and a sufficiently long time in that region together guarantee the decay of $V_\ell$, which supports improved generalization. These conditions also describe a mechanism underlying grokking.

The following corollary formalizes these requirements and quantifies the decay of layerwise prediction variation error.

\begin{corollary}[Sufficient condition for generalization]
    \label{cor:sufficient_condi_generalization}
    Consider the same assumptions as Theorem~\ref{thm:convergence_to_global_min} and Assumption~\ref{assump:approximate_homogeneity}. Suppose
    \begin{enumerate}
        \item There exists a small enough $\epsilon_{\rm ah}>0$, $0<c_{K_0}\le1$, and $K_1\le K_0\le K_2$, s.t.,
    \begin{align*}
        \epsilon_{K_0}:= c_{K_0} V_{\ell}(\theta^{K_0})-2C_{{\rm ah},V_\ell}(\ell+2)\,\epsilon_{\rm ah} -\frac{C_{{\rm grad},V_{\ell}}}{\lambda(\ell+2)}\mathsf{L}_{\cM_0}\epsilon_\eta - \frac{S_{V_{\ell}}}{2-\eta\,\mathsf{L}_\lambda}\left(1-\eta\left(1-\frac{\eta\, \mathsf{L}_\lambda}{2}\right)\frac{\lambda}{\mathsf{L}_\cM \zeta^2+2}\right)^{K_0}\cLwd(\theta^{0})>0.
    \end{align*}
    \item The exit time $K_2$ from the neighbourhood of the global minimizer set is sufficiently large, i.e., 
    \begin{align*}
        K_2-K_0\ge
        \frac{\log\left(\epsilon_{K_0}/V_{\ell}(\theta^{K_0})\right)}
        {\log\left(1-2\left(\sum_{r=0}^{\ell}s_{r,\rm in}+\ell+2\right)\eta\lambda\right)}=\Omega\left( \frac{\log\left(V_{\ell}(\theta^{K_0})/\epsilon_{K_0}\right)}{2\ell\eta\lambda} \right).
    \end{align*}
    \end{enumerate}
    Then we have
    $$ V_{\ell}(\theta^{K_2})\le c_{K_0} V_{\ell}(\theta^{K_0}).$$
\end{corollary}
We first explain the terms in Corollary~\ref{cor:sufficient_condi_generalization}. The assumptions are inherited from the convergence theorem and the approximate homogeneity assumption: the convergence result ensures that, for $K_1\le K_0\le K_2$, GD stays in the neighbourhood of global minimizers and hence $\|\nabla\cL(\theta^j)\|$ is controlled by the scale $\mathsf{L}_{\cM_0}\epsilon_\eta$, while Assumption~\ref{assump:approximate_homogeneity} controls the architecture and partition error through $\epsilon_{\rm ah}$. The quantity $\epsilon_{K_0}$ is the remaining decay budget at time $K_0$: it starts from the target amount $c_{K_0}V_\ell(\theta^{K_0})$ and subtracts the three noncontractive terms in Lemma~\ref{lem:V_decay}, namely the approximate homogeneity error, the empirical-gradient contribution in the neighbourhood of global minimizers, and the residual contribution of the regularized objective after $K_0$ iterations. The condition $\epsilon_{K_0}>0$ means that these error terms are strictly smaller than the target reduction of $V_\ell$, or equivalently, the target reduction can not be smaller than these three errors. The second condition requires $K_2-K_0$ to be large enough so that the geometric contraction factor $\left(1-2(\sum_{r=0}^{\ell}s_{r,\rm in}+\ell+2)\eta\lambda\right)^{K_2-K_0}$ reduces the initial oscillation $V_\ell(\theta^{K_0})$ below this positive margin $\epsilon_{K_0}$. Therefore the conclusion $V_\ell(\theta^{K_2})\le c_{K_0}V_\ell(\theta^{K_0})$ follows by combining a positive margin at $K_0$ with enough subsequent contraction before the trajectory exits the neighbourhood of global minimizers.

This gives a sufficient condition for improved layerwise generalization. When it holds for every $V_\ell$ and for $V_{\rm pre}$, the overall prediction variation error also decreases. Once the data error is controlled by the partition and output data assumptions, and the optimization error by the convergence of the empirical objective, the network's within-cell oscillation is the remaining contribution. Corollary~\ref{cor:sufficient_condi_generalization} guarantees that this oscillation decreases by a prescribed factor, making the learned function more stable within each partition cell, which tightens the generalization bound.

We now relate this condition to grokking \citep{power2022grokking}. Grokking is a delayed generalization phenomenon in which the model first reaches a small training loss while the population or validation error remains large, and then generalizes sharply only after additional training. In the present decomposition, the optimization error has already become small, but the prediction variation error has not yet decayed sufficiently. In this sense, grokking is the ``worst case" of generalization within our framework: even after optimization succeeds, the oscillation terms require additional time to decay in the small-loss region. Easier generalization occurs when $V_\ell$ decays before or together with the optimization error, whereas grokking requires a longer time window after $K_0$.

Corollary~\ref{cor:sufficient_condi_generalization} can therefore be interpreted as a general characterization of the grokking phenomenon. More precisely, grokking occurs when optimization has already made $\cL$ and $\cL_\lambda$ small, but the network still needs a long time from $K_0$ to $K_2$ for the layerwise oscillation to contract. The sufficient condition identifies three mechanisms that make this delayed improvement possible: a small approximate homogeneity error, a small empirical gradient and a small regularized objective in the neighbourhood of global minimizers, and a sufficiently long residence time in that region.

\section{Discussions}

This paper develops a general framework for analyzing generalization dynamics by separating the effects of input data and their relationship to the true distribution through partition, the layerwise and blockwise architecture, and the optimization dynamics. Its main advantage is this generality: the decomposition identifies how these components enter the population error and can be adapted to a broad range of learning settings. The same generality also limits the specificity of its conclusions. In particular, the connections to loss reweighting, layerwise learning behavior, and grokking describe qualitative mechanisms and conditions rather than precise predictions for a given task or architecture. A natural extension is therefore to specialize the data partition, architectural assumptions, and local approximate homogeneity bounds to concrete settings, yielding sharper cellwise and layerwise estimates that can be compared directly with empirical behavior. The current framework provides a baseline and a starting point for such refined analyses.

\clearpage
\appendix
\section*{Appendix}
\section{Optimization}
In this section, we adapt the proof from \citet{wang2026convergencegradientdescentgeneral} to the GD dynamics with weight decay.

We first recall the setting defined in the main body. We consider the following family of neural networks:
\begin{align*}
    u_{0,i}&=x_i;\quad\\
    u_{\ell+1,i}&=u_{\ell,i}+{\varphi}_\ell(\theta_\ell;u_{\ell,i}),\forall\ell=0,\cdots,L-1;\quad \\
    f(\theta;x_i)&={\varphi}_L(\theta_L;u_{L,i}).
\end{align*}
and we consider the objective function
\begin{align*}
    \cL_{\lambda}(\theta)=\cL(\theta)+\frac{\lambda}{2}\|\theta\|^2=\frac{1}{N}\sum_{i=1}^N l(\theta;x_i,y_i)+\frac{\lambda}{2}\|\theta\|^2
\end{align*}
minimized by GD
\begin{align*}
    \theta^{k+1}=\theta^k-\eta\nabla_\theta\cLwd(\theta^k),
\end{align*}
where 
\begin{align*}
    \cL(\theta):=\cL(\theta;\{x_i,y_i\}_{i=1}^N)=\frac{1}{N}\sum_{i=1}^N l(\theta;x_i,y_i),\ \text{and }l(\theta;x_i,y_i)=\frac{1}{2}\|y_i-f(\theta;x_i)\|^2.
\end{align*}

\subsection{Dissipative condition on the regularized objective}
In this section, we prove the dissipativity of the regularized objective from the dissipative condition of the empirical loss in Lemma~\ref{lem:ip_nabla_cLwd_theta_lower_bound}.

\begin{lemma}
\label{lem:Lwd_dissipative}
    Under Assumption~\ref{assump:0_stationary_point}, for ${1-2\rho\lambda}>0$, $\cLwd(\theta)$ satisfies the following dissipative condition for all $\theta\in B_\bzero(R_{\lambda})\backslash\{\theta:\|\nabla\cL(\theta)\|\le\epsilon_\cL\}$,
\begin{align*}
    \ip{\nabla_\theta\cLwd(\theta)}{\theta}\ge -\rho_{\lambda}\|\nabla_\theta \cLwd(\theta)\|^2+\tau_{\lambda}\|\theta\|^2
\end{align*}
where $\rho_{\lambda}=\frac{\rho}{1-2\rho\lambda}$, $\tau_{\lambda}=\frac{1-\rho\lambda}{1-2\rho\lambda}\lambda$.
\end{lemma}

\begin{proof}
Consider
    \begin{align*}
    \ip{\nabla \cLwd (\theta)}{\theta}=\ip{\nabla\cL(\theta)+\lambda\theta}{\theta}=\ip{\nabla\cL(\theta)}{\theta}+\lambda \|\theta\|^2.
\end{align*}
Then 
\begin{align*}
   - \rho\|\nabla \cLwd(\theta)\|^2&=-\rho\|\nabla\cL(\theta)+\lambda\theta\|^2=-\rho\|\nabla\cL(\theta)\|^2-2\rho\lambda\ip{\nabla\cL(\theta)}{\theta}-\rho\lambda^2\|\theta\|^2\\
    &\le (1-2\rho\lambda)\ip{\nabla\cL(\theta)}{\theta}-\rho\lambda^2\|\theta\|^2\\
    &= (1-2\rho\lambda)\ip{\nabla\cL(\theta)}{\theta}+(1-2\rho\lambda)\lambda\|\theta\|^2-(1-\rho\lambda)\lambda\|\theta\|^2\\
    &= (1-2\rho\lambda) \ip{\nabla \cLwd (\theta)}{\theta}-(1-\rho\lambda)\lambda\|\theta\|^2
\end{align*}
where the first inequality follows from the dissipative condition of $\cL(\theta)$. Rearranging the above inequality and dividing by $1-2\rho\lambda>0$, we have
\begin{align*}
    \ip{\nabla \cLwd (\theta)}{\theta}\ge -\frac{\rho}{1-2\rho\lambda}\|\nabla \cLwd(\theta)\|^2+\frac{(1-\rho\lambda)\lambda}{1-2\rho\lambda}\|\theta\|^2.
\end{align*}
\end{proof}

\begin{lemma}
    \label{lem:complement_dissipativity}
    When $\lambda>0$, $\|\nabla\cL(\theta)\|\le \epsilon_\cL$ and $\frac{\lambda}{2}\|\theta\|^2\ge \frac{2\epsilon_\cL^2}{\lambda}$,
    \begin{align*}
        \ip{\nabla\cL(\theta)}{\theta}\ge -\frac{\lambda}{2}\|\theta\|^2,
    \end{align*}
    and thus
    \begin{align*}
        \ip{\nabla\cLwd(\theta)}{\theta}\ge \frac{\lambda}{2}\|\theta\|^2.
    \end{align*}
\end{lemma}
\begin{proof}
     When $\|\nabla\cL(\theta)\|\le \epsilon_\cL$, by Cauchy-Schwarz inequality,
    \begin{align*}
        \ip{\nabla\cL(\theta)}{\theta}\ge -\|\nabla\cL(\theta)\| \|\theta\|\ge -\epsilon_\cL\|\theta\|.
    \end{align*}
    Since $\frac{\lambda}{2}\|\theta\|^2\ge \frac{2\epsilon_\cL^2}{\lambda}$, we have
    \begin{align*}
        \|\theta\|\ge \frac{2\epsilon_\cL}{\lambda}. 
    \end{align*}
    Combining the above results, we have
    \begin{align*}
        \ip{\nabla\cL(\theta)}{\theta}\ge  -\epsilon_\cL\|\theta\|=-\frac{\epsilon_\cL}{\|\theta\|}\|\theta\|^2 \ge -\frac{\lambda}{2}\|\theta\|^2.
    \end{align*}
Thus
\begin{align*}
    \ip{\nabla\cLwd(\theta)}{\theta}=\ip{\nabla\cL(\theta)}{\theta}+\lambda\|\theta\|^2\ge \frac{\lambda}{2}\|\theta\|^2.
\end{align*}
\end{proof}

Combining the above two lemmas, we obtain the following dissipativity for all points in the ball of radius $R_\lambda$.
\begin{lemma}
    \label{lem:ip_nabla_cLwd_theta_lower_bound}
    Under Assumption~\ref{assump:0_stationary_point}, let $\lambda>0$, $1-2\rho\lambda>0$, and $\rho_\lambda\ge0$. Then for all $\theta\in B_\bzero(R_{\lambda})\backslash B_\bzero(2\epsilon_\cL/\lambda)$, we have
    \begin{align*}
     \ip{\nabla_\theta\cLwd(\theta)}{\theta}\ge -\rho_{\lambda}\|\nabla_\theta \cLwd(\theta)\|^2+\frac{\lambda}{2}\|\theta\|^2.
\end{align*}
\end{lemma}
\begin{proof}
    When $\|\nabla\cL(\theta)\|>\epsilon_\cL$, by Lemma~\ref{lem:Lwd_dissipative}, we have
\begin{align*}
    \ip{\nabla_\theta\cLwd(\theta)}{\theta}\ge -\rho_{\lambda}\|\nabla_\theta \cLwd(\theta)\|^2+\tau_{\lambda}\|\theta\|^2
\end{align*}
where $\rho_{\lambda}=\frac{\rho}{1-2\rho\lambda}$, $\tau_{\lambda}=\frac{1-\rho\lambda}{1-2\rho\lambda}\lambda$.

When $\|\nabla\cL(\theta)\|\le\epsilon_\cL$, by Lemma~\ref{lem:complement_dissipativity}, we have
\begin{align*}
     \ip{\nabla\cLwd(\theta)}{\theta}\ge \frac{\lambda}{2}\|\theta\|^2.
\end{align*}

Additionally, since $$\frac{1-\rho\lambda}{1-2\rho\lambda} >\frac{1}{2},$$ and $\rho_\lambda\ge0$,  
we have for all $\frac{\lambda}{2}\|\theta\|^2\ge \frac{2\epsilon_\cL^2}{\lambda}$,
\begin{align*}
     \ip{\nabla_\theta\cLwd(\theta)}{\theta}\ge -\rho_{\lambda}\|\nabla_\theta \cLwd(\theta)\|^2+\frac{\lambda}{2}\|\theta\|^2.
\end{align*}
\end{proof}

\subsection{Convergence under GD dynamics with weight decay}

\begin{proof}[Proof of Theorem~\ref{thm:general_convergence_loss_decay}]
    The proof follows from similar ideas as the proof of Theorem 1 in \citet{wang2026convergencegradientdescentgeneral}. The discussions of the measure-zero arguments and probability 1 argument remains the same. 

    We first note that the radius used in Lemma~\ref{lem:upper_bound_theta_sum_nabla_cL_along_GD} is the same as the simplified radius in the theorem statement. Since $\rho>0$, $0<\lambda<1/(2\rho)$, and $\delta>0$, we have $\rho_\lambda=\rho/(1-2\rho\lambda)>0$ and hence
    \begin{align*}
        \max\left\{\frac{4\rho_\lambda+2}{\delta},0\right\}
        =\frac{4\rho_\lambda+2}{\delta}.
    \end{align*}
    Therefore,
    \begin{align*}
        R_\lambda^2
        &=\|\theta^0\|^2+\frac{\frac{4\rho_{\lambda}+2}{\delta}}{1+\frac{\lambda}{2}\frac{4\rho_{\lambda}+2}{\delta}}\cL(\theta^0)\\
        &=\|\theta^0\|^2+\frac{4\rho_{\lambda}+2}{\delta+\frac{\lambda}{2}(4\rho_{\lambda}+2)}\cL(\theta^0)\\
        &=\|\theta^0\|^2+\frac{2(2\rho_{\lambda}+1)}{\delta+\lambda(2\rho_{\lambda}+1)}\cL(\theta^0).
    \end{align*}

    By Lemma~\ref{lem:lower_bound_gradient}, we have 
\begin{align*}
        \|\nabla\cLwd(\theta)\|^2\ge \frac{2\mu_{\lambda,\theta,X}}{N}\cLwd(\theta)
    \end{align*}

    By Lemma~\ref{lem:cLwd_poly_smooth}, $\cLwd(\theta)$ is poly-smooth. Therefore, 
     \begin{align*}
    \cLwd(\theta^{k+1})&\le \cLwd(\theta^{k})+\nabla \cLwd(\theta^{k})^\top(\theta^{k+1}-\theta^{k})+\frac{S_{\lambda,k}}{2}\|\theta^k-\theta^{k+1}\|^2\\
    &= \cLwd(\theta^{k})-\eta\left(1-\frac{\eta S_{\lambda,k}}{2}\right)\|\nabla\cLwd(\theta^k)\|^2\\
    &\le \cLwd(\theta^{k})-\eta\left(1-\frac{\eta \mathsf{L}_\lambda}{2}\right)\|\nabla\cLwd(\theta^k)\|^2\\
    & \le \left(1- \eta\left(1-\frac{\eta\, \mathsf{L}_\lambda}{2}\right)\frac{2\mu_{\lambda,k,X}}{N}\right)\cLwd(\theta^{k})\\
    &\le \prod_{s=0}^{k}\left(1- \eta\left(1-\frac{\eta\, \mathsf{L}_\lambda}{2}\right)\frac{2\mu_{\lambda,s,X}}{N}\right)\cLwd(\theta^{0})
\end{align*}
where the second inequality follows from Lemma~\ref{lem:upper_bound_theta_sum_nabla_cL_along_GD}.
\end{proof}

\subsubsection{Smoothness}

\begin{lemma}
    \label{lem:cLwd_poly_smooth}
    Under Assumption~\ref{assump:model_assumptions},
    \begin{enumerate}
        \item $\cL(\theta)$ is poly-smooth, i.e.,
    \begin{align*}
         \|\nabla_\theta\cL(\theta)-\nabla_\theta\cL(\theta')\|\le S(\|\theta_{\max,1}\|,\cdots,\|\theta_{\max,n_\theta}\|) \ \|\theta-\theta'\|
    \end{align*}
    where $S(\cdot)$ is the polynomial corresponding to the poly-smoothness of $\cL(\theta)$ and $\|\theta_{\max,j}\|=\max\{\|\theta_j\|,\|\theta_j'\|\}$. Additionally,
    \begin{align*}
        \cL(\theta')&\le \cL(\theta)+\nabla \cL(\theta)^\top(\theta'-\theta)+\frac{S(\|\theta_{\max,1}\|,\cdots,\|\theta_{\max,n_\theta}\|)}{2}\|\theta-\theta'\|^2.
    \end{align*}
    \item $\cLwd(\theta)$ is poly-smooth, i.e.,
    \begin{align*}
         \|\nabla_\theta\cLwd(\theta)-\nabla_\theta\cLwd(\theta')\|\le S_\lambda(\|\theta_{\max,1}\|,\cdots,\|\theta_{\max,n_\theta}\|) \ \|\theta-\theta'\|
    \end{align*}
    where 
    \begin{align*}
        S_\lambda(\|\theta_{\max,1}\|,\cdots,\|\theta_{\max,n_\theta}\|)=S(\|\theta_{\max,1}\|,\cdots,\|\theta_{\max,n_\theta}\|)+\lambda. 
    \end{align*}
    Similarly, 
    \begin{align*}
        \cLwd(\theta')&\le \cLwd(\theta)+\nabla \cLwd(\theta)^\top(\theta'-\theta)+\frac{S_\lambda (\|\theta_{\max,1}\|,\cdots,\|\theta_{\max,n_\theta}\|)}{2}\|\theta-\theta'\|^2.
    \end{align*}
    \end{enumerate} 
\end{lemma}
\begin{proof}
    The proof follows directly from Proposition 1, Lemma 2, and Lemma 12 in \cite{wang2025data}.
\end{proof}

Similar to Lemma 11 in \citet{wang2026convergencegradientdescentgeneral}, we obtain the following lemma to upper bound the sum of $\nabla\cLwd$ norm squares and $\theta^k$. 
\begin{lemma}
\label{lem:upper_bound_theta_sum_nabla_cL_along_GD}
Let $$\mathsf{L}_\lambda=S\left(\cdots,R_{\lambda}+\sup_{\theta\in B_\bzero(R_{\lambda})}\|\nabla_{\theta_i}\cLwd(\theta)\|,\cdots\right)+\lambda.$$
where $S(\cdot)$ is the polynomial corresponding to the poly-smoothness of $\cL(\theta)$ and
\begin{align*}
    R_{\lambda}=\sqrt{ \|\theta^0\|^2+\frac{\max\left\{ \frac{4\rho_{\lambda}+2}{\delta},0 \right\}}{1+\frac{\lambda}{2}\max\left\{ \frac{4\rho_{\lambda}+2}{\delta},0 \right\}}\cL(\theta^0)}.
\end{align*}
and let 
$$\eta\le\min\left\{ \frac{2-\delta}{\mathsf{L}_\lambda},1 \right\},\text{ for some }0<\delta<2.$$
We then have 
\begin{enumerate}
    \item The sum of gradient norm squares
    \begin{align}
\label{eqn:sum_grad_square}
    \sum_{j=0}^k\|\nabla\cLwd(\theta^j)\|^2\le \frac{2}{\eta(2-\eta\,\mathsf{L}_{\lambda})}(\cLwd(\theta^0)-\cLwd(\theta^{k+1})).
\end{align}
\item For all $k\ge 0$ and $\frac{\lambda}{2}\|\theta^k\|^2\ge \frac{2\epsilon_\cL^2}{\lambda}$, we have $$\theta^{k}\in  B_\bzero (R_{\lambda}), \text{ and }S_{\lambda,k}\le \mathsf{L}_\lambda,$$ 
where \begin{align*}
    S_{\lambda,j}&=S(\cdots,\max\{\|\theta_i^j\|,\|\theta_i^{j+1}\|\},\cdots)+\lambda,
\end{align*}
with $S(\cdot)$ being the polynomial corresponding to the poly-smoothness of $\cL(\theta)$.
    
\item Especially,
\begin{align*}
        \|\theta^{k+1}\|^2&\le \frac{\left(1-\eta\lambda\right)^{k+1}+\frac{\lambda}{2}\max\left\{ \frac{4\rho_{\lambda}+2\eta}{2-\eta\,\mathsf{L}_{\lambda}},0 \right\}}{1+\frac{\lambda}{2}\max\left\{ \frac{4\rho_{\lambda}+2\eta}{2-\eta\,\mathsf{L}_{\lambda}},0 \right\}}\|\theta^0\|^2+\frac{\max\left\{ \frac{4\rho_{\lambda}+2\eta}{2-\eta\,\mathsf{L}_{\lambda}},0 \right\}}{1+\frac{\lambda}{2}\max\left\{ \frac{4\rho_{\lambda}+2\eta}{2-\eta\,\mathsf{L}_{\lambda}},0 \right\}}\cL(\theta^0)
    \end{align*}
    
\end{enumerate}
   
\end{lemma}

\begin{proof}

Obviously, we have $\theta^0\in  B_\bzero(R_{\lambda})$. Suppose $\theta^0,\cdots,\theta^k \in  B_\bzero(R_{\lambda})$. Then we have for all $j\le k$,
\begin{align*}
    S_{\lambda,j}&=S(\cdots,\max\{\|\theta_i^j\|,\|\theta_i^{j+1}\|\},\cdots)+\lambda\\
    &\le S(\cdots,\|\theta_i^j\|+\eta\|\nabla_{\theta_i}\cLwd(\theta^j)\|,\cdots)+\lambda\\
    &\le S\left(\cdots,R_{\lambda}+\sup_{\theta\in B_\bzero(R_{\lambda})}\|\nabla_{\theta_i}\cLwd(\theta)\|,\cdots\right)+\lambda=\mathsf{L}_{\lambda},
\end{align*}
where $S(\cdot)$ is the polynomial corresponding to the poly-smoothness of $\cL(\theta)$. Then we have
\begin{align*}
        \eta\le\frac{2-\delta}{\mathsf{L}_{\lambda}}\le\frac{2-\delta}{{S}_{\lambda,j}}.
    \end{align*}

We then would like to show that $\theta^{k+1}\in  B_\bzero(R_{\lambda})$.
    \begin{align*}
    0\le \cLwd(\theta^{k+1})&\le \cLwd(\theta^{k})+\nabla \cLwd(\theta^{k})^\top(\theta^{k+1}-\theta^{k})+\frac{S_{\lambda,k}}{2}\|\theta^k-\theta^{k+1}\|^2\\
    &= \cLwd(\theta^{k})-\eta\left(1-\frac{\eta S_{\lambda,k}}{2}\right)\|\nabla\cLwd(\theta^k)\|^2\\
    &\le \cLwd(\theta^0)-\eta\sum_{j=0}^k\left(1-\frac{\eta S_{\lambda,j}}{2}\right)\|\nabla\cLwd(\theta^j)\|^2\\
    &\le \cLwd(\theta^0)-\frac{(2-\eta\,\mathsf{L}_{\lambda})}{2}\eta\sum_{j=0}^k\|\nabla\cLwd(\theta^j)\|^2,
\end{align*}
namely, we have
\begin{align}
    \sum_{j=0}^k\|\nabla\cLwd(\theta^j)\|^2\le \frac{2}{\eta(2-\eta\,\mathsf{L}_{\lambda})}(\cLwd(\theta^0)-\cLwd(\theta^{k+1})).
\end{align}

Then for $\frac{\lambda}{2}\|\theta^j\|^2\ge \frac{2\epsilon_\cL^2}{\lambda}$ where $j=0,\cdots,k$, we have
 \begin{align*}
                \|\theta^{k+1}\|^2
                &=\|\theta^{k}-\eta\nabla \cLwd(\theta^k)\|^2\\
        &=\|\theta^{k}\|^2-2\eta\nabla \cLwd(\theta^k)^\top \theta^{k}+\eta^2\|\nabla \cLwd(\theta^k)\|^2\\
        &\le \left(1-\eta{\lambda}\right) \|\theta^{k}\|^2+\left(2\eta\rho_{\lambda}+\eta^2\right)\|\nabla \cLwd(\theta^k)\|^2\\
        &\le \left(1-\eta\lambda\right)^{k+1}\|\theta^0\|^2+\left(2\eta\rho_{\lambda}+\eta^2\right)\sum_{j=0}^k \left(1-\eta\lambda\right)^{k-j} \|\nabla\cLwd(\theta^j)\|^2\\
        & \le \left(1-\eta\lambda\right)^{k+1}\|\theta^0\|^2+\left(2\eta\rho_{\lambda}+\eta^2\right)\sum_{j=0}^k  \|\nabla\cLwd(\theta^j)\|^2\\
        &\le \left(1-\eta\lambda\right)^{k+1}\|\theta^0\|^2 + \max\left\{ \frac{4\rho_{\lambda}+2\eta}{2-\eta\,\mathsf{L}_{\lambda}},0 \right\}(\cLwd(\theta^0)-\cLwd(\theta^{k+1}))\\
        &\le\left(\left(1-\eta\lambda\right)^{k+1}+\frac{\lambda}{2}\max\left\{ \frac{4\rho_{\lambda}+2\eta}{2-\eta\,\mathsf{L}_{\lambda}},0 \right\}\right)\|\theta^0\|^2 \\
        &\qquad+ \max\left\{ \frac{4\rho_{\lambda}+2\eta}{2-\eta\,\mathsf{L}_{\lambda}},0 \right\}\cL(\theta^0)-\frac{\lambda}{2}\max\left\{ \frac{4\rho_{\lambda}+2\eta}{2-\eta\,\mathsf{L}_{\lambda}},0 \right\}\|\theta^{k+1}\|^2
    \end{align*}
    where the first inequality follows from Lemma~\ref{lem:ip_nabla_cLwd_theta_lower_bound}; the second inequality follows from Gronwall's inequality. Then, combining $\|\theta^{k+1}\|^2$, we have
    \begin{align*}
        \|\theta^{k+1}\|^2&\le \frac{\left(1-\eta\lambda\right)^{k+1}+\frac{\lambda}{2}\max\left\{ \frac{4\rho_{\lambda}+2\eta}{2-\eta\,\mathsf{L}_{\lambda}},0 \right\}}{1+\frac{\lambda}{2}\max\left\{ \frac{4\rho_{\lambda}+2\eta}{2-\eta\,\mathsf{L}_{\lambda}},0 \right\}}\|\theta^0\|^2+\frac{\max\left\{ \frac{4\rho_{\lambda}+2\eta}{2-\eta\,\mathsf{L}_{\lambda}},0 \right\}}{1+\frac{\lambda}{2}\max\left\{ \frac{4\rho_{\lambda}+2\eta}{2-\eta\,\mathsf{L}_{\lambda}},0 \right\}}\cL(\theta^0)\\
        &\le \left(1-\frac{\eta\lambda}{1+\frac{\lambda}{2}\max\left\{ \frac{4\rho_{\lambda}+2\eta}{2-\eta\,\mathsf{L}_{\lambda}},0 \right\}}\right)\|\theta^0\|^2+\frac{\max\left\{ \frac{4\rho_{\lambda}+2\eta}{2-\eta\,\mathsf{L}_{\lambda}},0 \right\}}{1+\frac{\lambda}{2}\max\left\{ \frac{4\rho_{\lambda}+2\eta}{2-\eta\,\mathsf{L}_{\lambda}},0 \right\}}\cL(\theta^0)\\
        &\le \|\theta^0\|^2+\frac{\max\left\{ \frac{4\rho_{\lambda}+2}{\delta},0 \right\}}{1+\frac{\lambda}{2}\max\left\{ \frac{4\rho_{\lambda}+2}{\delta},0 \right\}}\cL(\theta^0)\\
        &= R_{\lambda}^2
    \end{align*}
    where the second inequality follows from taking $k=0$, and the last inequality follows from $0<\eta\le 1$ and $2-\eta\mathsf{L}_\lambda\ge \delta$. 
\end{proof}

\subsubsection{Lower bound of gradient}

\begin{lemma}[lower bound of $\|\nabla\cLwd\|^2$] 
\label{lem:lower_bound_gradient}
Let $0<\lambda<\frac{1}{2\rho}$. Under Assumption~\ref{assump:data}, Assumption~\ref{assump:0_stationary_point}, and Assumption~\ref{assump:model_assumptions}, for all $\theta\in B_\bzero(R_{\lambda})$ satisfying either $\|\nabla\cL(\theta)\|>\epsilon_{\cL}$ or $\frac{\lambda}{2}\|\theta\|^2\ge \frac{2\epsilon_\cL^2}{\lambda}$,
    \begin{align*}
        \|\nabla\cLwd(\theta)\|^2\ge \frac{2\mu_{\lambda,\theta,X}}{N}\cLwd(\theta)
    \end{align*}
     where
     \begin{align*}
     \mu_{\lambda,\theta,X}:=
     \begin{cases}
     \min\left\{(1-2\rho\lambda){\mu_{\mathrm{low},\theta,X}},\lambda N\right\}, & \|\nabla\cL(\theta)\|>\epsilon_{\cL},\\
     \frac{N\epsilon_\cL^2}{2\sup_{\tilde\theta\in B_\bzero(R_{\lambda})}\cLwd(\tilde\theta)}, & \|\nabla\cL(\theta)\|\le\epsilon_{\cL}\text{ and }\frac{\lambda}{2}\|\theta\|^2\ge \frac{2\epsilon_\cL^2}{\lambda}.
     \end{cases}
     \end{align*}
     Additionally, when $\|\nabla\cL(\theta)\|>\epsilon_{\cL}$,	
 	     \begin{align*}
 	     \|\nabla\cLwd(\theta)\|^2\ge \lambda^2\|\theta\|^2 .
 	     \end{align*}
     When $\|\nabla\cL(\theta)\|\le\epsilon_{\cL}$ and $\frac{\lambda}{2}\|\theta\|^2\ge \frac{2\epsilon_\cL^2}{\lambda}$,
     \begin{align*}
     \|\nabla\cLwd(\theta)\|^2\ge \frac{\lambda^2}{4}\|\theta\|^2 .
     \end{align*}

\end{lemma}

\begin{proof}
We first consider the case $\|\nabla\cL(\theta)\|>\epsilon_\cL$. The existing lower bound gives
    \begin{align*}
        \|\nabla\cLwd(\theta)\|^2&=\|\nabla\cL(\theta)\|^2+2\lambda\ip{\nabla\cL(\theta)}{\theta}+\lambda^2\|\theta\|^2\\
        &\ge(1-2\rho\lambda) \|\nabla\cL(\theta)\|^2+\lambda^2\|\theta\|^2\\
        &\ge (1-2\rho\lambda)\frac{2\mu_{\mathrm{low},\theta,X}}{N} \cL(\theta)+2\lambda\cdot\frac{\lambda}{2}\|\theta\|^2\\
        &\ge \frac{2\mu_{\lambda,\theta,X}}{N}\cLwd(\theta)
    \end{align*}
    where $ \mu_{\lambda,\theta,X}=\min\left\{(1-2\rho\lambda){\mu_{\mathrm{low},\theta,X}} , \lambda N\right\}$. The first inequality follows from the dissipative condition of $\cL$; the second inequality follows from Lemma 12 in \citet{wang2026convergencegradientdescentgeneral}. Moreover, in the same case,
    \begin{align*}
		\|\nabla\cLwd(\theta)\|^2&=\|\nabla\cL(\theta)+\lambda\theta\|^2=\|\nabla\cL(\theta)\|^2+2\lambda\ip{\nabla\cL(\theta)}{\theta}+\lambda^2\|\theta\|^2\\
		&\ge (1-2\rho\lambda)\|\nabla\cL(\theta)\|^2+\lambda^2\|\theta\|^2\\
		&\ge \lambda^2\|\theta\|^2 
	\end{align*}
	where $1-2\rho\lambda>0$.
    
We next consider the case $\|\nabla\cL(\theta)\|\le\epsilon_\cL$ and $\frac{\lambda}{2}\|\theta\|^2\ge \frac{2\epsilon_\cL^2}{\lambda}$. By Lemma~\ref{lem:complement_dissipativity},
    \begin{align*}
        \ip{\nabla\cLwd(\theta)}{\theta}\ge \frac{\lambda}{2}\|\theta\|^2.
    \end{align*}
    Hence, by Cauchy-Schwarz,
    \begin{align*}
        \|\nabla\cLwd(\theta)\|\|\theta\|\ge \ip{\nabla\cLwd(\theta)}{\theta}\ge \frac{\lambda}{2}\|\theta\|^2,
    \end{align*}
    and therefore
    \begin{align*}
        \|\nabla\cLwd(\theta)\|^2\ge \frac{\lambda^2}{4}\|\theta\|^2\ge \epsilon_\cL^2.
    \end{align*}
    Since $\theta\in B_\bzero(R_\lambda)$, we also have $\cLwd(\theta)\le \sup_{\tilde\theta\in B_\bzero(R_{\lambda})}\cLwd(\tilde\theta)$. Thus
    \begin{align*}
        \|\nabla\cLwd(\theta)\|^2\ge \epsilon_\cL^2\ge \frac{\epsilon_\cL^2}{\sup_{\tilde\theta\in B_\bzero(R_{\lambda})}\cLwd(\tilde\theta)}\cLwd(\theta)=\frac{2}{N}\cdot \frac{N\epsilon_\cL^2}{2\sup_{\tilde\theta\in B_\bzero(R_{\lambda})}\cLwd(\tilde\theta)}\cLwd(\theta).
    \end{align*}
Combining the two cases gives the claimed lower bound for all such $\theta\in B_\bzero(R_\lambda)$.
\end{proof}

\subsection{Convergence to a neighbourhood of global minimum}

\begin{proof}[Proof of Theorem~\ref{thm:convergence_to_global_min}]

We only need to show that GD will enter the neighbourhood of the global minimizer.

From \begin{align*}
        \eta<\frac{\|\theta^0
        \|-2\epsilon_\cL/\lambda}{\sup_{\theta\in B_\bzero(R_{\lambda})}\|\nabla\cLwd(\theta)\|},
    \end{align*}
we know that
\begin{align*}
    \|\theta^0\|>2\epsilon_\cL/\lambda+\epsilon_\eta.
\end{align*}

    By Corollary~\ref{cor:near_global_min_grad_lower_bound}, when $\|\theta\|\ge\frac{2\epsilon_\cL}{\lambda}$, we have
        \begin{align*}
        \|\nabla\cLwd(\theta)\|^2&\ge \frac{\lambda}{\mathsf{L}_\cM \zeta^2+2}\cLwd(\theta).
    \end{align*}
    Indeed, by the definition of $\zeta$, every $\theta\in B_\bzero(R_\lambda)$ with $\|\theta\|\ge 2\epsilon_\cL/\lambda$ satisfies $\zeta^2\frac{\lambda}{2}\|\theta\|^2\ge\operatorname{dist}(\theta,\cM)^2$. Thus both conditions of Corollary~\ref{cor:near_global_min_grad_lower_bound} hold.
    Therefore, by Lemma~\ref{lem:cLwd_poly_smooth}, when $\|\theta\|\ge\frac{2\epsilon_\cL}{\lambda}$,
     \begin{align}
        \label{eqn:cLwd_exp_decay_global_min}
    \cLwd(\theta^{k+1})&\le \cLwd(\theta^{k})+\nabla \cLwd(\theta^{k})^\top(\theta^{k+1}-\theta^{k})+\frac{S_{\lambda,k}}{2}\|\theta^k-\theta^{k+1}\|^2\notag\\
    &= \cLwd(\theta^{k})-\eta\left(1-\frac{\eta S_{\lambda,k}}{2}\right)\|\nabla\cLwd(\theta^k)\|^2\notag\\
    &\le \cLwd(\theta^{k})-\eta\left(1-\frac{\eta \mathsf{L}_\lambda}{2}\right)\|\nabla\cLwd(\theta^k)\|^2\notag\\
    & \le \left(1- \eta\left(1-\frac{\eta\, \mathsf{L}_\lambda}{2}\right)\frac{\lambda}{\mathsf{L}_\cM \zeta^2+2}\right)\cLwd(\theta^{k})\notag\\
    &\le \left(1- \eta\left(1-\frac{\eta\, \mathsf{L}_\lambda}{2}\right)\frac{\lambda}{\mathsf{L}_\cM \zeta^2+2}\right)^{k+1}\cLwd(\theta^{0})
\end{align}

Next, consider an $\epsilon_\eta$ neighbourhood of $\cM_0$
\begin{align*}
    \cM_{0,\epsilon_\eta}=\{\theta\in B_\bzero(R_{\lambda}) |\operatorname{dist}(\theta,\cM_0)\le \epsilon_\eta\}.
\end{align*}

\textbf{Step 1:} we would like to show that $\theta^k$ is guaranteed to leave $\cB_1\backslash\cM_{0,\epsilon_\eta}$.

By Assumption~\ref{assump:global_min}, every $\theta^*\in\cM_0$ satisfies $$\|\theta^*\|\ge 2\epsilon_\cL/\lambda+\epsilon_\eta.$$ For any $\theta\in\cB_1\backslash\cM_{0,\epsilon_\eta}$, the radial path from $\bzero\in\cB_2$ to $\theta\in\cB_1$ crosses some $\theta^*\in\cM_0$. Since $\operatorname{dist}(\theta,\cM_0)>\epsilon_\eta$, we have $$\|\theta\|=\|\theta^*\|+\|\theta-\theta^*\|>2\epsilon_\cL/\lambda+2\epsilon_\eta.$$ Therefore, defining
\begin{align*}
    M_{\rm out}:=\inf_{\theta\in\cB_1\backslash\cM_{0,\epsilon_\eta}}\|\theta\|,
\end{align*}
we have $$M_{\rm out}\ge 2\epsilon_\cL/\lambda+2\epsilon_\eta.$$
Hence, consider~\eqref{eqn:cLwd_exp_decay_global_min}.
If $\theta^0\in\cM_{0,\epsilon_\eta}$, then GD has already entered the desired neighbourhood. Otherwise, suppose for contradiction that $\theta^j\in\cB_1\backslash\cM_{0,\epsilon_\eta}$ for $j=0,\ldots,K$, where
\begin{align*}
K:=1+\left\lceil\frac{\log\left(\frac{\lambda M_{\rm out}^2}{2\cLwd(\theta^0)}\right)}{\log\left(1-\eta\left(1-\frac{\eta\mathsf{L}_\lambda}{2}\right)\frac{\lambda}{\mathsf{L}_\cM\zeta^2+2}\right)}\right\rceil=\cO\left(1+\frac{\mathsf{L}_\cM\zeta^2+2}{\delta\eta\lambda}\log\left(\frac{2\cLwd(\theta^0)}{\lambda M_{\rm out}^2}\right)\right).
\end{align*}
Then~\eqref{eqn:cLwd_exp_decay_global_min} gives
\begin{align*}
    \cLwd(\theta^K)<\frac{\lambda}{2}M_{\rm out}^2.
\end{align*}
On the other hand, since $\cL(\theta^K)\ge0$ and $\theta^K\in\cB_1\backslash\cM_{0,\epsilon_\eta}$,
\begin{align*}
    \cLwd(\theta^K)\ge\frac{\lambda}{2}\|\theta^K\|^2\ge\frac{\lambda}{2}M_{\rm out}^2,
\end{align*}
which is a contradiction. Therefore, GD leaves $\cB_1\backslash\cM_{0,\epsilon_\eta}$ in finite time.

\textbf{Step 2:} we would like to show that $\theta$ will first enter $\cM_{0,\epsilon_\eta}$ before entering $\cB_2\backslash\cM_{0,\epsilon_\eta}$.

It suffices to show that the change in $\theta$ at each step does not exceed $\epsilon_\eta$
\begin{align*}
    \|\theta^{k+1}-\theta^k\|=\eta\|\nabla\cLwd(\theta^k)\|\le \eta \sup_{\theta\in B_\bzero(R_{\lambda})}\|\nabla\cLwd(\theta)\|
\end{align*}
Therefore, when $\epsilon_\eta\ge \eta \sup_{\theta\in B_\bzero(R_{\lambda})}\|\nabla\cLwd(\theta)\|$
\begin{align*}
     \|\theta^{k+1}-\theta^k\|\le\epsilon_\eta
\end{align*}
i.e., GD has to enter $\cM_{0,\epsilon_\eta}$ first in order to leave $\cB_1\backslash\cM_{0,\epsilon_\eta}$.
Indeed, suppose one GD step has endpoints $\theta^k\in\cB_1\backslash\cM_{0,\epsilon_\eta}$ and $\theta^{k+1}\in\cB_2\backslash\cM_{0,\epsilon_\eta}$. Since $\cB_1$ and $\cB_2$ form a separation of $B_\bzero(R_\lambda)\backslash\cM_0$, the line segment between $\theta^k$ and $\theta^{k+1}$ intersects $\cM_0$. Consequently,
\begin{align*}
    \operatorname{dist}(\theta^k,\cM_0)\le\|\theta^{k+1}-\theta^k\|\le\epsilon_\eta,
\end{align*}
which contradicts $\theta^k\notin\cM_{0,\epsilon_\eta}$. Therefore, the first iterate that leaves $\cB_1\backslash\cM_{0,\epsilon_\eta}$ must enter $\cM_{0,\epsilon_\eta}$.

Define
\begin{align*}
    K_1&:=\min\{k\ge0:\theta^k\in\cM_{0,\epsilon_\eta}\},\\
    K_2&:=\sup\{K\in\mathbb{N}:K\ge K_1,\ \theta^k\in\cM_{0,\epsilon_\eta}\text{ for every }K_1\le k\le K\}.
\end{align*}
The argument in Step 1 shows that $K_1<\infty$, while $K_2$ may be infinite. Moreover, every $\theta^*\in\cM_0$ satisfies $\|\theta^*\|\ge2\epsilon_\cL/\lambda+\epsilon_\eta$. For every $\theta\in\cM_{0,\epsilon_\eta}$, there exists some $\theta^*\in\cM_0$ such that $\|\theta-\theta^*\|\le\epsilon_\eta$. Hence,
\begin{align*}
    \|\theta\|\ge\|\theta^*\|-\|\theta-\theta^*\|\ge\frac{2\epsilon_\cL}{\lambda}.
\end{align*}
Together with the corresponding lower bound on $\|\theta^k\|$ before $K_1$, this shows that~\eqref{eqn:cLwd_exp_decay_global_min} holds for every $k\le K_2$.

For each $\theta\in\cM_{0,\epsilon_\eta}$, by the compactness of $\cM_0$, let $$\theta^*=\arg\min_{\theta'\in\cM_0} \|\theta-\theta'\|.$$ 
Then, by the poly-smoothness of $\cL$ in Lemma~\ref{lem:cLwd_poly_smooth}, for $\theta\in\cM_{0,\epsilon_\eta}$,
\begin{align*}
    \cL(\theta)\le \cL(\theta^*)+\frac{1}{2}S_{\theta,\theta^*}\ \|\theta-\theta^*\|^2=\frac{1}{2}S_{\theta,\theta^*}\operatorname{dist}(\theta,\cM_0)^2\le \frac{1}{2}\mathsf{L}_{\cM_{0}} \epsilon_\eta^2
\end{align*}
where $$\mathsf{L}_{\cM_{0}}=\sup_{\theta\in B_\bzero(R_{\lambda})}\max_{\theta^*\in\cM_0} S(\cdots,\max\{\|\theta_i\|,\|\theta^*_i\|\},\cdots).$$

Also, by Lemma~\ref{lem:cLwd_poly_smooth}, we have
\begin{align*}
     \|\nabla\cL(\theta)\|= \|\nabla\cL(\theta)-\nabla\cL(\theta^*)\|\le S_{\theta,\theta^*}\ \|\theta-\theta^*\|\le \mathsf{L}_{\cM_{0}}\epsilon_\eta.
\end{align*}
Since $\epsilon_\eta=\eta\sup_{\theta\in B_\bzero(R_\lambda)}\|\nabla\cLwd(\theta)\|$ and $\mathsf{L}_{\cM_0}\le\mathsf{L}_\cM$, where both suprema are independent of $\eta$, we conclude that $\cL(\theta^k)=\cO(\eta^2)$ and $\|\nabla\cL(\theta^k)\|=\cO(\eta)$ for every $K_1\le k\le K_2$.
\end{proof}

\subsubsection{Supplementary results}
    Below are some supplementary results used in the proof above, including the detailed upper bound and lower bound of $\nabla\cLwd$. 

\begin{proposition}
    \label{prop:nabla_upper_bounded_by_funct}
    If the function $\psi(w)$ is $\mathsf{L}$-Lipschitz smooth in $\Omega$, $w-\frac{1}{\mathsf{L}}\nabla\psi(w)\in\Omega$ for all $w\in\Omega$, and $w^*$ is a global minimizer of $\psi$ in $\Omega$ with $\nabla\psi(w^*)=0$, then
    \begin{align*}
        \|\nabla\psi(w)\|^2\le 2\mathsf{L} (\psi(w)-\psi(w^*)),\text{ for all }w\in\Omega.
    \end{align*}
\end{proposition}
\begin{proof}
    By Lipschitz smoothness, for any $w\in\Omega$, let $\tilde{w}=w-\frac{1}{\mathsf{L}}\nabla\psi(w)$. Then
    \begin{align*}
        \psi(\tilde{w})&\le \psi(w)+\nabla\psi(w)^\top(\tilde{w}-w)+\frac{\mathsf{L}}{2}\|\tilde{w}-w\|^2\\
        &=\psi(w)-\frac{1}{\mathsf{L}}\|\nabla\psi(w)\|^2+\frac{1}{2\mathsf{L}}\|\nabla\psi(w)\|^2\\
        &=\psi(w)-\frac{1}{2\mathsf{L}}\|\nabla\psi(w)\|^2.
    \end{align*}
    Since $w^*$ is a global minimizer of $\psi$ in $\Omega$, we have $\psi(w^*)\le \psi(\tilde{w})$. Therefore,
    \begin{align*}
        \psi(w^*)\le \psi(w)-\frac{1}{2\mathsf{L}}\|\nabla\psi(w)\|^2,
    \end{align*}
    which implies
    \begin{align*}
        \|\nabla\psi(w)\|^2\le 2\mathsf{L}(\psi(w)-\psi(w^*)).
    \end{align*}
\end{proof}

\begin{lemma}
\label{lem:near_local_min_grad_lower_bound}
        Suppose Assumption~\ref{assump:0_stationary_point}, Assumption~\ref{assump:model_assumptions}, and the dissipative condition in Theorem~\ref{thm:general_convergence_loss_decay} holds. Let $0<\lambda<\frac{1}{2\rho}$. Consider the case when $\theta\in B_\bzero(R_{\lambda})$ is close to some local minimizer $\theta^*_{\rm loc}\in B_\bzero(R_\lambda)$ of $\cL(\theta)$, i.e., $\nabla\cL(\theta^*_{\rm loc})=\bzero$. Define $$\mathsf{L}_{\theta^*_{\rm loc}}=\sup_{\theta\in B_\bzero(R_\lambda)} S(\cdots,\max\{\|\theta_i\|,\|\theta^*_{{\rm loc},i}\|\},\cdots).$$ 
        There exists some $\epsilon_{\theta^*_{\rm loc}}>\frac{\sqrt{2}\epsilon_\cL}{\sqrt{\lambda}}$, s.t., when $$\theta\in\left\{\theta\,\bigg\vert\,\frac{\lambda}{2}\|\theta\|^2\ge \epsilon_{\theta^*_{\rm loc}}^2\text{ and }\|\theta-\theta^*_{\rm loc}\|\le \zeta\epsilon_{\theta^*_{\rm loc}}\right\}$$ for some universal constant $\zeta>0$, we have:
    \begin{enumerate}
        \item $\|\nabla\cL(\theta)\|\le \zeta\mathsf{L}_{\theta^*_{\rm loc}}\epsilon_{\theta^*_{\rm loc}}$
        \item When $\|\nabla\cL(\theta)\|>\epsilon_\cL$,
        \begin{align*}
        \|\nabla\cLwd(\theta)\|^2&\ge \frac{2\lambda \epsilon_{\theta^*_{\rm loc}}^2}{{\cL(\theta^*_{\rm loc})+\frac{\mathsf{L}_{\theta^*_{\rm loc}}\zeta^2\epsilon_{\theta^*_{\rm loc}}^2}{2}+\epsilon_{\theta^*_{\rm loc}}^2}}\cLwd(\theta).
    \end{align*}
    \item When $\|\nabla\cL(\theta)\|\le\epsilon_\cL$,
    \begin{align*}
        \|\nabla\cLwd(\theta)\|^2&\ge   \frac{\frac{\lambda}{2}\epsilon_{\theta^*_{\rm loc}}^2}{{\cL(\theta^*_{\rm loc})+\frac{\mathsf{L}_{\theta^*_{\rm loc}}\zeta^2\epsilon_{\theta^*_{\rm loc}}^2}{2}+\epsilon_{\theta^*_{\rm loc}}^2}}\cLwd(\theta)
    \end{align*}
    \end{enumerate}
    
\end{lemma}
\begin{proof}

By Lemma~\ref{lem:cLwd_poly_smooth},
\begin{align*}
    \cL(\theta)\le \cL(\theta^*_{\rm loc})+\frac{1}{2}S_{\theta,\theta^*_{\rm loc}}\ \|\theta-\theta^*_{\rm loc}\|^2
\end{align*}
where $S_{\theta,\theta^*_{\rm loc}}=S(\|\theta_{\max,1}\|,\cdots,\|\theta_{\max,n_\theta}\|) $ with $\|\theta_{\max,i}\|=\max\{\|\theta_i\|,\|\theta^*_{{\rm loc},i}\|\}$. 

To get rid of the dependency on $\theta$, for all $\theta\in B_\bzero(R_{\lambda})$, we have
    \begin{align*}
        S_{\theta,\theta^*_{\rm loc}}=S(\cdots,\max\{\|\theta_i\|,\|\theta^*_{{\rm loc},i}\|\},\cdots)\le \mathsf{L}_{\theta^*_{\rm loc}}=\sup_{\theta\in B_\bzero(R_\lambda)} S(\cdots,\max\{\|\theta_i\|,\|\theta^*_{{\rm loc},i}\|\},\cdots).
    \end{align*}

Then consider 
$$\theta\in \cA_\epsilon=\{\theta\,|\,\|\theta-\theta^*_{\rm loc}\|\le \epsilon\}.$$

Then we have
\begin{align}
    &\cL(\theta)\le \cL(\theta^*_{\rm loc})+\frac{\mathsf{L}_{\theta^*_{\rm loc}}\epsilon^2}{2}, \label{eqn_proof:cL_bound_min_value_epsilon}
\end{align}
By Lemma~\ref{lem:cLwd_poly_smooth}, we also have
\begin{align*}
    \|\nabla\cL(\theta)\|= \|\nabla\cL(\theta)-\nabla\cL(\theta^*_{\rm loc})\|\le S_{\theta,\theta^*_{\rm loc}}\ \|\theta-\theta^*_{\rm loc}\|\le \mathsf{L}_{\theta^*_{\rm loc}}\epsilon.
\end{align*}

\textbf{Case 1: when $\|\nabla\cL\|>\epsilon_\cL$.}

By Lemma~\ref{lem:lower_bound_gradient}, the gradient can be lower bounded
    \begin{align*}
        \|\nabla \cLwd(\theta)\|^2\ge {\lambda^2}\|\theta\|^2.
    \end{align*}

Then consider
    \begin{align*}
        \frac{{\lambda^2}\|\theta\|^2}{{\cL(\theta^*_{\rm loc})+\frac{\mathsf{L}_{\theta^*_{\rm loc}}\epsilon^2}{2}+\frac{\lambda}{2}\|\theta\|^2}}=2{\lambda}-\frac{2{\lambda}\left(\cL(\theta^*_{\rm loc})+\frac{\mathsf{L}_{\theta^*_{\rm loc}}\epsilon^2}{2}\right)}{{\cL(\theta^*_{\rm loc})+\frac{\mathsf{L}_{\theta^*_{\rm loc}}\epsilon^2}{2}+\frac{\lambda}{2}\|\theta\|^2}}.
    \end{align*}

    Then for all $\frac{\lambda}{2}\|\theta\|^2\ge \epsilon_{\theta^*_{\rm loc}}^2$, and for some $0<\epsilon\le \zeta\epsilon_{\theta^*_{\rm loc}}$ with $\zeta>0$,
    \begin{align*}
         \frac{{\lambda^2}\|\theta\|^2}{{\cL(\theta^*_{\rm loc})+\frac{\mathsf{L}_{\theta^*_{\rm loc}}\epsilon^2}{2}+\frac{\lambda}{2}\|\theta\|^2}}&=2{\lambda}-\frac{2{\lambda}\left(\cL(\theta^*_{\rm loc})+\frac{\mathsf{L}_{\theta^*_{\rm loc}}\epsilon^2}{2}\right)}{{\cL(\theta^*_{\rm loc})+\frac{\mathsf{L}_{\theta^*_{\rm loc}}\epsilon^2}{2}+\frac{\lambda}{2}\|\theta\|^2}}\\
        &\ge2{\lambda}-\frac{2{\lambda}\left(\cL(\theta^*_{\rm loc})+\frac{\mathsf{L}_{\theta^*_{\rm loc}}\epsilon^2}{2}\right)}{{\cL(\theta^*_{\rm loc})+\frac{\mathsf{L}_{\theta^*_{\rm loc}}\epsilon^2}{2}+\epsilon_{\theta^*_{\rm loc}}^2}}\\
        &\ge\frac{2\lambda \epsilon_{\theta^*_{\rm loc}}^2}{{\cL(\theta^*_{\rm loc})+\frac{\mathsf{L}_{\theta^*_{\rm loc}}\zeta^2\epsilon_{\theta^*_{\rm loc}}^2}{2}+\epsilon_{\theta^*_{\rm loc}}^2}}
    \end{align*}
    where the first inequality follows from $\frac{\lambda}{2}\|\theta\|^2\ge \epsilon_{\theta^*_{\rm loc}}^2$, and the last inequality follows from $\epsilon\le \zeta\epsilon_{\theta^*_{\rm loc}}$.
    Hence combining the above inequalities, we have, for all $\theta\in\cA_\epsilon$ and $\frac{\lambda}{2}\|\theta\|^2\ge \epsilon_{\theta^*_{\rm loc}}^2$,
    \begin{align*}
        \|\nabla\cLwd(\theta)\|^2&\ge  \frac{2\lambda \epsilon_{\theta^*_{\rm loc}}^2}{{\cL(\theta^*_{\rm loc})+\frac{\mathsf{L}_{\theta^*_{\rm loc}}\zeta^2\epsilon_{\theta^*_{\rm loc}}^2}{2}+\epsilon_{\theta^*_{\rm loc}}^2}}\left( \cL(\theta^*_{\rm loc})+\frac{\mathsf{L}_{\theta^*_{\rm loc}}\epsilon^2}{2}+\frac{\lambda}{2}\|\theta\|^2 \right)\\
        &\ge \frac{2\lambda \epsilon_{\theta^*_{\rm loc}}^2}{{\cL(\theta^*_{\rm loc})+\frac{\mathsf{L}_{\theta^*_{\rm loc}}\zeta^2\epsilon_{\theta^*_{\rm loc}}^2}{2}+\epsilon_{\theta^*_{\rm loc}}^2}}\left( \cL(\theta)+\frac{\lambda}{2}\|\theta\|^2 \right)\\
        &= \frac{2\lambda \epsilon_{\theta^*_{\rm loc}}^2}{{\cL(\theta^*_{\rm loc})+\frac{\mathsf{L}_{\theta^*_{\rm loc}}\zeta^2\epsilon_{\theta^*_{\rm loc}}^2}{2}+\epsilon_{\theta^*_{\rm loc}}^2}}\cLwd(\theta)
    \end{align*}
    where the second inequality follows from \eqref{eqn_proof:cL_bound_min_value_epsilon}.

    \textbf{Case 2: when $\|\nabla\cL\|\le\epsilon_\cL$.}

    In this case, we no longer have the dissipative condition. However, by Lemma~\ref{lem:lower_bound_gradient}, we can still lower bound the gradient
\begin{align*}
    \|\nabla\cLwd(\theta)\|^2\ge \frac{\lambda^2}{4}\|\theta\|^2
\end{align*}
Then consider $\frac{\lambda}{2}\|\theta\|^2\ge \epsilon_{\theta^*_{\rm loc}}^2> \frac{2\epsilon_\cL^2}{\lambda}$ and  $0<\epsilon\le \zeta\epsilon_{\theta^*_{\rm loc}}$ with $\zeta>0$. Similar to the case above, we have
    \begin{align*}
         \frac{\frac{\lambda^2}{4}\|\theta\|^2}{{\cL(\theta^*_{\rm loc})+\frac{\mathsf{L}_{\theta^*_{\rm loc}}\epsilon^2}{2}+\frac{\lambda}{2}\|\theta\|^2}}\ge \frac{\frac{\lambda}{2}\epsilon_{\theta^*_{\rm loc}}^2}{{\cL(\theta^*_{\rm loc})+\frac{\mathsf{L}_{\theta^*_{\rm loc}}\zeta^2\epsilon_{\theta^*_{\rm loc}}^2}{2}+\epsilon_{\theta^*_{\rm loc}}^2}}
    \end{align*}
    Thus 
    \begin{align*}
        \|\nabla\cLwd(\theta)\|^2&\ge   \frac{\frac{\lambda}{2}\epsilon_{\theta^*_{\rm loc}}^2}{{\cL(\theta^*_{\rm loc})+\frac{\mathsf{L}_{\theta^*_{\rm loc}}\zeta^2\epsilon_{\theta^*_{\rm loc}}^2}{2}+\epsilon_{\theta^*_{\rm loc}}^2}}\cLwd(\theta)
    \end{align*}
\end{proof}

\begin{corollary}
    \label{cor:near_global_min_grad_lower_bound}
      Let $\zeta>0$ be some constant. Under the same assumptions as Lemma~\ref{lem:near_local_min_grad_lower_bound}, when $$\zeta^2\frac{\lambda}{2}\|\theta\|^2\ge \operatorname{dist}(\theta,\cM)^2,\text{ and }\frac{\lambda}{2}\|\theta\|^2\ge 2\frac{\epsilon_\cL^2}{\lambda}$$ we have:
      \begin{enumerate}
        \item When $\|\nabla\cL\|>\epsilon_\cL$
        \begin{align*}
        \|\nabla\cLwd(\theta)\|^2&\ge \frac{4\lambda }{{\mathsf{L}_\cM \zeta^2}+2}\cLwd(\theta)
    \end{align*}
    where $\mathsf{L}_\cM =\sup_{\theta^*\in\cM}\sup_{\theta\in B_\bzero(R_\lambda)} S(\cdots,\max\{\|\theta_i\|,\|\theta^*_i\|\},\cdots)$.
        \item When $\|\nabla\cL\|\le\epsilon_\cL$,
        \begin{align*}
        \|\nabla\cLwd(\theta)\|^2&\ge \frac{\lambda}{\mathsf{L}_\cM \zeta^2+2}\cLwd(\theta)
    \end{align*}
      \end{enumerate}
    
\end{corollary}
\begin{proof}
    By Lemma~\ref{lem:near_local_min_grad_lower_bound}, we know that for some global minimum point $\theta^*$, when $$\theta\in\left\{\theta\,\bigg\vert\,\frac{\lambda}{2}\|\theta\|^2\ge \epsilon_{\theta^*}^2\text{ and }\|\theta-\theta^*\|\le \zeta\epsilon_{\theta^*}\right\}$$
    namely
    $$\zeta^2\frac{\lambda}{2}\|\theta\|^2\ge \|\theta-\theta^*\|^2,$$
    we have for $\|\nabla\cL\|>\epsilon_\cL$
    \begin{align*}
        \|\nabla\cLwd(\theta)\|^2&\ge   \frac{2\lambda \epsilon_{\theta^*}^2}{{\cL(\theta^*)+\frac{\mathsf{L}_{\theta^*}\zeta^2\epsilon_{\theta^*}^2}{2}+\epsilon_{\theta^*}^2}}\cLwd(\theta).
    \end{align*}
    Since $\cL(\theta^*)=0$, we have
    \begin{align*}
        \|\nabla\cLwd(\theta)\|^2&\ge   \frac{2\lambda \epsilon_{\theta^*}^2}{{\frac{\mathsf{L}_{\theta^*}\zeta^2\epsilon_{\theta^*}^2}{2}+\epsilon_{\theta^*}^2}}\cLwd(\theta)\\
        &=\frac{4\lambda }{{\mathsf{L}_{\theta^*}\zeta^2}+2}\cLwd(\theta)\\
        &\ge \frac{4\lambda }{{\mathsf{L}_\cM \zeta^2}+2}\cLwd(\theta)
    \end{align*}
    where 
    $$\mathsf{L}_\cM =\sup_{\theta^*\in\cM}\sup_{\theta\in B_\bzero(R_\lambda)} S(\cdots,\max\{\|\theta_i\|,\|\theta^*_i\|\},\cdots)=\sup_{\theta^*\in\cM} \mathsf{L}_{\theta^*}\ge \mathsf{L}_{\theta^*}.$$
    Similarly, for $\|\nabla\cL\|\le\epsilon_\cL$, we have $\lambda\epsilon_{\theta^*}^2\ge 2\epsilon_\cL^2$, and thus
    \begin{align*}
        \|\nabla\cLwd(\theta)\|^2&\ge  \frac{\lambda }{{\mathsf{L}_\cM \zeta^2}+2}\cLwd(\theta)
    \end{align*}

    In the end, the above inequalities hold for any $\theta^*\in\cM$; hence the condition can be replaced by
    $$\zeta^2\frac{\lambda}{2}\|\theta\|^2\ge \operatorname{dist}(\theta,\cM)^2.$$
    \end{proof}

\section{Generalization}

In this section, we prove all results related to generalization.

First, we prove Lemma~\ref{lem:non_overlapping_input}, which justifies partitioning the input domain.
\begin{proof}[Proof of Lemma~\ref{lem:non_overlapping_input}]
For each pair $1\le i<j\le N$, define the collision set
\[
    D_{ij}=\{(z_1,\cdots,z_N)\in\Omega^N: z_i=z_j\}.
\]
The set $D_{ij}$ is contained in a lower-dimensional affine subspace of
$\RR^{Nd}$, and therefore $\Leb_{Nd}(D_{ij})=0$. Since Assumption~\ref{assump:data}
gives $\pi_N\ll\Leb_{Nd}$, we have $\pi_N(D_{ij})=0$ for every $i<j$. Hence
\[
    \pi_N\left(\bigcup_{1\le i<j\le N}D_{ij}\right)
    \le \sum_{1\le i<j\le N}\pi_N(D_{ij})=0.
\]
Therefore, outside a $\pi_N$-measure-zero set, no two training inputs coincide,
which proves $x_i\ne x_j$ for all $i\ne j$ with probability 1.
\end{proof}

We next recall the decompositions and definitions introduced in Lemma~\ref{lem:prediction_variation_layer_decomposition}. For the neural network structure in~\eqref{eqn:NN}, under Assumption~\ref{assump:architecture_generalization},
\begin{align*}
    f(\theta;x)=\theta_L\left(x+\sum_{\ell=0}^{L-1}\varphi_\ell(\theta_\ell;u_\ell(\theta_{\ell-1:0};x))\right).
\end{align*}
Consequently, the prediction variation error satisfies
\begin{align*}
     E_{\rm PV}(\theta)
     =&3\sum_{i=1}^N\left(\int_{\Omega_i}\pi(x)dx\right)\EE_{\Omega_i}\|f(\theta;x)-f(\theta;x_i)\|^2\\
     \le&3(L+1)\left(V_{\rm pre}(\theta)+\sum_{\ell=0}^{L-1}V_\ell(\theta)\right).
\end{align*}
\begin{proof}[Proof of Lemma~\ref{lem:prediction_variation_layer_decomposition}]
    The network identity follows directly from the neural network form in~\eqref{eqn:NN}. The decomposition bound follows from Jensen's inequality.
\end{proof}

Here
\begin{align*}
    V_{\rm pre}(\theta)&:=\EE_x\left\|\theta_Lx-\sum_{i=1}^N\theta_Lx_i\bbone_{x\in\Omega_i}\right\|^2,\\
    V_{\ell}(\theta)&:=\EE_x\left\|\theta_L\varphi_\ell(\theta_\ell;u_\ell(\theta_{\ell-1:0};x))-\sum_{i=1}^N\theta_L\varphi_\ell(\theta_\ell;u_\ell(\theta_{\ell-1:0};x_i))\bbone_{x\in\Omega_i}\right\|^2,
\end{align*}
for $\ell=0,\cdots,L-1$. To express these quantities cellwise, define
\begin{align*}
    V_{{\rm pre},\Omega_i}(\theta)&:=\EE_{\Omega_i}\|\theta_Lx-\theta_Lx_i\|^2,\\
    V_{\ell,\Omega_i}(\theta)&:=\EE_{\Omega_i}\|\theta_L\varphi_\ell(\theta_\ell;u_\ell(\theta_{\ell-1:0};x))-\theta_L\varphi_\ell(\theta_\ell;u_\ell(\theta_{\ell-1:0};x_i))\|^2.
\end{align*}
The global quantities are the corresponding population-mass-weighted sums:
\begin{align*}
    V_{\rm pre}(\theta)&=\sum_{i=1}^N\left(\int_{\Omega_i}\pi(x)dx\right)V_{{\rm pre},\Omega_i}(\theta),\\
    V_{\ell}(\theta)&=\sum_{i=1}^N\left(\int_{\Omega_i}\pi(x)dx\right)V_{\ell,\Omega_i}(\theta).
\end{align*}
We therefore define the layerwise prediction variation error by
\begin{align*}
    E_{{\rm PV,layer}}(\theta):=3(L+1)\left(V_{\rm pre}(\theta)+\sum_{\ell=0}^{L-1}V_\ell(\theta)\right),
\end{align*}
so that $E_{\rm PV}(\theta)\le E_{{\rm PV,layer}}(\theta)$.

We also use a refined decomposition that separates the within-cell variance of each layer contribution from the discrepancy between its cellwise mean and its value at the representative input. The bias--variance identity gives
\begin{align*}
    E_{\rm PV}(\theta)\le& 3(L+1)\bigg[V_{\rm pre}(\theta)\\
    &+\sum_{\ell=0}^{L-1}\sum_{i=1}^N\left(\int_{\Omega_i}\pi(x)dx\right)\bigg(\underbrace{\EE_{\Omega_i}\left\|\theta_L\varphi_\ell(\theta_\ell;u_\ell(\theta_{\ell-1:0};x))-\EE_{\Omega_i}\theta_L\varphi_\ell(\theta_\ell;u_\ell(\theta_{\ell-1:0};x))\right\|^2}_{V_{\ell,\Omega_i,\rm var}(\theta)}\\
    &+\underbrace{\left\|\EE_{\Omega_i}\theta_L\varphi_\ell(\theta_\ell;u_\ell(\theta_{\ell-1:0};x))-\theta_L\varphi_\ell(\theta_\ell;u_\ell(\theta_{\ell-1:0};x_i))\right\|^2}_{M_{\ell,\Omega_i}(\theta)}\bigg)\bigg].
\end{align*}
Accordingly, define
\begin{align*}
    V_{\ell,\Omega_i,\rm var}(\theta)&:=\EE_{\Omega_i}\left\|\theta_L\varphi_\ell(\theta_\ell;u_\ell(\theta_{\ell-1:0};x))-\EE_{\Omega_i}\theta_L\varphi_\ell(\theta_\ell;u_\ell(\theta_{\ell-1:0};x))\right\|^2,\\
    M_{\ell,\Omega_i}(\theta)&:=\left\|\EE_{\Omega_i}\theta_L\varphi_\ell(\theta_\ell;u_\ell(\theta_{\ell-1:0};x))-\theta_L\varphi_\ell(\theta_\ell;u_\ell(\theta_{\ell-1:0};x_i))\right\|^2.
\end{align*}
Thus
\begin{align*}
    V_{\ell,\Omega_i}(\theta)&=V_{\ell,\Omega_i,\rm var}(\theta)+M_{\ell,\Omega_i}(\theta),\\
    V_{\ell,\rm var}(\theta)&:=\sum_{i=1}^N\left(\int_{\Omega_i}\pi(x)dx\right)V_{\ell,\Omega_i,\rm var}(\theta),\\
    M_{\ell}(\theta)&:=\sum_{i=1}^N\left(\int_{\Omega_i}\pi(x)dx\right)M_{\ell,\Omega_i}(\theta),\\
    V_{\ell}(\theta)&=V_{\ell,\rm var}(\theta)+M_{\ell}(\theta).
\end{align*}

\subsection{Data}

\begin{proof}[Proof of Theorem~\ref{thm:gen_data_error}]
In this proof, we work on the probability-one event from Lemma~\ref{lem:non_overlapping_input}, where the training inputs are pairwise distinct and the partition adapted to them is well defined. Conditioned on these inputs and the resulting partition, the cell masses are fixed and the remaining randomness comes from the independent Gaussian variables $\xi_1,\cdots,\xi_N$. We have
    \begin{align*}
   E_{\rm data}&=3 \sum_{i=1}^N\int_{\Omega_i}\pi(x)dx\cdot \EE_{\Omega_i}\|y_i-g(x)\|^2\\
   &\le 6 \sum_{i=1}^N\int_{\Omega_i}\pi(x)dx\cdot \left(\|y_i-g(x_i)\|^2+\EE_{\Omega_i}\|g(x)-g(x_i)\|^2\right)\\
   &=6\,\sigma_{\rm data}^2 \sum_{i=1}^N\int_{\Omega_i}\pi(x)dx\cdot \|\xi_i\|^2+6\,\EE_{x\sim\pi(x)}\left\|\sum_{i=1}^N g(x_i) \bbone_{x\in\Omega_i}-g(x)\right\|^2
\end{align*}

Since $\xi_i\overset{i.i.d.}{\sim}\cN(0,I)$, then $$\EE \sum_{i=1}^N\int_{\Omega_i}\pi(x)dx\cdot \|\xi_i\|^2 =d.$$
By \citet{laurent2000adaptive}, we have
\begin{align*}
    \PP\left( \sum_{i=1}^N\int_{\Omega_i}\pi(x)dx\cdot \|\xi_i\|^2-d\ge 2 \sqrt{d\sum_{i=1}^N\left(\int_{\Omega_i}\pi(x)dx\right)^2t}+ 2t\max_{1\le i\le N}\int_{\Omega_i}\pi(x)dx  \right)\le \exp(-t).
\end{align*}
Therefore, conditioned on the probability one input event, with probability at least $1-\delta$ over the output noise, we have
\begin{align*}
    E_{\rm data}\le& 6\,\sigma_{\rm data}^2 \left(d+ 2 \sqrt{d\sum_{i=1}^N\left(\int_{\Omega_i}\pi(x)dx\right)^2\log\frac{1}{\delta}}+ 2\log\frac{1}{\delta}\max_{1\le i\le N}\int_{\Omega_i}\pi(x)dx \right)\\ 
    &\qquad+6\,\EE_{x\sim\pi(x)}\left\|\sum_{i=1}^N g(x_i) \bbone_{x\in\Omega_i}-g(x)\right\|^2.
\end{align*}
\end{proof}

\subsection{Optimization error}

\begin{proof}[Proof of Corollary~\ref{cor:gen_optimization_error}]
    Here in this proof, we also work on the probability-one event from Lemma~\ref{lem:non_overlapping_input}, where the training inputs are pairwise distinct and the partition adapted to them is well defined. This event can be intersected with, and hence absorbed into, the probability-one event underlying Theorem~\ref{thm:convergence_to_global_min}. For every $K_1\le K\le K_2$, that theorem and the definition of $E_{\rm opt}$ give
     \begin{align*}
    E_{{\rm opt},K}&=3\sum_{i=1}^N\int_{\Omega_i}\pi(x)dx\cdot\|f(\theta^K;x_i)-y_i\|^2\\
    &\le 6N \max_{1\le i\le N}\int_{\Omega_i}\pi(x)dx\cdot\cL(\theta^K)\\ 
    &\le 6N \max_{1\le i\le N}\int_{\Omega_i}\pi(x)dx\cdot  \frac{1}{2}\mathsf{L}_{\cM_{0}} \epsilon_\eta^2.
\end{align*}
\end{proof}

\subsection{Prediction variation error}

We provide the full version of Theorem~\ref{thm:gen_NN_error} in the following:
\begin{theorem}
        \label{thm:gen_NN_error_full}
       Under the assumptions of Theorem~\ref{thm:general_convergence_loss_decay}, Assumption~\ref{assump:architecture_generalization}, and~\ref{assump:approximate_homogeneity}, for $k\ge0$,
        \begin{align*}
            E_{\rm PV}(\theta^{K_0+k})\le& E_{\rm PV,layer}(\theta^{K_0+k})\\ 
            \le&3(L+1)\bigg[ \left(1-2\eta\lambda \right)^{k} V_{{\rm pre}}(\theta^{K_0})+\sum_{\ell=0}^{L-1} \left(1-2\left(\sum_{r=0}^{\ell}s_{r,\rm in}+\ell+2\right)\eta\lambda \right)^{k} V_{\ell}(\theta^{K_0})\bigg]\\
            &+6(L+1)\sum_{\ell=0}^{L-1}C_{{\rm ah},V_\ell}(\ell+2)\epsilon_{\rm ah}\\
            &+3(L+1)\left(\frac{C_{{\rm grad},V_{\rm pre}}}{2\lambda}+\sum_{\ell=0}^{L-1}\frac{C_{{\rm grad},V_{\ell}}}{\lambda(\ell+2)}\right)\max_{K_0\le j\le K_0+k}\|\nabla\cL(\theta^j)\|\\
            &+\frac{3(L+1)}{2-\eta\,\mathsf{L}_\lambda}\left(S_{V_{\rm pre}}+\sum_{\ell=0}^{L-1}S_{V_{\ell}}\right)\cLwd(\theta^{K_0}).
    \end{align*}
    where $C_{{\rm ah},V_\ell}, C_{{\rm grad},V_{\ell}}, C_{{\rm grad},V_{\rm pre}}, S_{V_{\ell}}, S_{V_{\rm pre}}>0$ are constants defined in Lemma~\ref{lem:V_M_decay}. Consequently,
    \begin{align*}
            E_{\rm PV,layer}(\theta^{K_0+k})
            \le& \left(1-2\eta\lambda \right)^{k} E_{\rm PV,layer}(\theta^{K_0})+\underbrace{3L(L+1)(L+3)C_{\rm ah}\epsilon_{\rm ah}}_{E_{\rm PV,I}}+\underbrace{\frac{3(L+1)^2S_V}{2-\eta\,\mathsf{L}_\lambda}\cLwd(\theta^{K_0})}_{E_{\rm PV,II}}\\
            &+\underbrace{\frac{3(L+1)C_{\rm grad}}{\lambda}\left(\frac{1}{2}+\sum_{\ell=0}^{L-1}\frac{1}{\ell+2}\right)\max_{K_0\le j\le K_0+k}\|\nabla\cL(\theta^j)\|}_{E_{\rm PV,III}}.
    \end{align*}
\end{theorem}
\begin{proof}
In this proof, we also work on the probability-one event from Lemma~\ref{lem:non_overlapping_input}, where the training inputs are pairwise distinct and the partition adapted to them is well defined. This event can be intersected with, and hence absorbed into, the probability-one event underlying Theorem~\ref{thm:convergence_to_global_min}.

By Lemma~\ref{lem:prediction_variation_layer_decomposition},
\begin{align*}
     E_{\rm PV}(\theta)
    \le E_{\rm PV,layer}(\theta)
            =&3(L+1)\left( V_{\rm pre}(\theta)+\sum_{\ell=0}^{L-1}V_{\ell}(\theta)\right).
\end{align*}
Applying Lemma~\ref{lem:V_M_decay} to $V_{\rm pre}$ and $V_\ell$ yields
\begin{align*}
     E_{\rm PV}(\theta^{K_0+k})
    \le&E_{\rm PV,layer}(\theta^{K_0+k})\\ 
            \le&3(L+1)\bigg[ \left(1-2\eta\lambda \right)^{k} V_{{\rm pre}}(\theta^{K_0})+\sum_{\ell=0}^{L-1} \left(1-2\left(\sum_{r=0}^{\ell}s_{r,\rm in}+\ell+2\right)\eta\lambda \right)^{k} V_{\ell}(\theta^{K_0})\bigg]\\
    &+6(L+1)\sum_{\ell=0}^{L-1}C_{{\rm ah},V_\ell}(\ell+2)\epsilon_{\rm ah}\\
    &+3(L+1)\left(\frac{C_{{\rm grad},V_{\rm pre}}}{2\lambda}+\sum_{\ell=0}^{L-1}\frac{C_{{\rm grad},V_{\ell}}}{\lambda(\ell+2)}\right)\max_{K_0\le j\le K_0+k}\|\nabla\cL(\theta^j)\|\\
    &+\frac{3(L+1)}{2-\eta\,\mathsf{L}_\lambda}\left(S_{V_{\rm pre}}+\sum_{\ell=0}^{L-1}S_{V_{\ell}}\right)\cLwd(\theta^{K_0}).
\end{align*}
Let
        \begin{align*}
            C_{\rm ah}&=\max_{0\le \ell\le L-1}C_{{\rm ah},V_\ell},\\
            C_{\rm grad}&=\max\left\{C_{{\rm grad},V_{\rm pre}},\max_{0\le \ell\le L-1}C_{{\rm grad},V_{\ell}}\right\},\\
            S_V&=\max\left\{S_{V_{\rm pre}},\max_{0\le \ell\le L-1}S_{V_{\ell}}\right\}.
        \end{align*}
        We then have
        \begin{align*}
            E_{\rm PV}(\theta^{K_0+k})\le&E_{\rm PV,layer}(\theta^{K_0+k})\\ 
            \le&3(L+1)\bigg[ \left(1-2\eta\lambda \right)^{k} V_{{\rm pre}}(\theta^{K_0})+\sum_{\ell=0}^{L-1} \left(1-2\left(\sum_{r=0}^{\ell}s_{r,\rm in}+\ell+2\right)\eta\lambda \right)^{k} V_{\ell}(\theta^{K_0})\bigg]\\
            &+3L(L+1)(L+3)C_{\rm ah}\epsilon_{\rm ah}\\
            &+\frac{3(L+1)C_{\rm grad}}{\lambda}\left(\frac{1}{2}+\sum_{\ell=0}^{L-1}\frac{1}{\ell+2}\right)\max_{K_0\le j\le K_0+k}\|\nabla\cL(\theta^j)\|\\
            &+\frac{3(L+1)^2S_V}{2-\eta\,\mathsf{L}_\lambda}\cLwd(\theta^{K_0}).
    \end{align*}
    It remains only to simplify this right hand side. Since $s_{r,\rm in}\ge0$ and the contraction factors are nonnegative,
\begin{align*}
    \left(1-2\left(\sum_{r=0}^{\ell}s_{r,\rm in}+\ell+2\right)\eta\lambda\right)^k\le (1-2\eta\lambda)^k.
\end{align*}
Consequently, the first terms are bounded by $(1-2\eta\lambda)^kE_{\rm PV,layer}(\theta^{K_0})$.
\end{proof}

\subsection{Necessary/sufficient conditions of generalization}

\begin{proof}[Proof of Corollary~\ref{cor:necessary_condi_generalization}]
    By Lemma~\ref{lem:V_M_decay}, the estimate for $V_\ell$ contains the approximate homogeneity error
    \begin{align*}
        2C_{{\rm ah},V_\ell}(\ell+2)\epsilon_{\rm ah}.
    \end{align*}
    Therefore, in order to have
    \begin{align*}
        V_{\ell}(\theta^{K_2})\le c_{K_0}V_{\ell}(\theta^{K_0}),
    \end{align*}
    we need
    \begin{align*}
        c_{K_0}V_{\ell}(\theta^{K_0})>2C_{{\rm ah},V_\ell}(\ell+2)\epsilon_{\rm ah}.
    \end{align*}
    Hence
    \begin{align*}
        V_{\ell}(\theta^{K_0})>\frac{2C_{{\rm ah},V_\ell}}{c_{K_0}}(\ell+2)\epsilon_{\rm ah}.
    \end{align*}
\end{proof}

\begin{proof}[Proof of Corollary~\ref{cor:sufficient_condi_generalization}]
    Let $k=K_2-K_0$. By Lemma~\ref{lem:V_M_decay},
    \begin{align*}
        V_{\ell}(\theta^{K_2})
        \le& \left(1-2\left(\sum_{r=0}^{\ell}s_{r,\rm in}+\ell+2\right)\eta\lambda \right)^{K_2-K_0}V_{\ell}(\theta^{K_0})\\
        &\quad+2C_{{\rm ah},V_\ell}(\ell+2)\epsilon_{\rm ah}
        +\frac{C_{{\rm grad},V_{\ell}}}{\lambda(\ell+2)}\max_{K_0\le j\le K_2}\|\nabla\cL(\theta^j)\|\\
        &\quad+\frac{S_{V_{\ell}}}{2-\eta\,\mathsf{L}_\lambda}\cLwd(\theta^{K_0}).
    \end{align*}
    Since $K_1\le K_0\le K_2$, Theorem~\ref{thm:convergence_to_global_min} implies
    \begin{align*}
        \max_{K_0\le j\le K_2}\|\nabla\cL(\theta^j)\|\le \mathsf{L}_{\cM_0}\epsilon_\eta
    \end{align*}
    and
    \begin{align*}
        \cLwd(\theta^{K_0})\le
        \left(1-\eta\left(1-\frac{\eta\, \mathsf{L}_\lambda}{2}\right)\frac{\lambda}{\mathsf{L}_\cM \zeta^2+2}\right)^{K_0}\cLwd(\theta^{0}).
    \end{align*}
    Hence
    \begin{align*}
        V_{\ell}(\theta^{K_2})
        \le& \left(1-2\left(\sum_{r=0}^{\ell}s_{r,\rm in}+\ell+2\right)\eta\lambda \right)^{K_2-K_0}V_{\ell}(\theta^{K_0})\\
        &\quad+c_{K_0}V_{\ell}(\theta^{K_0})-\epsilon_{K_0}.
    \end{align*}
    By the lower bound on $K_2-K_0$, the first term on the right-hand side is at most $\epsilon_{K_0}$. Therefore
    \begin{align*}
        V_{\ell}(\theta^{K_2})\le c_{K_0}V_{\ell}(\theta^{K_0}).
    \end{align*}
\end{proof}

\subsection{Further interpretations of prediction variation}
\label{subsec:further_prediction_variation_interpretations}

This subsection records two further consequences of the layerwise prediction variation analysis.

\subsubsection{Different decay scales of variance and bias}
\label{subsubsec:bias_variance_decay}

The preceding analysis further decomposes each layerwise prediction variation into a within-cell variance term and a squared bias-type term:
\begin{align*}
    V_{\ell,\Omega_i}(\theta)&=V_{\ell,\Omega_i,\rm var}(\theta)+M_{\ell,\Omega_i}(\theta),\\
    V_\ell(\theta)&=V_{\ell,\rm var}(\theta)+M_\ell(\theta).
\end{align*}
Lemma~\ref{lem:V_M_decay} shows that these two terms have the same geometric contraction factor but different approximate homogeneity remainders:
\begin{align*}
V_{\ell,\rm var}(\theta^{K_0+k})\le&\left(1-2\left(\sum_{r=0}^{\ell}s_{r,\rm in}+\ell+2\right)\eta\lambda\right)^kV_{\ell,\rm var}(\theta^{K_0})+C_{{\rm ah},V_{\ell,\rm var}}(\ell+2)\epsilon_{\rm ah}\\
&+\frac{C_{{\rm grad},V_{\ell,\rm var}}}{\lambda(\ell+2)}\max_{K_0\le j\le K_0+k}\|\nabla\cL(\theta^j)\|+\frac{S_{V_{\ell,\rm var}}}{2-\eta\,\mathsf{L}_\lambda}\cLwd(\theta^{K_0}),\\
M_\ell(\theta^{K_0+k})\le&\left(1-2\left(\sum_{r=0}^{\ell}s_{r,\rm in}+\ell+2\right)\eta\lambda\right)^kM_\ell(\theta^{K_0})+2C_{{\rm ah},V_{\ell,\rm var}}(\ell+2)\epsilon_{\rm ah}\\
&+\frac{C_{{\rm grad},M_\ell}}{\lambda(\ell+2)}\max_{K_0\le j\le K_0+k}\|\nabla\cL(\theta^j)\|+\frac{S_{M_\ell}}{2-\eta\,\mathsf{L}_\lambda}\cLwd(\theta^{K_0}).
\end{align*}
In particular, the ratio between the two approximate homogeneity contributions is
\begin{align*}
    \frac{2C_{{\rm ah},V_{\ell,\rm var}}(\ell+2)\epsilon_{\rm ah}}{C_{{\rm ah},V_{\ell,\rm var}}(\ell+2)\epsilon_{\rm ah}}=2.
\end{align*}
Thus the bias term has a larger guaranteed approximation error, whereas the variance term has the smaller error. More precisely, consider a regime in which all the remaining contributions are asymptotically smaller than the approximate homogeneity scale, namely,
\begin{align*}
&\frac{C_{{\rm grad},V_{\ell,\rm var}}}{\lambda(\ell+2)}\sup_{j\ge K_0}\|\nabla\cL(\theta^j)\|+\frac{S_{V_{\ell,\rm var}}}{2-\eta\,\mathsf{L}_\lambda}\cLwd(\theta^{K_0})=o(\epsilon_{\rm ah}),\\
&\frac{C_{{\rm grad},M_\ell}}{\lambda(\ell+2)}\sup_{j\ge K_0}\|\nabla\cL(\theta^j)\|+\frac{S_{M_\ell}}{2-\eta\,\mathsf{L}_\lambda}\cLwd(\theta^{K_0})=o(\epsilon_{\rm ah}).
\end{align*}
The geometrically decaying contraction terms then vanish as $k\to\infty$, and the estimates give
\begin{align*}
    \limsup_{k\to\infty}V_{\ell,\rm var}(\theta^{K_0+k})&\le C_{{\rm ah},V_{\ell,\rm var}}(\ell+2)\epsilon_{\rm ah}+o(\epsilon_{\rm ah}),\\
    \limsup_{k\to\infty}M_\ell(\theta^{K_0+k})&\le 2C_{{\rm ah},V_{\ell,\rm var}}(\ell+2)\epsilon_{\rm ah}+o(\epsilon_{\rm ah}).
\end{align*}
Therefore, at the leading order in $\epsilon_{\rm ah}$, the upper bound for the bias-type component is larger by one additional copy of $C_{{\rm ah},V_{\ell,\rm var}}(\ell+2)\epsilon_{\rm ah}$. This asymmetry arises because the bias compares the cellwise mean with the representative value, so the approximate homogeneity remainder enters through both quantities, whereas the centered variance benefits from cancellation around the cellwise mean.

\subsubsection{Weights inside and outside the nonlinear nodes}
\label{subsubsec:inner_outer_weights}

For the architecture
\begin{align*}
    \varphi_\ell(\theta_\ell;u)=\theta_{\ell,\rm out}\sigma(\theta_{\ell,\rm in};\tau(u)),
\end{align*}
the outer weight $\theta_{\ell,\rm out}$ enters its own block linearly. Consequently, Corollary~\ref{cor:difference_theta_in_and_out} gives the exact identity
\begin{align*}
    \ip{\nabla_{\theta_{\ell,\rm out}}V_{\ell,\Omega_i}(\theta)}{\theta_{\ell,\rm out}}=2V_{\ell,\Omega_i}(\theta),
\end{align*}
which contains no approximate homogeneity remainder. In contrast, an inner weight appears inside the nonlinear map and satisfies, for $0\le r\le\ell$,
\begin{align*}
    \ip{\nabla_{\theta_{r,\rm in}}V_{\ell,\Omega_i}(\theta)}{\theta_{r,\rm in}}\ge 2s_{r,\rm in}V_{\ell,\Omega_i}(\theta)-4{C_{i,\ell,r,\rm in}}\epsilon_{\rm ah}.
\end{align*}
Thus inner weights contribute directly to the approximate homogeneity error.

This distinction is local to the current layer. If $r<\ell$, the earlier outer weight $\theta_{r,\rm out}$ affects the hidden state entering layer $\ell$ through
\begin{align*}
    \nabla_{\theta_{r,\rm out}}u_\ell(x)=\nabla_{u_{\ell-1}}u_\ell(x)\cdots\nabla_{u_r}u_{r+1}(x)\nabla_{\theta_{r,\rm out}}u_r(x),
\end{align*}
and therefore generally $\nabla_{\theta_{r,\rm out}}V_{\ell,\Omega_i}(\theta)\ne0$. For example, the variance calculation in the appendix yields
\begin{align*}
    \ip{\nabla_{\theta_{r,\rm out}}V_{\ell,\Omega_i,\rm var}(\theta)}{\theta_{r,\rm out}}\ge 2V_{\ell,\Omega_i,\rm var}(\theta)-2{C_{i,\ell,r,\rm out}}\epsilon_{\rm ah},\qquad r<\ell.
\end{align*}
Hence an outer weight creates no direct approximate homogeneity error for the variation of its own layer, but it can contribute to the error bounds of every later layer through the composition of Jacobians. In this sense, the direct contribution is asymmetric between inner and outer weights, while depth propagates the effect of both types of weights forward.

\subsection{Supplementary lemmas for generalization}

\begin{proposition}
    \label{prop:poly_smooth_preserved_under_integration}
    If $\psi(w,x)$ is poly-smooth w.r.t. $w$, then $\int_x \psi(w,x)dx$ is also poly-smooth w.r.t. $w$.
\end{proposition}
\begin{proof}
    Suppose $\psi(w,x)$ is poly-smooth w.r.t. $w$, i.e.,
    \begin{align*}
        \|\nabla_w \psi(w,x)-\nabla_w\psi(w',x)\|\le S_x(\cdots,\max\{\|w_i\|,\|w_i'\|,\cdots\})\|w-w'\|
    \end{align*}
    where $S_x(\cdot)$ is a polynomial.
    Then 
    \begin{align*}
         \|\nabla_w \int_x \psi(w,x)dx-\nabla_w\int_x \psi(w',x)dx\|&=\|\int_x \nabla_w \psi(w,x)dx-\int_x \nabla_w \psi(w',x)dx\|\\
         &\le \int_x\|\nabla_w \psi(w,x)-\nabla_w\psi(w',x)\|dx\\ 
         &\le \int_x S_x(\cdots,\max\{\|w_i\|,\|w_i'\|,\cdots\})dx\, \|w-w'\|
    \end{align*}
    where $\int_x S_x(\cdots,\max\{\|w_i\|,\|w_i'\|,\cdots\})dx$ is also a polynomial.

    Hence $\int_x \psi(w,x)dx$ is also poly-smooth.
\end{proof}

\begin{lemma}
    \label{lem:gen_nabla_var_theta_lower_bound}
    Consider the vectorized version of all parameters, i.e., $\theta,\theta_1,\cdots,\theta_L,\theta_{1,\rm out},\cdots,\theta_{L-1,\rm out}$ are all vectors. Under Assumption~\ref{assump:architecture_generalization} and~\ref{assump:approximate_homogeneity}, we have
    \begin{align*}
    \ip{\nabla_{\theta}V_{\ell,\Omega_i,\rm var}(\theta)}{\theta}
    \ge& 2\left(\sum_{r=0}^{\ell}s_{i,r,\rm in}+\ell+2\right)V_{\ell,\Omega_i,\rm var}(\theta)-2\left({C_{i,\ell}}+\sum_{r=0}^{\ell-1}{C_{i,\ell,r,\rm in}}+{C_{i,\ell,r,\rm out}}\right)\,\epsilon_{\rm ah},
\end{align*}
and
    \begin{align*}
    \ip{\nabla_\theta M_{\ell,\Omega_i}(\theta)}{\theta}
    \ge& 2\left(\sum_{r=0}^{\ell}s_{i,r,\rm in}+\ell+2\right)M_{\ell,\Omega_i}(\theta)-4\left({C_{i,\ell}}+\sum_{r=0}^{\ell-1}{C_{i,\ell,r,\rm in}}+{C_{i,\ell,r,\rm out}}\right)\,\epsilon_{\rm ah}
\end{align*}
where 
\begin{align*}
     &{C_{i,\ell}}=\sup_{\theta\in B_\bzero(R_{\lambda})}\sup_{x\in\Omega_i}\|(\theta_L\varphi_\ell (\theta_\ell;u_\ell)-\EE_{\Omega_i}\theta_L\varphi_\ell(\theta_\ell;u_\ell))^\top\theta_L\theta_{\ell,\rm out}\| \\ 
     &{C_{i,\ell,r,\rm in}}=\\ 
     &\sup_{\theta\in B_\bzero(R_{\lambda})}\sup_{x\in\Omega_i}\Big( \sum_{t=r}^{\ell-1}\| (\theta_L\varphi_\ell (\theta_\ell;u_\ell)-\EE_{\Omega_i}\theta_L\varphi_\ell(\theta_\ell;u_\ell))^\top \theta_L\| \|\nabla_{u_{\ell}}\varphi_\ell (\theta_\ell;u_\ell)\|\prod_{s=0}^{\ell-1-t}\|\nabla_{u_{\ell-1-s}}u_{\ell-s} \| \\
    &\qquad+ \|(\theta_L\varphi_\ell (\theta_\ell;u_\ell)-\EE_{\Omega_i}\theta_L\varphi_\ell(\theta_\ell;u_\ell))^\top \theta_L\| \|\nabla_{u_{\ell}}\varphi_\ell (\theta_\ell;u_\ell) \|\\
    &\qquad+\| (\theta_L\varphi_\ell (\theta_\ell;u_\ell)-\EE_{\Omega_i}\theta_L\varphi_\ell(\theta_\ell;u_\ell))^\top \theta_L \| \Big) \\ 
    &{C_{i,\ell,r,\rm out}}=\\ 
    &\sup_{\theta\in B_\bzero(R_{\lambda})}\sup_{x\in\Omega_i}\Big( \sum_{t=r}^{\ell-1}\| (\theta_L\varphi_\ell (\theta_\ell;u_\ell)-\EE_{\Omega_i}\theta_L\varphi_\ell(\theta_\ell;u_\ell))^\top \theta_L\| \|\nabla_{u_{\ell}}\varphi_\ell (\theta_\ell;u_\ell)\|\prod_{s=0}^{\ell-1-t}\|\nabla_{u_{\ell-1-s}}u_{\ell-s} \| \\
    &\qquad+ \|(\theta_L\varphi_\ell (\theta_\ell;u_\ell)-\EE_{\Omega_i}\theta_L\varphi_\ell(\theta_\ell;u_\ell))^\top \theta_L\| \|\nabla_{u_{\ell}}\varphi_\ell (\theta_\ell;u_\ell) \|\\
    &\qquad+\| (\theta_L\varphi_\ell (\theta_\ell;u_\ell)-\EE_{\Omega_i}\theta_L\varphi_\ell(\theta_\ell;u_\ell))^\top \theta_L \| \Big)
\end{align*}
and $$s_{i,r,\rm in}=\inf_{\theta\in B_\bzero(R_{\lambda})} s_{i,r,\theta,\rm in}$$
with $s_{i,r,\theta,\rm in}$ being the degree of local approximate homogeneity of $\varphi_{r}(\theta_{r,\rm in};\theta_{r,\rm out};u_r(x))$ w.r.t. $\theta_{r,\rm in}$ in $x\in\Omega_i$ at some $\theta$.

For the combined version, we also have
\begin{align*}
    \ip{\nabla_{\theta}V_{\ell,\Omega_i}(\theta)}{\theta}
    \ge& 2\left(\sum_{r=0}^{\ell}s_{i,r,\rm in}+\ell+2\right)V_{\ell,\Omega_i}(\theta)-4\left({C^{\rm comb}_{i,\ell}}+\sum_{r=0}^{\ell-1}\left({C^{\rm comb}_{i,\ell,r,\rm in}}+{C^{\rm comb}_{i,\ell,r,\rm out}}\right)\right)\,\epsilon_{\rm ah},
\end{align*}
where
\begin{align*}
     &{C^{\rm comb}_{i,\ell}}=\sup_{\theta\in B_\bzero(R_{\lambda})}\sup_{x\in\Omega_i}\|\left(\theta_L\varphi_\ell(\theta_\ell;u_\ell(x))-\theta_L\varphi_\ell(\theta_\ell;u_{\ell,i})\right)^\top\theta_L\theta_{\ell,\rm out}\|,\\
     &{C^{\rm comb}_{i,\ell,r,\rm in}}= \sup_{\theta\in B_\bzero(R_{\lambda})}\sup_{x,z\in\Omega_i}\\ 
     &\Big( \sum_{t=r}^{\ell-1}\|\left(\theta_L\varphi_\ell(\theta_\ell;u_\ell(x))-\theta_L\varphi_\ell(\theta_\ell;u_{\ell,i})\right)^\top \theta_L\| \|\nabla_{u_{\ell}}\varphi_\ell (\theta_\ell;u_\ell(z))\|\prod_{s=0}^{\ell-1-t}\|\nabla_{u_{\ell-1-s}}u_{\ell-s}(z) \| \\
    &\qquad+ \|\left(\theta_L\varphi_\ell(\theta_\ell;u_\ell(x))-\theta_L\varphi_\ell(\theta_\ell;u_{\ell,i})\right)^\top \theta_L\| \|\nabla_{u_{\ell}}\varphi_\ell (\theta_\ell;u_\ell(z)) \|
    \\ 
    &\qquad+\|\left(\theta_L\varphi_\ell(\theta_\ell;u_\ell(x))-\theta_L\varphi_\ell(\theta_\ell;u_{\ell,i})\right)^\top \theta_L \| \Big),\\
     &{C^{\rm comb}_{i,\ell,r,\rm out}}=\sup_{\theta\in B_\bzero(R_{\lambda})}\sup_{x,z\in\Omega_i}\\ 
     &\Big( \sum_{t=r}^{\ell-1}\|\left(\theta_L\varphi_\ell(\theta_\ell;u_\ell(x))-\theta_L\varphi_\ell(\theta_\ell;u_{\ell,i})\right)^\top \theta_L\| \|\nabla_{u_{\ell}}\varphi_\ell (\theta_\ell;u_\ell(z))\|\prod_{s=0}^{\ell-1-t}\|\nabla_{u_{\ell-1-s}}u_{\ell-s}(z) \| \\
    &\qquad+ \|\left(\theta_L\varphi_\ell(\theta_\ell;u_\ell(x))-\theta_L\varphi_\ell(\theta_\ell;u_{\ell,i})\right)^\top \theta_L\| \|\nabla_{u_{\ell}}\varphi_\ell (\theta_\ell;u_\ell(z)) \|
    \\
    &\qquad+\|\left(\theta_L\varphi_\ell(\theta_\ell;u_\ell(x))-\theta_L\varphi_\ell(\theta_\ell;u_{\ell,i})\right)^\top \theta_L \| \Big).
\end{align*}

Additionally, we have
\begin{align*}
   \ip{\nabla_\theta V_{{\rm pre},\Omega_i}(\theta)}{\theta}&=2V_{{\rm pre},\Omega_i}(\theta)
\end{align*}

\end{lemma}
\begin{proof}

First, consider the gradient w.r.t. the weights of the $\ell$th layer $\theta_{\ell}$, where $\theta_{\ell,\rm in}$ represents the parameters inside the activation and $\theta_{\ell,\rm out}$ is the weights outside the activation, and $\theta_L$ is the weight of the last layer. 
    \begin{align*}
    &\nabla_{\theta_{\ell,\rm in}}\EE_{\Omega_i} \|\theta_L\varphi_\ell (\theta_\ell;u_\ell)-\EE_{\Omega_i}\theta_L\varphi_\ell(\theta_\ell;u_\ell)\|^2\\
    =&2\EE_{\Omega_i} (\theta_L\theta_{\ell,\rm out} \sigma_{\ell}(\theta_{\ell,\rm in};\bar u_{\ell})-\EE_{\Omega_i}\theta_L\varphi_\ell(\theta_\ell;u_\ell))^\top\sum_{r=1}^m(\theta_L\theta_{\ell,\rm out})_{\mathrm{col}\,r}\nabla_{\theta_{\ell,\rm in}}\sigma_{\ell,r}(\theta_{\ell,\rm in};\bar u_{\ell})\\
    &-2\EE_{\Omega_i} (\theta_L\theta_{\ell,\rm out} \sigma_{\ell}(\theta_{\ell,\rm in};\bar u_{\ell})-\EE_{\Omega_i}\theta_L\varphi_\ell(\theta_\ell;u_\ell))^\top \nabla_{\theta_{\ell,\rm in}}\EE_{\Omega_i}\theta_L\varphi_\ell(\theta_\ell;u_\ell)\\
    =&2\EE_{\Omega_i} (\theta_L\theta_{\ell,\rm out} \sigma_{\ell}(\theta_{\ell,\rm in};\bar u_{\ell})-\EE_{\Omega_i}\theta_L\varphi_\ell(\theta_\ell;u_\ell))^\top\sum_{r=1}^m(\theta_L\theta_{\ell,\rm out})_{\mathrm{col}\,r}\nabla_{\theta_{\ell,\rm in}}\sigma_{\ell,r}(\theta_{\ell,\rm in};\bar u_{\ell})
\end{align*}
where $ \sigma_{\ell,r}$ is the $r$th element of the vector valued function $\sigma_{\ell}$.
\begin{align*}
    &\EE_{\Omega_i} (\theta_L\theta_{\ell,\rm out} \sigma_{\ell}(\theta_{\ell,\rm in};\bar u_{\ell})-\EE_{\Omega_i}\theta_L\varphi_\ell(\theta_\ell;u_\ell))^\top \nabla_{\theta_{\ell,\rm in}}\EE_{\Omega_i}\theta_L\varphi_\ell(\theta_\ell;u_\ell)\\
    =&(\EE_{\Omega_i} \theta_L\varphi_\ell(\theta_\ell;u_\ell)-\EE_{\Omega_i}\theta_L\varphi_\ell(\theta_\ell;u_\ell))^\top \nabla_{\theta_{\ell,\rm in}}\EE_{\Omega_i}\theta_L\varphi_\ell(\theta_\ell;u_\ell)=0.
\end{align*}

Similarly, 
\begin{align*}
    &\nabla_{\theta_{\ell,2,\mathrm{row}\,j}}\EE_{\Omega_i} \|\theta_L\varphi_\ell (\theta_\ell;u_\ell)-\EE_{\Omega_i}\theta_L\varphi_\ell(\theta_\ell;u_\ell)\|^2\\
    &=2\EE_{\Omega_i}\left[(\theta_L\theta_{\ell,\rm out} \sigma_{\ell}(\theta_{\ell,\rm in};\bar u_{\ell})-\EE_{\Omega_i}\theta_L\varphi_\ell(\theta_\ell;u_\ell))^\top \theta_L\right]_j \sigma_{\ell}(\theta_{\ell,\rm in};\bar u_{\ell})^\top
\end{align*}
and 
\begin{align*}
     &\nabla_{\theta_{L,\mathrm{row}\,j}}\EE_{\Omega_i} \|\theta_L\varphi_\ell (\theta_\ell;u_\ell)-\EE_{\Omega_i}\theta_L\varphi_\ell(\theta_\ell;u_\ell)\|^2\\ &=2\EE_{\Omega_i}\left[\theta_L\theta_{\ell,\rm out} \sigma_{\ell}(\theta_{\ell,\rm in};\bar u_{\ell})-\EE_{\Omega_i}\theta_L\varphi_\ell(\theta_\ell;u_\ell) \right]_j(\theta_{\ell,\rm out} \sigma_{\ell}(\theta_{\ell,\rm in};\bar u_{\ell}))^\top
\end{align*}

Next, consider the inner product of the gradient w.r.t. $\theta_{\ell,\rm out}$ and $\theta_{\ell,\rm out}$,
\begin{align*}
     &\ip{ \nabla_{\theta_{\ell,\rm out}}\EE_{\Omega_i} \|\theta_L\varphi_\ell (\theta_\ell;u_\ell)-\EE_{\Omega_i}\theta_L\varphi_\ell(\theta_\ell;u_\ell)\|^2}{\theta_{\ell,\rm out}}\\
     =&2\Tr(\EE_{\Omega_i}\theta_L^\top(\theta_L\theta_{\ell,\rm out} \sigma_{\ell}(\theta_{\ell,\rm in};\bar u_{\ell})-\EE_{\Omega_i}\theta_L\varphi_\ell(\theta_\ell;u_\ell))  \sigma_{\ell}(\theta_{\ell,\rm in};\bar u_{\ell})^\top\theta_{\ell,\rm out}^\top)\\
     =&2\EE_{\Omega_i}\Tr((\theta_L\theta_{\ell,\rm out} \sigma_{\ell}(\theta_{\ell,\rm in};\bar u_{\ell})-\EE_{\Omega_i}\theta_L\varphi_\ell(\theta_\ell;u_\ell))  \sigma_{\ell}(\theta_{\ell,\rm in};\bar u_{\ell})^\top\theta_{\ell,\rm out}^\top\theta_L^\top)\\
     =&2\Tr(\var_{\Omega_i}(\theta_L\varphi_\ell(\theta_\ell;u_\ell)))\\
     =&2\EE_{\Omega_i} \|\theta_L\varphi_\ell (\theta_\ell;u_\ell)-\EE_{\Omega_i}\theta_L\varphi_\ell(\theta_\ell;u_\ell)\|^2
\end{align*}
where $\var(v)=\EE(vv^\top)$ is the covariance matrix.

For the gradient w.r.t. $\theta_L$,
\begin{align*}
    &\ip{ \nabla_{\theta_{L}}\EE_{\Omega_i} \|\theta_L\varphi_\ell (\theta_\ell;u_\ell)-\EE_{\Omega_i}\theta_L\varphi_\ell(\theta_\ell;u_\ell)\|^2}{\theta_{L}}\\
     =&2\EE_{\Omega_i}\Tr((\theta_L\theta_{\ell,\rm out} \sigma_{\ell}(\theta_{\ell,\rm in};\bar u_{\ell})-\EE_{\Omega_i}\theta_L\varphi_\ell(\theta_\ell;u_\ell))  \sigma_{\ell}(\theta_{\ell,\rm in};\bar u_{\ell})^\top\theta_{\ell,\rm out}^\top\theta_L^\top)\\
     =&2\Tr(\var_{\Omega_i}(\theta_L\varphi_\ell(\theta_\ell; u_\ell)))\\
     =&2\EE_{\Omega_i} \|\theta_L\varphi_\ell (\theta_\ell;u_\ell)-\EE_{\Omega_i}\theta_L\varphi_\ell(\theta_\ell;u_\ell)\|^2
\end{align*}

For the gradient w.r.t. $\theta_{\ell,\rm in}$,
\begin{align*}
    &\ip{ \nabla_{\theta_{\ell,\rm in}}\EE_{\Omega_i} \|\theta_L\varphi_\ell (\theta_\ell;u_\ell)-\EE_{\Omega_i}\theta_L\varphi_\ell(\theta_\ell;u_\ell)\|^2}{\theta_{\ell,\rm in}}\\
    =& 2\EE_{\Omega_i}(\theta_L\varphi_\ell (\theta_\ell;u_\ell)-\EE_{\Omega_i}\theta_L\varphi_\ell(\theta_\ell;u_\ell))^\top\sum_{r=1}^m(\theta_L\theta_{\ell,\rm out})_{\mathrm{col}\,r}\nabla_{\theta_{\ell,\rm in}}\sigma_{\ell,r}(\theta_{\ell,\rm in};\bar u_{\ell})\ve(\theta_{\ell,\rm in})\\
    =&2\EE_{\Omega_i}(\theta_L\varphi_\ell (\theta_\ell;u_\ell)-\EE_{\Omega_i}\theta_L\varphi_\ell(\theta_\ell;u_\ell))^\top\theta_L\theta_{\ell,\rm out}\nabla_{\theta_{\ell,\rm in}} \sigma_{\ell}(\theta_{\ell,\rm in};\bar u_{\ell})\ve(\theta_{\ell,\rm in})\\
    = &2s_{i,\ell,{\rm in}}\EE_{\Omega_i}(\theta_L\varphi_\ell (\theta_\ell;u_\ell)-\EE_{\Omega_i}\theta_L\varphi_\ell(\theta_\ell;u_\ell))^\top\theta_L\theta_{\ell,\rm out} \sigma_{\ell}(\theta_{\ell,\rm in};\bar u_{\ell})\\
    &\quad+2\EE_{\Omega_i}(\theta_L\varphi_\ell (\theta_\ell;u_\ell)-\EE_{\Omega_i}\theta_L\varphi_\ell(\theta_\ell;u_\ell))^\top\theta_L\theta_{\ell,\rm out}\\
    &\qquad \times\left(\nabla_{\theta_{\ell,\rm in}} \sigma_{\ell}(\theta_{\ell,\rm in};\bar u_{\ell})\ve(\theta_{\ell,\rm in})-s_{i,\ell,{\rm in}}\sigma_{\ell}(\theta_{\ell,\rm in};\bar u_{\ell})\right)\\
    \ge&2s_{i,\ell,{\rm in}}\EE_{\Omega_i} \|\theta_L\varphi_\ell (\theta_\ell;u_\ell)-\EE_{\Omega_i}\theta_L\varphi_\ell(\theta_\ell;u_\ell)\|^2-2 {C_{i,\ell}}\,\epsilon_{\rm ah}
\end{align*}
where $${C_{i,\ell}}=\sup_{\theta\in B_\bzero(R_{\lambda})}\sup_{x\in\Omega_i}\|(\theta_L\varphi_\ell (\theta_\ell;u_\ell)-\EE_{\Omega_i}\theta_L\varphi_\ell(\theta_\ell;u_\ell))^\top\theta_L\theta_{\ell,\rm out}\|,$$ and the last inequality follows from Assumption~\ref{assump:approximate_homogeneity} and Cauchy-Schwarz inequality.

For the weight $\theta_r$ that is not in the $\ell$th layer,
\begin{align*}
     &\nabla_{\theta_{r}}\EE_{\Omega_i} \|\theta_L\varphi_\ell (\theta_\ell;u_\ell)-\EE_{\Omega_i}\theta_L\varphi_\ell(\theta_\ell;u_\ell)\|^2\\
     =& 2\EE_{\Omega_i} (\theta_L\varphi_\ell (\theta_\ell;u_\ell)-\EE_{\Omega_i}\theta_L\varphi_\ell(\theta_\ell;u_\ell))^\top \theta_L\nabla_{u_{\ell}}\varphi_\ell (\theta_\ell;u_\ell)\nabla_{u_r}u_\ell \nabla_{\theta_r}u_r\\
     =&2\EE_{\Omega_i} (\theta_L\varphi_\ell (\theta_\ell;u_\ell)-\EE_{\Omega_i}\theta_L\varphi_\ell(\theta_\ell;u_\ell))^\top \theta_L\nabla_{u_{\ell}}\varphi_\ell (\theta_\ell;u_\ell)\nabla_{u_{\ell-1}}u_\ell\cdots \nabla_{u_r}u_{r+1} \nabla_{\theta_r}u_r
\end{align*}
Then
\begin{align*}
    &\ip{\nabla_{\theta_{r,\rm in}}\EE_{\Omega_i} \|\theta_L\varphi_\ell (\theta_\ell;u_\ell)-\EE_{\Omega_i}\theta_L\varphi_\ell(\theta_\ell;u_\ell)\|^2}{\theta_{r,\rm in}}\\
    =&2\EE_{\Omega_i} (\theta_L\varphi_\ell (\theta_\ell;u_\ell)-\EE_{\Omega_i}\theta_L\varphi_\ell(\theta_\ell;u_\ell))^\top \theta_L\nabla_{u_{\ell}}\varphi_\ell (\theta_\ell;u_\ell)\nabla_{u_{\ell-1}}u_\ell\cdots \nabla_{u_r}u_{r+1} \nabla_{\theta_{r,\rm in}}u_r\ve(\theta_{r,\rm in})\\
    \ge & 2s_{i,r,\rm in}\EE_{\Omega_i} \|\theta_L\varphi_\ell (\theta_\ell;u_\ell)-\EE_{\Omega_i}\theta_L\varphi_\ell(\theta_\ell;u_\ell)\|^2 - 2{C_{i,\ell,r,\rm in}}\,\epsilon_{\rm ah} 
\end{align*}
where the inequality follows from Assumption~\ref{assump:approximate_homogeneity} and Cauchy-Schwarz inequality, and 
\begin{align*}
    &{C_{i,\ell,r,\rm in}}=\\ &\sup_{\theta\in B_\bzero(R_{\lambda})}\sup_{x\in\Omega_i}\Big( \sum_{t=r}^{\ell-1}\| (\theta_L\varphi_\ell (\theta_\ell;u_\ell)-\EE_{\Omega_i}\theta_L\varphi_\ell(\theta_\ell;u_\ell))^\top \theta_L\| \|\nabla_{u_{\ell}}\varphi_\ell (\theta_\ell;u_\ell)\|\prod_{s=0}^{\ell-1-t}\|\nabla_{u_{\ell-1-s}}u_{\ell-s} \| \\
    &\qquad+ \|(\theta_L\varphi_\ell (\theta_\ell;u_\ell)-\EE_{\Omega_i}\theta_L\varphi_\ell(\theta_\ell;u_\ell))^\top \theta_L\| \|\nabla_{u_{\ell}}\varphi_\ell (\theta_\ell;u_\ell) \|\\
    &\qquad+\| (\theta_L\varphi_\ell (\theta_\ell;u_\ell)-\EE_{\Omega_i}\theta_L\varphi_\ell(\theta_\ell;u_\ell))^\top \theta_L \| \Big)
\end{align*}
Also,
\begin{align*}
    &\ip{\nabla_{\theta_{r,\rm out}}\EE_{\Omega_i} \|\theta_L\varphi_\ell (\theta_\ell;u_\ell)-\EE_{\Omega_i}\theta_L\varphi_\ell(\theta_\ell;u_\ell)\|^2}{\theta_{r,\rm out}}\\
    =&2\EE_{\Omega_i} (\theta_L\varphi_\ell (\theta_\ell;u_\ell)-\EE_{\Omega_i}\theta_L\varphi_\ell(\theta_\ell;u_\ell))^\top \theta_L\nabla_{u_{\ell}}\varphi_\ell (\theta_\ell;u_\ell)\nabla_{u_{\ell-1}}u_\ell\cdots \nabla_{u_r}u_{r+1} \nabla_{\theta_{r,\rm out}}u_r\ve(\theta_{r,\rm out})\\
    \ge & 2\EE_{\Omega_i} \|\theta_L\varphi_\ell (\theta_\ell;u_\ell)-\EE_{\Omega_i}\theta_L\varphi_\ell(\theta_\ell;u_\ell)\|^2 - 2{C_{i,\ell,r,\rm out}}\,\epsilon_{\rm ah}
\end{align*}
where 
\begin{align*}
    &{C_{i,\ell,r,\rm out}}=\sup_{\theta\in B_\bzero(R_{\lambda})}\sup_{x\in\Omega_i}\\ &\Big( \sum_{t=r}^{\ell-1}\| (\theta_L\varphi_\ell (\theta_\ell;u_\ell)-\EE_{\Omega_i}\theta_L\varphi_\ell(\theta_\ell;u_\ell))^\top \theta_L\| \|\nabla_{u_{\ell}}\varphi_\ell (\theta_\ell;u_\ell)\|\prod_{s=0}^{\ell-1-t}\|\nabla_{u_{\ell-1-s}}u_{\ell-s} \| \\
    &\qquad+ \|(\theta_L\varphi_\ell (\theta_\ell;u_\ell)-\EE_{\Omega_i}\theta_L\varphi_\ell(\theta_\ell;u_\ell))^\top \theta_L\| \|\nabla_{u_{\ell}}\varphi_\ell (\theta_\ell;u_\ell) \|\\
    &\qquad+\| (\theta_L\varphi_\ell (\theta_\ell;u_\ell)-\EE_{\Omega_i}\theta_L\varphi_\ell(\theta_\ell;u_\ell))^\top \theta_L \| \Big)
\end{align*}

Summing over all the above inner product, we have
\begin{align*}
    &\ip{\nabla_{\theta}\EE_{\Omega_i} \|\theta_L\varphi_\ell (\theta_\ell;u_\ell)-\EE_{\Omega_i}\theta_L\varphi_\ell(\theta_\ell;u_\ell)\|^2}{\theta}\\
    \ge& 2\left(\sum_{r=0}^{\ell}s_{i,r,\rm in}+\ell+2\right)\EE_{\Omega_i} \|\theta_L\varphi_\ell (\theta_\ell;u_\ell)-\EE_{\Omega_i}\theta_L\varphi_\ell(\theta_\ell;u_\ell)\|^2-2\left({C_{i,\ell}}+\sum_{r=0}^{\ell-1}{C_{i,\ell,r,\rm in}}+{C_{i,\ell,r,\rm out}}\right)\,\epsilon_{\rm ah}
\end{align*}

Similarly, 
\begin{align*}
    &\nabla_\theta\|\EE_{\Omega_i}\theta_L\varphi_\ell(\theta_\ell;u_\ell)-\theta_L\varphi_\ell(\theta_\ell;u_{\ell,i})\|^2\\
    =&2(\EE_{\Omega_i}\theta_L\varphi_\ell(\theta_\ell;u_\ell)-\theta_L\varphi_\ell(\theta_\ell;u_{\ell,i}))^\top \nabla_\theta\EE_{\Omega_i}\theta_L\varphi_\ell(\theta_\ell;u_\ell)\\
    &\qquad- 2(\EE_{\Omega_i}\theta_L\varphi_\ell(\theta_\ell;u_\ell)-\theta_L\varphi_\ell(\theta_\ell;u_{\ell,i}))^\top \nabla_\theta \theta_L\varphi_\ell(\theta_\ell;u_{\ell,i})\\
    =&2(\EE_{\Omega_i}\theta_L\varphi_\ell(\theta_\ell;u_\ell)-\theta_L\varphi_\ell(\theta_\ell;u_{\ell,i}))^\top \EE_{\Omega_i}\nabla_\theta\theta_L\varphi_\ell(\theta_\ell;u_\ell)\\
    &\qquad- 2(\EE_{\Omega_i}\theta_L\varphi_\ell(\theta_\ell;u_\ell)-\theta_L\varphi_\ell(\theta_\ell;u_{\ell,i}))^\top \nabla_\theta \theta_L\varphi_\ell(\theta_\ell;u_{\ell,i})
\end{align*}
where the second inequality follows from exchanging the integration and gradient.

Then
\begin{align*}
     &\ip{ \nabla_{\theta_{\ell,\rm out}}\|\EE_{\Omega_i}\theta_L\varphi_\ell(\theta_\ell;u_\ell)-\theta_L\varphi_\ell(\theta_\ell;u_{\ell,i})\|^2}{\theta_{\ell,\rm out}}=2\|\EE_{\Omega_i}\theta_L\varphi_\ell(\theta_\ell;u_\ell)-\theta_L\varphi_\ell(\theta_\ell;u_{\ell,i})\|^2,\\
     & \ip{ \nabla_{\theta_{L}}\|\EE_{\Omega_i}\theta_L\varphi_\ell(\theta_\ell;u_\ell)-\theta_L\varphi_\ell(\theta_\ell;u_{\ell,i})\|^2}{\theta_{L}}=2\|\EE_{\Omega_i}\theta_L\varphi_\ell(\theta_\ell;u_\ell)-\theta_L\varphi_\ell(\theta_\ell;u_{\ell,i})\|^2.
\end{align*}
Also
\begin{align*}
    &\ip{ \nabla_{\theta_{\ell,\rm in}}\|\EE_{\Omega_i}\theta_L\varphi_\ell(\theta_\ell;u_\ell)-\theta_L\varphi_\ell(\theta_\ell;u_{\ell,i})\|^2}{\theta_{\ell,\rm in}}\\ &\ge 2s_{i,\ell,\rm in}\|\EE_{\Omega_i}\theta_L\varphi_\ell(\theta_\ell;u_\ell)-\theta_L\varphi_\ell(\theta_\ell;u_{\ell,i})\|^2-4 {C_{i,\ell}}\,\epsilon_{\rm ah}
\end{align*}
where the error $4 {C_{i,\ell}}\,\epsilon_{\rm ah}$ doubles compared to that of $\EE_{\Omega_i} \|\theta_L\varphi_\ell (\theta_\ell;u_\ell)-\EE_{\Omega_i}\theta_L\varphi_\ell(\theta_\ell;u_\ell)\|^2$. We also have
\begin{align*}
     &\ip{\nabla_{\theta_{r,\rm in}}\|\EE_{\Omega_i}\theta_L\varphi_\ell(\theta_\ell;u_\ell)-\theta_L\varphi_\ell(\theta_\ell;u_{\ell,i})\|^2}{\theta_{r,\rm in}}
    \\ &\ge  2s_{i,r,\rm in}\|\EE_{\Omega_i}\theta_L\varphi_\ell(\theta_\ell;u_\ell)-\theta_L\varphi_\ell(\theta_\ell;u_{\ell,i})\|^2 - 4{C_{i,\ell,r,\rm in}}\,\epsilon_{\rm ah} 
\end{align*}

Thus
\begin{align*}
    &\ip{\nabla_\theta\|\EE_{\Omega_i}\theta_L\varphi_\ell(\theta_\ell;u_\ell)-\theta_L\varphi_\ell(\theta_\ell;u_{\ell,i})\|^2}{\theta}\\
    \ge& 2\left(\sum_{r=0}^{\ell}s_{i,r,\rm in}+\ell+2\right)\|\EE_{\Omega_i}\theta_L\varphi_\ell(\theta_\ell;u_\ell)-\theta_L\varphi_\ell(\theta_\ell;u_{\ell,i})\|^2-4\left({C_{i,\ell}}+\sum_{r=0}^{\ell-1}{C_{i,\ell,r,\rm in}}+{C_{i,\ell,r,\rm out}}\right)\,\epsilon_{\rm ah}
\end{align*}

For the combined quantity, apply the same calculation directly to
\begin{align*}
    D_{\ell,i}(x;\theta)=\theta_L\varphi_\ell(\theta_\ell;u_\ell(x))-\theta_L\varphi_\ell(\theta_\ell;u_{\ell,i}).
\end{align*}
Both terms in $D_{\ell,i}$ have the same homogeneity degrees in the parameters $\theta_L,\theta_{\ell,\rm out},\theta_{\ell,\rm in}$ and in the preceding layers. Hence the homogeneous part gives
\begin{align*}
    2\left(\sum_{r=0}^{\ell}s_{i,r,\rm in}+\ell+2\right)\EE_{\Omega_i}\|D_{\ell,i}(x;\theta)\|^2.
\end{align*}
The approximate homogeneity remainder appears once at $x$ and once at $x_i$, so
\begin{align*}
    &\ip{\nabla_{\theta}\EE_{\Omega_i}\|\theta_L\varphi_\ell(\theta_\ell;u_\ell)-\theta_L\varphi_\ell(\theta_\ell;u_{\ell,i})\|^2}{\theta}
    \\ &\ge 2\left(\sum_{r=0}^{\ell}s_{i,r,\rm in}+\ell+2\right)\EE_{\Omega_i}\|\theta_L\varphi_\ell(\theta_\ell;u_\ell)-\theta_L\varphi_\ell(\theta_\ell;u_{\ell,i})\|^2\\
    &\quad-4\left({C^{\rm comb}_{i,\ell}}+\sum_{r=0}^{\ell-1}\left({C^{\rm comb}_{i,\ell,r,\rm in}}+{C^{\rm comb}_{i,\ell,r,\rm out}}\right)\right)\epsilon_{\rm ah}.
\end{align*}

Also,
\begin{align*}
   \ip{\nabla_\theta\EE_{\Omega_i}\|\theta_Lx-\theta_L x_i\|^2}{\theta}&= \ip{\nabla_{\theta_L}\EE_{\Omega_i}\|\theta_Lx-\theta_L x_i\|^2}{\theta_L}\\
   &=2\EE_{\Omega_i}\|\theta_Lx-\theta_L x_i\|^2
\end{align*}
\end{proof}

\begin{corollary}
    \label{cor:difference_theta_in_and_out}
    Following the same assumptions as Lemma~\ref{lem:gen_nabla_var_theta_lower_bound}, we have
    \begin{align*}
        \ip{\nabla_{\theta_{\ell,\rm out}}V_{\ell,\Omega_i}(\theta)}{\theta_{\ell,\rm out}}=2V_{\ell,\Omega_i}(\theta)
    \end{align*}
    and
    \begin{align*}
        \ip{\nabla_{\theta_{r,\rm in}}V_{\ell,\Omega_i}(\theta)}{\theta_{r,\rm in}}
    \ge  2s_{r,\rm in}V_{\ell,\Omega_i}(\theta) - 4{C_{i,\ell,r,\rm in}}\,\epsilon_{\rm ah} 
    \end{align*}
\end{corollary}
\begin{proof}
    The proof is a direct results of Lemma~\ref{lem:gen_nabla_var_theta_lower_bound}.
\end{proof}

\begin{lemma}
    \label{lem:individual_variance_decay}
Under Assumption~\ref{assump:model_assumptions}, \ref{assump:architecture_generalization}, and~\ref{assump:approximate_homogeneity},
    \begin{align*}
    &V_{\ell,\Omega_i,\rm var}(\theta^{K_0+k+1})\\ \le&  \left(1-2\left(\sum_{r=0}^{\ell}s_{i,r,\rm in}+\ell+2\right)\eta\lambda \right)^{k+1}V_{\ell,\Omega_i,\rm var}(\theta^{K_0}) +\frac{{C_{i,\ell}}+\sum_{r=0}^{\ell-1}{C_{i,\ell,r,\rm in}}+{C_{i,\ell,r,\rm out}}}{\sum_{r=0}^{\ell}s_{i,r,\rm in}+\ell+2}\epsilon_{\rm ah}\\
    &\quad +\frac{\sup_{\theta\in B_\bzero(R_{\lambda})}\|{\nabla_\theta V_{\ell,\Omega_i,\rm var}(\theta)}\|}{2\lambda\left(\sum_{r=0}^{\ell}s_{i,r,\rm in}+\ell+2\right)}\max_{K_0\le j\le K_0+k}\|\nabla\cL(\theta^j)\| + \frac{S_{V_{\ell,\Omega_i,\rm var}}}{2-\eta\,\mathsf{L}_\lambda}\cLwd(\theta^{K_0})
\end{align*}
    \begin{align*}
    &M_{\ell,\Omega_i}(\theta^{K_0+k+1}) 
    \\ &\le  \left(1-2\left(\sum_{r=0}^{\ell}s_{i,r,\rm in}+\ell+2\right)\eta\lambda \right)^{k+1}M_{\ell,\Omega_i}(\theta^{K_0}) +2\frac{{C_{i,\ell}}+\sum_{r=0}^{\ell-1}{C_{i,\ell,r,\rm in}}+{C_{i,\ell,r,\rm out}}}{\sum_{r=0}^{\ell}s_{i,r,\rm in}+\ell+2}\epsilon_{\rm ah}\\
    &\quad +\frac{\sup_{\theta\in B_\bzero(R_{\lambda})}\|{\nabla_\theta M_{\ell,\Omega_i}(\theta)}\|}{2\lambda\left(\sum_{r=0}^{\ell}s_{i,r,\rm in}+\ell+2\right)}\max_{K_0\le j\le K_0+k}\|\nabla\cL(\theta^j)\| + \frac{S_{M_{\ell,\Omega_i}}}{2-\eta\,\mathsf{L}_\lambda}\cLwd(\theta^{K_0})
\end{align*}
    \begin{align*}
    & V_{\ell,\Omega_i}(\theta^{K_0+k+1})\\ \le&  \left(1-2\left(\sum_{r=0}^{\ell}s_{i,r,\rm in}+\ell+2\right)\eta\lambda \right)^{k+1}V_{\ell,\Omega_i}(\theta^{K_0}) +2\frac{{C^{\rm comb}_{i,\ell}}+\sum_{r=0}^{\ell-1}\left({C^{\rm comb}_{i,\ell,r,\rm in}}+{C^{\rm comb}_{i,\ell,r,\rm out}}\right)}{\sum_{r=0}^{\ell}s_{i,r,\rm in}+\ell+2}\epsilon_{\rm ah}\\
    &\quad +\frac{\sup_{\theta\in B_\bzero(R_{\lambda})}\|{\nabla_\theta V_{\ell,\Omega_i}(\theta)}\|}{2\lambda\left(\sum_{r=0}^{\ell}s_{i,r,\rm in}+\ell+2\right)}\max_{K_0\le j\le K_0+k}\|\nabla\cL(\theta^j)\| + \frac{S_{V_{\ell,\Omega_i}}}{2-\eta\,\mathsf{L}_\lambda}\cLwd(\theta^{K_0})
\end{align*}
    \begin{align*}
    & V_{{\rm pre},\Omega_i}(\theta^{K_0+k+1}) 
    \\ &\le  \left(1-2\eta\lambda \right)^{k+1}V_{{\rm pre},\Omega_i}(\theta^{K_0})\\
    &\quad +\frac{\sup_{\theta\in B_\bzero(R_{\lambda})}\|{\nabla_\theta V_{{\rm pre},\Omega_i}(\theta)}\|}{2\lambda}\max_{K_0\le j\le K_0+k}\|\nabla\cL(\theta^j)\| + \frac{S_{V_{{\rm pre},\Omega_i}}}{2-\eta\,\mathsf{L}_\lambda}\cLwd(\theta^{K_0})
\end{align*}
for some constants $S_{V_{\ell,\Omega_i,\rm var}},S_{V_{\ell,\Omega_i}},S_{V_{{\rm pre},\Omega_i}}, S_{M_{\ell,\Omega_i}}>0$.
\end{lemma}
\begin{proof}

By Assumption~\ref{assump:model_assumptions}, Proposition 1 in\citet{wang2026convergencegradientdescentgeneral}, and Proposition~\ref{prop:poly_smooth_preserved_under_integration}, we have that $ V_{\ell,\Omega_i,\rm var}(\theta)$, $V_{\ell,\Omega_i}(\theta)$, $V_{{\rm pre},\Omega_i}(\theta)$, and $M_{\ell,\Omega_i}(\theta)$ are all poly-smooth with polynomial $S_{1,\ell},S_{4,\ell},S_2,S_{3,\ell}$. Then define
\begin{align*}
    S_{V_{\ell,\Omega_i,\rm var}}&=S_{1,\ell}\left(\cdots,R_{\lambda}+\sup_{\theta\in B_\bzero(R_{\lambda})}\|\nabla_{\theta_i}\cLwd(\theta)\|,\cdots\right)\\
    S_{V_{\ell,\Omega_i}}&=S_{4,\ell}\left(\cdots,R_{\lambda}+\sup_{\theta\in B_\bzero(R_{\lambda})}\|\nabla_{\theta_i}\cLwd(\theta)\|,\cdots\right)\\
    S_{V_{{\rm pre},\Omega_i}}&=S_2\left(\cdots,R_{\lambda}+\sup_{\theta\in B_\bzero(R_{\lambda})}\|\nabla_{\theta_i}\cLwd(\theta)\|,\cdots\right)\\
    S_{M_{\ell,\Omega_i}}&=S_{3,\ell}\left(\cdots,R_{\lambda}+\sup_{\theta\in B_\bzero(R_{\lambda})}\|\nabla_{\theta_i}\cLwd(\theta)\|,\cdots\right)
\end{align*}
which are the local Lipschitz smooth constants of the four functions in $\theta\in B_\bzero(R_{\lambda})$.
    
Then consider $ V_{\ell,\Omega_i,\rm var}(\theta)$,
\begin{align*}
    &V_{\ell,\Omega_i,\rm var}(\theta^{K_0+k+1})\\ \le& V_{\ell,\Omega_i,\rm var}(\theta^{K_0+k})+\nabla_\theta V_{\ell,\Omega_i,\rm var}(\theta^{K_0+k})^\top (\theta^{K_0+k+1}-\theta^{K_0+k})+\frac{S_{V_{\ell,\Omega_i,\rm var}}}{2}\|\theta^{K_0+k+1}-\theta^{K_0+k}\|^2\\
    =& V_{\ell,\Omega_i,\rm var}(\theta^{K_0+k})-\eta\ip{\nabla_\theta V_{\ell,\Omega_i,\rm var}(\theta^{K_0+k})}{\nabla \cLwd(\theta^{K_0+k})}+\frac{S_{V_{\ell,\Omega_i,\rm var}}}{2}\eta^2\|\nabla \cLwd(\theta^{K_0+k})\|^2\\
    =& V_{\ell,\Omega_i,\rm var}(\theta^{K_0+k})-\eta\lambda\ip{\nabla_\theta V_{\ell,\Omega_i,\rm var}(\theta^{K_0+k})}{\theta^{K_0+k}}-\eta\ip{\nabla_\theta V_{\ell,\Omega_i,\rm var}(\theta^{K_0+k})}{\nabla \cL(\theta^{K_0+k})}\\
    &\quad+\frac{S_{V_{\ell,\Omega_i,\rm var}}}{2}\eta^2\|\nabla \cLwd(\theta^{K_0+k})\|^2\\
    &\le \left(1-2\left(\sum_{r=0}^{\ell}s_{i,r,\rm in}+\ell+2\right)\eta\lambda \right)V_{\ell,\Omega_i,\rm var}(\theta^{K_0+k})+2\eta\lambda\left({C_{i,\ell}}+\sum_{r=0}^{\ell-1}{C_{i,\ell,r,\rm in}}+{C_{i,\ell,r,\rm out}}\right)\,\epsilon_{\rm ah}\\
    &\quad+\eta\|{\nabla_\theta V_{\ell,\Omega_i,\rm var}(\theta^{K_0+k})}\| \|{\nabla \cL(\theta^{K_0+k})}\|+\frac{S_{V_{\ell,\Omega_i,\rm var}}}{2}\eta^2\|\nabla \cLwd(\theta^{K_0+k})\|^2\\
    &\le \left(1-2\left(\sum_{r=0}^{\ell}s_{i,r,\rm in}+\ell+2\right)\eta\lambda \right)V_{\ell,\Omega_i,\rm var}(\theta^{K_0+k})+2\eta\lambda\left({C_{i,\ell}}+\sum_{r=0}^{\ell-1}{C_{i,\ell,r,\rm in}}+{C_{i,\ell,r,\rm out}}\right)\,\epsilon_{\rm ah}\\
    &\quad+\eta\sup_{\theta\in B_\bzero(R_{\lambda})}\|{\nabla_\theta V_{\ell,\Omega_i,\rm var}(\theta)}\| \|{\nabla \cL(\theta^{K_0+k})}\|+\frac{S_{V_{\ell,\Omega_i,\rm var}}}{2}\eta^2\|\nabla \cLwd(\theta^{K_0+k})\|^2
\end{align*}

By Gronwall's inequality,
\begin{align*}
    &V_{\ell,\Omega_i,\rm var}(\theta^{K_0+k+1})\\ &\le \left(1-2\left(\sum_{r=0}^{\ell}s_{i,r,\rm in}+\ell+2\right)\eta\lambda \right)^{k+1}V_{\ell,\Omega_i,\rm var}(\theta^{K_0})\\
    &\quad+\sum_{j=0}^{k}\left(1-2\left(\sum_{r=0}^{\ell}s_{i,r,\rm in}+\ell+2\right)\eta\lambda \right)^{k-j}\bigg[ 2\eta\lambda\left({C_{i,\ell}}+\sum_{r=0}^{\ell-1}{C_{i,\ell,r,\rm in}}+{C_{i,\ell,r,\rm out}}\right)\,\epsilon_{\rm ah}\\
    &\quad+\eta\sup_{\theta\in B_\bzero(R_{\lambda})}\|{\nabla_\theta V_{\ell,\Omega_i,\rm var}(\theta)}\| \|{\nabla \cL(\theta^{K_0+j})}\|+\frac{S_{V_{\ell,\Omega_i,\rm var}}}{2}\eta^2\|\nabla \cLwd(\theta^{K_0+j})\|^2\bigg]\\
    &\le  \left(1-2\left(\sum_{r=0}^{\ell}s_{i,r,\rm in}+\ell+2\right)\eta\lambda \right)^{k+1}V_{\ell,\Omega_i,\rm var}(\theta^{K_0}) +\frac{{C_{i,\ell}}+\sum_{r=0}^{\ell-1}{C_{i,\ell,r,\rm in}}+{C_{i,\ell,r,\rm out}}}{\sum_{r=0}^{\ell}s_{i,r,\rm in}+\ell+2}\epsilon_{\rm ah}\\
    &\quad +\frac{\sup_{\theta\in B_\bzero(R_{\lambda})}\|{\nabla_\theta V_{\ell,\Omega_i,\rm var}(\theta)}\|}{2\lambda\left(\sum_{r=0}^{\ell}s_{i,r,\rm in}+\ell+2\right)}\max_{K_0\le j\le K_0+k}\|\nabla\cL(\theta^j)\| + \frac{S_{V_{\ell,\Omega_i,\rm var}}}{2-\eta\,\mathsf{L}_\lambda}\cLwd(\theta^{K_0})
\end{align*}
where the last inequality follows from Lemma~\ref{lem:upper_bound_theta_sum_nabla_cL_along_GD}.

 Similarly, we have
    \begin{align*}
       & M_{\ell,\Omega_i}(\theta^{K_0+k+1})
   \\ \le &\left(1-2\left(\sum_{r=0}^{\ell}s_{i,r,\rm in}+\ell+2\right)\eta\lambda \right)M_{\ell,\Omega_i}(\theta^{K_0+k})+4\eta\lambda\left({C_{i,\ell}}+\sum_{r=0}^{\ell-1}{C_{i,\ell,r,\rm in}}+{C_{i,\ell,r,\rm out}}\right)\,\epsilon_{\rm ah}\\
    &\quad+\eta\sup_{\theta\in B_\bzero(R_{\lambda})}\|{\nabla_\theta M_{\ell,\Omega_i}(\theta)}\| \|{\nabla \cL(\theta^{K_0+k})}\|+\frac{S_{M_{\ell,\Omega_i}}}{2}\eta^2\|\nabla \cLwd(\theta^{K_0+k})\|^2
\end{align*}
and by Gronwall's inequality, we have
\begin{align*}
    & M_{\ell,\Omega_i}(\theta^{K_0+k+1}) 
    \\ &\le  \left(1-2\left(\sum_{r=0}^{\ell}s_{i,r,\rm in}+\ell+2\right)\eta\lambda \right)^{k+1}M_{\ell,\Omega_i}(\theta^{K_0}) +2\frac{{C_{i,\ell}}+\sum_{r=0}^{\ell-1}{C_{i,\ell,r,\rm in}}+{C_{i,\ell,r,\rm out}}}{\sum_{r=0}^{\ell}s_{i,r,\rm in}+\ell+2}\epsilon_{\rm ah}\\
    &\quad +\frac{\sup_{\theta\in B_\bzero(R_{\lambda})}\|{\nabla_\theta M_{\ell,\Omega_i}(\theta)}\|}{2\lambda\left(\sum_{r=0}^{\ell}s_{i,r,\rm in}+\ell+2\right)}\max_{K_0\le j\le K_0+k}\|\nabla\cL(\theta^j)\| + \frac{S_{M_{\ell,\Omega_i}}}{2-\eta\,\mathsf{L}_\lambda}\cLwd(\theta^{K_0})
\end{align*}

For the combined quantity $V_{\ell,\Omega_i}$, Lemma~\ref{lem:gen_nabla_var_theta_lower_bound} gives
\begin{align*}
    V_{\ell,\Omega_i}(\theta^{K_0+k+1})\le& V_{\ell,\Omega_i}(\theta^{K_0+k})-\eta\ip{\nabla_\theta V_{\ell,\Omega_i}(\theta^{K_0+k})}{\nabla \cLwd(\theta^{K_0+k})}\\
    &\quad+\frac{S_{V_{\ell,\Omega_i}}}{2}\eta^2\|\nabla \cLwd(\theta^{K_0+k})\|^2\\
    =& V_{\ell,\Omega_i}(\theta^{K_0+k})-\eta\lambda\ip{\nabla_\theta V_{\ell,\Omega_i}(\theta^{K_0+k})}{\theta^{K_0+k}}\\
    &\quad-\eta\ip{\nabla_\theta V_{\ell,\Omega_i}(\theta^{K_0+k})}{\nabla \cL(\theta^{K_0+k})}+\frac{S_{V_{\ell,\Omega_i}}}{2}\eta^2\|\nabla \cLwd(\theta^{K_0+k})\|^2\\
    \le& \left(1-2\left(\sum_{r=0}^{\ell}s_{i,r,\rm in}+\ell+2\right)\eta\lambda \right)V_{\ell,\Omega_i}(\theta^{K_0+k})\\
    &\quad+4\eta\lambda\left({C^{\rm comb}_{i,\ell}}+\sum_{r=0}^{\ell-1}\left({C^{\rm comb}_{i,\ell,r,\rm in}}+{C^{\rm comb}_{i,\ell,r,\rm out}}\right)\right)\epsilon_{\rm ah}\\
    &\quad+\eta\sup_{\theta\in B_\bzero(R_{\lambda})}\|{\nabla_\theta V_{\ell,\Omega_i}(\theta)}\| \|{\nabla \cL(\theta^{K_0+k})}\|+\frac{S_{V_{\ell,\Omega_i}}}{2}\eta^2\|\nabla \cLwd(\theta^{K_0+k})\|^2.
\end{align*}
Applying Gronwall's inequality and Lemma~\ref{lem:upper_bound_theta_sum_nabla_cL_along_GD} gives
\begin{align*}
    V_{\ell,\Omega_i}(\theta^{K_0+k+1})\le&  \left(1-2\left(\sum_{r=0}^{\ell}s_{i,r,\rm in}+\ell+2\right)\eta\lambda \right)^{k+1}V_{\ell,\Omega_i}(\theta^{K_0})\\
    &\quad +2\frac{{C^{\rm comb}_{i,\ell}}+\sum_{r=0}^{\ell-1}\left({C^{\rm comb}_{i,\ell,r,\rm in}}+{C^{\rm comb}_{i,\ell,r,\rm out}}\right)}{\sum_{r=0}^{\ell}s_{i,r,\rm in}+\ell+2}\epsilon_{\rm ah}\\
    &\quad +\frac{\sup_{\theta\in B_\bzero(R_{\lambda})}\|{\nabla_\theta V_{\ell,\Omega_i}(\theta)}\|}{2\lambda\left(\sum_{r=0}^{\ell}s_{i,r,\rm in}+\ell+2\right)}\max_{K_0\le j\le K_0+k}\|\nabla\cL(\theta^j)\| + \frac{S_{V_{\ell,\Omega_i}}}{2-\eta\,\mathsf{L}_\lambda}\cLwd(\theta^{K_0}).
\end{align*}

For $V_{{\rm pre},\Omega_i}$, we have
\begin{align*}
     & V_{{\rm pre},\Omega_i}(\theta^{K_0+k+1})\\ \le& V_{{\rm pre},\Omega_i}(\theta^{K_0+k})+\nabla_\theta V_{{\rm pre},\Omega_i}(\theta^{K_0+k})^\top (\theta^{K_0+k+1}-\theta^{K_0+k})+\frac{S_{V_{{\rm pre},\Omega_i}}}{2}\|\theta^{K_0+k+1}-\theta^{K_0+k}\|^2\\
    =& V_{{\rm pre},\Omega_i}(\theta^{K_0+k})-\eta\ip{\nabla_\theta V_{{\rm pre},\Omega_i}(\theta^{K_0+k})}{\nabla \cLwd(\theta^{K_0+k})}+\frac{S_{V_{{\rm pre},\Omega_i}}}{2}\eta^2\|\nabla \cLwd(\theta^{K_0+k})\|^2\\
    \le &(1-2\eta\lambda)V_{{\rm pre},\Omega_i}(\theta^{K_0+k})+\eta\sup_{\theta\in B_\bzero(R_{\lambda})}\|{\nabla_\theta V_{{\rm pre},\Omega_i}(\theta)}\| \|{\nabla \cL(\theta^{K_0+k})}\|\\
    &\quad+\frac{S_{V_{{\rm pre},\Omega_i}}}{2}\eta^2\|\nabla \cLwd(\theta^{K_0+k})\|^2
\end{align*}
and by Gronwall's inequality,
\begin{align*}
    V_{{\rm pre},\Omega_i}(\theta^{K_0+k+1}) 
    &\le  \left(1-2\eta\lambda \right)^{k+1}V_{{\rm pre},\Omega_i}(\theta^{K_0})\\
    &\quad +\frac{\sup_{\theta\in B_\bzero(R_{\lambda})}\|{\nabla_\theta V_{{\rm pre},\Omega_i}(\theta)}\|}{2\lambda}\max_{K_0\le j\le K_0+k}\|\nabla\cL(\theta^j)\| + \frac{S_{V_{{\rm pre},\Omega_i}}}{2-\eta\,\mathsf{L}_\lambda}\cLwd(\theta^{K_0})
\end{align*}

\end{proof}

\begin{lemma}
    \label{lem:order_of_constants}
    There exist constants $C_{{\rm ah},V_\ell},C_{{\rm ah},V_{\ell,\rm var}}>0$, independent of the layer index $\ell$, such that the following quantities have at most quadratic dependence on $\ell$, i.e.,
    \begin{align*}
        {C^{\rm comb}_{i,\ell}}+\sum_{r=0}^{\ell-1}\left({C^{\rm comb}_{i,\ell,r,\rm in}}+{C^{\rm comb}_{i,\ell,r,\rm out}}\right)&\le C_{{\rm ah},V_\ell}(\ell+2)^2,\\
        {C_{i,\ell}}+\sum_{r=0}^{\ell-1}\left({C_{i,\ell,r,\rm in}}+{C_{i,\ell,r,\rm out}}\right)&\le C_{{\rm ah},V_{\ell,\rm var}}(\ell+2)^2.
    \end{align*}
    Consequently, 
    \begin{align*}
        \frac{{C^{\rm comb}_{i,\ell}}+\sum_{r=0}^{\ell-1}\left({C^{\rm comb}_{i,\ell,r,\rm in}}+{C^{\rm comb}_{i,\ell,r,\rm out}}\right)}{\sum_{r=0}^{\ell}s_{i,r,\rm in}+\ell+2}
        &\le C_{{\rm ah},V_\ell}(\ell+2),\\
        \frac{{C_{i,\ell}}+\sum_{r=0}^{\ell-1}\left({C_{i,\ell,r,\rm in}}+{C_{i,\ell,r,\rm out}}\right)}{\sum_{r=0}^{\ell}s_{i,r,\rm in}+\ell+2}
        &\le C_{{\rm ah},V_{\ell,\rm var}}(\ell+2).
    \end{align*}
\end{lemma}
\begin{proof}
    Let $D_{\ell,i}(x;\theta)=\theta_L\varphi_\ell(\theta_\ell;u_\ell(x))-\theta_L\varphi_\ell(\theta_\ell;u_{\ell,i})$. Define
    \begin{align*}
        M_0&=\max_{1\le i\le N}\sup_{\theta\in B_\bzero(R_{\lambda})}\sup_{x\in\Omega_i}\|D_{\ell,i}(x;\theta)^\top\theta_L\theta_{\ell,\rm out}\|,\\
        M_1&=\max_{1\le i\le N}\sup_{\theta\in B_\bzero(R_{\lambda})}\sup_{x\in\Omega_i}\|D_{\ell,i}(x;\theta)^\top\theta_L\|,\\
        M_2&=\max_{1\le i\le N}\sup_{\theta\in B_\bzero(R_{\lambda})}\sup_{z\in\Omega_i}\|\nabla_{u_\ell}\varphi_\ell(\theta_\ell;u_\ell(z))\|,\\
        M_3&=\max_{1\le i\le N}\max_{0\le r\le t\le \ell-1}\sup_{\theta\in B_\bzero(R_{\lambda})}\sup_{z\in\Omega_i}\prod_{s=0}^{\ell-1-t}\|\nabla_{u_{\ell-1-s}}u_{\ell-s}(z)\|.
    \end{align*}
    These constants upper bound the corresponding factors in the definitions of ${C^{\rm comb}_{i,\ell}}$, ${C^{\rm comb}_{i,\ell,r,\rm in}}$, and ${C^{\rm comb}_{i,\ell,r,\rm out}}$. Hence
    \begin{align*}
        {C^{\rm comb}_{i,\ell}}\le M_0.
    \end{align*}
    Moreover, for each $r=0,\cdots,\ell-1$, each of ${C^{\rm comb}_{i,\ell,r,\rm in}}$ and ${C^{\rm comb}_{i,\ell,r,\rm out}}$ is bounded by
    \begin{align*}
        M_1\left(\sum_{t=r}^{\ell-1}M_2M_3+M_2+1\right)
        \le M_1(M_2M_3+M_2+1)(\ell-r+1).
    \end{align*}
    Therefore, for $A=M_1(M_2M_3+M_2+1)$,
    \begin{align*}
        {C^{\rm comb}_{i,\ell}}+\sum_{r=0}^{\ell-1}\left({C^{\rm comb}_{i,\ell,r,\rm in}}+{C^{\rm comb}_{i,\ell,r,\rm out}}\right)
        &\le M_0+2A\sum_{r=0}^{\ell-1}(\ell-r+1)\\
        &=M_0+2A\sum_{q=1}^{\ell}(q+1)\\
        &\le C_{{\rm ah},V_\ell}(\ell+2)^2
    \end{align*}
    for a constant $C_{{\rm ah},V_\ell}>0$. Since $s_{i,r,\rm in}\ge0$,
    \begin{align*}
        \sum_{r=0}^{\ell}s_{i,r,\rm in}+\ell+2\ge \ell+2.
    \end{align*}
    Thus
    \begin{align*}
        \frac{{C^{\rm comb}_{i,\ell}}+\sum_{r=0}^{\ell-1}\left({C^{\rm comb}_{i,\ell,r,\rm in}}+{C^{\rm comb}_{i,\ell,r,\rm out}}\right)}{\sum_{r=0}^{\ell}s_{i,r,\rm in}+\ell+2}
        \le \frac{C_{{\rm ah},V_\ell}(\ell+2)^2}{\ell+2}
        = C_{{\rm ah},V_\ell}(\ell+2),
    \end{align*}
    
    The proof for the non-combined constants is the same. Let
    \begin{align*}
        \widetilde D_{\ell,i}(x;\theta)=\theta_L\varphi_\ell(\theta_\ell;u_\ell(x))-\EE_{\Omega_i}\theta_L\varphi_\ell(\theta_\ell;u_\ell).
    \end{align*}
    Define $\widetilde M_0,\widetilde M_1,\widetilde M_2,\widetilde M_3$ as the same suprema as $M_0,M_1,M_2,M_3$, with $D_{\ell,i}$ replaced by $\widetilde D_{\ell,i}$. Then
    \begin{align*}
        {C_{i,\ell}}\le \widetilde M_0,
    \end{align*}
    and for each $r=0,\cdots,\ell-1$, each of ${C_{i,\ell,r,\rm in}}$ and ${C_{i,\ell,r,\rm out}}$ is bounded by
    \begin{align*}
        \widetilde M_1(\widetilde M_2\widetilde M_3+\widetilde M_2+1)(\ell-r+1).
    \end{align*}
    Therefore, for some constant $C_{{\rm ah},V_{\ell,\rm var}}>0$,
    \begin{align*}
        {C_{i,\ell}}+\sum_{r=0}^{\ell-1}\left({C_{i,\ell,r,\rm in}}+{C_{i,\ell,r,\rm out}}\right)
        \le C_{{\rm ah},V_{\ell,\rm var}}(\ell+2)^2.
    \end{align*}
    Dividing by $\sum_{r=0}^{\ell}s_{i,r,\rm in}+\ell+2\ge \ell+2$ gives
    \begin{align*}
        \frac{{C_{i,\ell}}+\sum_{r=0}^{\ell-1}\left({C_{i,\ell,r,\rm in}}+{C_{i,\ell,r,\rm out}}\right)}{\sum_{r=0}^{\ell}s_{i,r,\rm in}+\ell+2}
        \le C_{{\rm ah},V_{\ell,\rm var}}(\ell+2),
    \end{align*}
    Since there are finitely many layers, the constants $C_{{\rm ah},V_\ell}$ and $C_{{\rm ah},V_{\ell,\rm var}}$ can be chosen uniformly over $0\le \ell\le L-1$, and hence independently of $\ell$.
\end{proof}

\begin{lemma}
    \label{lem:V_M_decay}
    Under Assumption~\ref{assump:model_assumptions}, \ref{assump:architecture_generalization}, and~\ref{assump:approximate_homogeneity}, let $s_{r,\rm in}:=\min_{1\le i\le N}s_{i,r,\rm in}$. Then, for every $k\ge0$,
    \begin{align*}
     V_{\ell,\rm var}(\theta^{K_0+k})\le&  \left(1-2\left(\sum_{r=0}^{\ell}s_{r,\rm in}+\ell+2\right)\eta\lambda \right)^{k}V_{\ell,\rm var}(\theta^{K_0}) +C_{{\rm ah},V_{\ell,\rm var}}(\ell+2)\,\epsilon_{\rm ah}\\
    &\quad +\frac{C_{{\rm grad},V_{\ell,\rm var}}}{\lambda(\ell+2)}\max_{K_0\le j\le K_0+k}\|\nabla\cL(\theta^j)\| + \frac{S_{V_{\ell,\rm var}}}{2-\eta\,\mathsf{L}_\lambda}\cLwd(\theta^{K_0}),
\end{align*}
\begin{align*}
     M_{\ell}(\theta^{K_0+k})\le&  \left(1-2\left(\sum_{r=0}^{\ell}s_{r,\rm in}+\ell+2\right)\eta\lambda \right)^{k}M_{\ell}(\theta^{K_0}) +2C_{{\rm ah},V_{\ell,\rm var}}(\ell+2)\,\epsilon_{\rm ah}\\
    &\quad +\frac{C_{{\rm grad},M_{\ell}}}{\lambda(\ell+2)}\max_{K_0\le j\le K_0+k}\|\nabla\cL(\theta^j)\| + \frac{S_{M_{\ell}}}{2-\eta\,\mathsf{L}_\lambda}\cLwd(\theta^{K_0}),
\end{align*}
\begin{align*}
     V_{\ell}(\theta^{K_0+k})\le&  \left(1-2\left(\sum_{r=0}^{\ell}s_{r,\rm in}+\ell+2\right)\eta\lambda \right)^{k}V_{\ell}(\theta^{K_0}) +2C_{{\rm ah},V_\ell}(\ell+2)\,\epsilon_{\rm ah}\\
    &\quad +\frac{C_{{\rm grad},V_{\ell}}}{\lambda(\ell+2)}\max_{K_0\le j\le K_0+k}\|\nabla\cL(\theta^j)\| + \frac{S_{V_{\ell}}}{2-\eta\,\mathsf{L}_\lambda}\cLwd(\theta^{K_0}),
\end{align*}
and
\begin{align*}
     V_{\rm pre}(\theta^{K_0+k})\le&  \left(1-2\eta\lambda \right)^{k}V_{\rm pre}(\theta^{K_0})\\
    &\quad +\frac{C_{{\rm grad},V_{\rm pre}}}{2\lambda}\max_{K_0\le j\le K_0+k}\|\nabla\cL(\theta^j)\| + \frac{S_{V_{\rm pre}}}{2-\eta\,\mathsf{L}_\lambda}\cLwd(\theta^{K_0}).
\end{align*}
where $C_{{\rm ah},V_{\ell,\rm var}}, C_{{\rm ah},V_\ell}, C_{{\rm grad},V_{\ell,\rm var}}, C_{{\rm grad},M_{\ell}}, C_{{\rm grad},V_{\ell}}, C_{{\rm grad},V_{\rm pre}}, S_{V_{\ell,\rm var}}, S_{M_{\ell}}, S_{V_{\ell}}, S_{V_{\rm pre}}>0$ are universal constants.
\end{lemma}
\begin{proof}
    Let $w_i=\int_{\Omega_i}\pi(x)dx$. Since the $\Omega_i$ form a partition with respect to the density $\pi$, $\sum_{i=1}^Nw_i=1$. Define
    \begin{align*}
        C_{{\rm grad},V_{\ell,\rm var}}&=\max_{0\le m\le L-1}\sum_{i=1}^N w_i\sup_{\theta\in B_\bzero(R_{\lambda})}\|{\nabla_\theta V_{m,\Omega_i,\rm var}(\theta)}\|,\\
        C_{{\rm grad},M_{\ell}}&=\max_{0\le m\le L-1}\sum_{i=1}^N w_i\sup_{\theta\in B_\bzero(R_{\lambda})}\|{\nabla_\theta M_{m,\Omega_i}(\theta)}\|,\\
        C_{{\rm grad},V_{\ell}}&=\max_{0\le m\le L-1}\sum_{i=1}^N w_i\sup_{\theta\in B_\bzero(R_{\lambda})}\|{\nabla_\theta V_{m,\Omega_i}(\theta)}\|,\\
        C_{{\rm grad},V_{\rm pre}}&=\sum_{i=1}^N w_i\sup_{\theta\in B_\bzero(R_{\lambda})}\|{\nabla_\theta V_{{\rm pre},\Omega_i}(\theta)}\|,
    \end{align*}
    and
    \begin{align*}
        S_{V_{\ell,\rm var}}&=\max_{0\le m\le L-1}\sum_{i=1}^Nw_iS_{V_{m,\Omega_i,\rm var}},&
        S_{M_{\ell}}&=\max_{0\le m\le L-1}\sum_{i=1}^Nw_iS_{M_{m,\Omega_i}},\\
        S_{V_{\ell}}&=\max_{0\le m\le L-1}\sum_{i=1}^Nw_iS_{V_{m,\Omega_i}},&
        S_{V_{\rm pre}}&=\sum_{i=1}^Nw_iS_{V_{{\rm pre},\Omega_i}}.
    \end{align*}
    The case $k=0$ is immediate. For $k\ge1$, applying Lemma~\ref{lem:individual_variance_decay} with $k-1$, multiplying by $w_i$, and summing over $i$ gives
    \begin{align*}
        V_{\ell,\rm var}(\theta^{K_0+k})
        \le& \sum_{i=1}^Nw_i\left(1-2\left(\sum_{r=0}^{\ell}s_{i,r,\rm in}+\ell+2\right)\eta\lambda \right)^kV_{\ell,\Omega_i,\rm var}(\theta^{K_0})\\
        &\quad+\sum_{i=1}^Nw_i\frac{{C_{i,\ell}}+\sum_{r=0}^{\ell-1}\left({C_{i,\ell,r,\rm in}}+{C_{i,\ell,r,\rm out}}\right)}{\sum_{r=0}^{\ell}s_{i,r,\rm in}+\ell+2}\epsilon_{\rm ah}\\
        &\quad+\sum_{i=1}^Nw_i\frac{\sup_{\theta\in B_\bzero(R_{\lambda})}\|{\nabla_\theta V_{\ell,\Omega_i,\rm var}(\theta)}\|}{2\lambda\left(\sum_{r=0}^{\ell}s_{i,r,\rm in}+\ell+2\right)}\max_{K_0\le j\le K_0+k}\|\nabla\cL(\theta^j)\|\\
        &\quad+\frac{S_{V_{\ell,\rm var}}}{2-\eta\,\mathsf{L}_\lambda}\cLwd(\theta^{K_0}).
    \end{align*}
    Since $s_{r,\rm in}\le s_{i,r,\rm in}$ and the contraction factors are nonnegative,
    \begin{align*}
        \left(1-2\left(\sum_{r=0}^{\ell}s_{i,r,\rm in}+\ell+2\right)\eta\lambda \right)^k
        \le \left(1-2\left(\sum_{r=0}^{\ell}s_{r,\rm in}+\ell+2\right)\eta\lambda \right)^k.
    \end{align*}
    Moreover, Lemma~\ref{lem:order_of_constants} implies
    \begin{align*}
        \frac{{C_{i,\ell}}+\sum_{r=0}^{\ell-1}\left({C_{i,\ell,r,\rm in}}+{C_{i,\ell,r,\rm out}}\right)}{\sum_{r=0}^{\ell}s_{i,r,\rm in}+\ell+2}
        \le C_{{\rm ah},V_{\ell,\rm var}}(\ell+2),
    \end{align*}
    and $\sum_{r=0}^{\ell}s_{i,r,\rm in}+\ell+2\ge \ell+2$. Therefore
    \begin{align*}
        \sum_{i=1}^Nw_i\frac{\sup_{\theta\in B_\bzero(R_{\lambda})}\|{\nabla_\theta V_{\ell,\Omega_i,\rm var}(\theta)}\|}{2\lambda\left(\sum_{r=0}^{\ell}s_{i,r,\rm in}+\ell+2\right)}
        \le \frac{C_{{\rm grad},V_{\ell,\rm var}}}{\lambda(\ell+2)}.
    \end{align*}
    This proves the bound for $V_{\ell,\rm var}$.

    The proof for $M_{\ell}$ is the same, except that the approximate homogeneity term in Lemma~\ref{lem:individual_variance_decay} has the additional factor $2$. This gives the bound with $2C_{{\rm ah},V_{\ell,\rm var}}(\ell+2)\epsilon_{\rm ah}$ and the constants $C_{{\rm grad},M_{\ell}}$ and $S_{M_{\ell}}$.
    
    For $V_{\ell}$, the same summation argument uses the combined part of Lemma~\ref{lem:individual_variance_decay}. Lemma~\ref{lem:order_of_constants} gives
    \begin{align*}
        \frac{{C^{\rm comb}_{i,\ell}}+\sum_{r=0}^{\ell-1}\left({C^{\rm comb}_{i,\ell,r,\rm in}}+{C^{\rm comb}_{i,\ell,r,\rm out}}\right)}{\sum_{r=0}^{\ell}s_{i,r,\rm in}+\ell+2}
        \le C_{{\rm ah},V_\ell}(\ell+2),
    \end{align*}
    and hence the factor $2$ in Lemma~\ref{lem:individual_variance_decay} yields $2C_{{\rm ah},V_\ell}(\ell+2)\epsilon_{\rm ah}$.
    
    Finally, multiplying the $V_{{\rm pre},\Omega_i}$ estimate from Lemma~\ref{lem:individual_variance_decay} by $w_i$ and summing over $i$ gives the stated bound for $V_{\rm pre}$ directly.
\end{proof}

\bibliographystyle{plainnat}
\bibliography{ref}

@article{wang2025data,
  title={Data Uniformity Improves Training Efficiency and More, with a Convergence Framework Beyond the NTK Regime},
  author={Wang, Yuqing and Gu, Shangding},
  journal={arXiv preprint arXiv:2506.24120},
  year={2025}
}

@article{laurent2000adaptive,
  title={Adaptive estimation of a quadratic functional by model selection},
  author={Laurent, Beatrice and Massart, Pascal},
  journal={Annals of statistics},
  pages={1302--1338},
  year={2000},
  publisher={JSTOR}
}

@article{kawaguchi2018generalization,
  title={Generalization in machine learning via analytical learning theory},
  author={Kawaguchi, Kenji and Bengio, Yoshua and Verma, Vikas and Kaelbling, Leslie Pack},
  journal={arXiv preprint arXiv:1802.07426},
  year={2018}
}

@inproceedings{
qiao2024stable,
title={Stable Minima Cannot Overfit in Univariate Re{LU} Networks: Generalization by Large Step Sizes},
author={Dan Qiao and Kaiqi Zhang and Esha Singh and Daniel Soudry and Yu-Xiang Wang},
booktitle={The Thirty-eighth Annual Conference on Neural Information Processing Systems},
year={2024},
}

@misc{schliserman2025flatminimageneralizationinsights,
      title={Flat Minima and Generalization: Insights from Stochastic Convex Optimization}, 
      author={Matan Schliserman and Shira Vansover-Hager and Tomer Koren},
      year={2025},
      eprint={2511.03548},
      archivePrefix={arXiv},
      primaryClass={cs.LG},
}

@inproceedings{DBLP:conf/iclr/ForetKMN21,
  author       = {Pierre Foret and
                  Ariel Kleiner and
                  Hossein Mobahi and
                  Behnam Neyshabur},
  title        = {Sharpness-aware Minimization for Efficiently Improving Generalization},
  booktitle    = {9th International Conference on Learning Representations, {ICLR} 2021,
                  Virtual Event, Austria, May 3-7, 2021},
  publisher    = {OpenReview.net},
  year         = {2021},
}

@inproceedings{
Jiang*2020Fantastic,
title={Fantastic Generalization Measures and Where to Find Them},
author={Yiding Jiang and Behnam Neyshabur and Hossein Mobahi and Dilip Krishnan and Samy Bengio},
booktitle={International Conference on Learning Representations},
year={2020},
}

@InProceedings{pmlr-v202-andriushchenko23a,
  title = 	 {A Modern Look at the Relationship between Sharpness and Generalization},
  author =       {Andriushchenko, Maksym and Croce, Francesco and M\"{u}ller, Maximilian and Hein, Matthias and Flammarion, Nicolas},
  booktitle = 	 {Proceedings of the 40th International Conference on Machine Learning},
  pages = 	 {840--902},
  year = 	 {2023},
  volume = 	 {202},
  series = 	 {Proceedings of Machine Learning Research},
  month = 	 {23--29 Jul},
  publisher =    {PMLR},
  }

@inproceedings{
wen2023how,
title={How Sharpness-Aware Minimization Minimizes Sharpness?},
author={Kaiyue Wen and Tengyu Ma and Zhiyuan Li},
booktitle={The Eleventh International Conference on Learning Representations },
year={2023},
}

@InProceedings{pmlr-v70-dinh17b,
  title = 	 {Sharp Minima Can Generalize For Deep Nets},
  author =       {Laurent Dinh and Razvan Pascanu and Samy Bengio and Yoshua Bengio},
  booktitle = 	 {Proceedings of the 34th International Conference on Machine Learning},
  pages = 	 {1019--1028},
  year = 	 {2017},
  volume = 	 {70},
  series = 	 {Proceedings of Machine Learning Research},
  month = 	 {06--11 Aug},
  publisher =    {PMLR},
}

@inproceedings{
keskar2017on,
title={On Large-Batch Training for Deep Learning: Generalization Gap and Sharp Minima},
author={Nitish Shirish Keskar and Dheevatsa Mudigere and Jorge Nocedal and Mikhail Smelyanskiy and Ping Tak Peter Tang},
booktitle={International Conference on Learning Representations},
year={2017},
}

@inproceedings{hochreiter,
author = {Hochreiter, Sepp and Schmidhuber, J\"{u}rgen},
title = {Simplifying neural nets by discovering flat minima},
year = {1994},
publisher = {MIT Press},
address = {Cambridge, MA, USA},
booktitle = {Proceedings of the 8th International Conference on Neural Information Processing Systems},
pages = {529–536},
numpages = {8},
location = {Denver, Colorado},
series = {NIPS'94}
}

@inproceedings{
chiang2023loss,
title={Loss Landscapes are All You Need: Neural Network Generalization Can Be Explained Without the Implicit Bias of Gradient Descent},
author={Ping-yeh Chiang and Renkun Ni and David Yu Miller and Arpit Bansal and Jonas Geiping and Micah Goldblum and Tom Goldstein},
booktitle={The Eleventh International Conference on Learning Representations },
year={2023},
}

@inproceedings{blanc2020implicit,
  title={Implicit regularization for deep neural networks driven by an ornstein-uhlenbeck like process},
  author={Blanc, Guy and Gupta, Neha and Valiant, Gregory and Valiant, Paul},
  booktitle={Conference on Learning Theory},
  pages={483--513},
  year={2020}
}

@inproceedings{du2019gradient,
  title={Gradient descent finds global minima of deep neural networks},
  author={Du, Simon and Lee, Jason and Li, Haochuan and Wang, Liwei and Zhai, Xiyu},
  booktitle={International Conference on Machine Learning},
  pages={1675--1685},
  year={2019}
}

@article{jacot2018neural,
  title={Neural tangent kernel: Convergence and generalization in neural networks},
  author={Jacot, Arthur and Gabriel, Franck and Hongler, Cl{\'e}ment},
  journal={Advances in Neural Information Processing Systems},
  volume={31},
  year={2018}
}

@article{lee2019wide,
  title={Wide neural networks of any depth evolve as linear models under gradient descent},
  author={Lee, Jaehoon and Xiao, Lechao and Schoenholz, Samuel and Bahri, Yasaman and Novak, Roman and Sohl-Dickstein, Jascha and Pennington, Jeffrey},
  journal={Advances in Neural Information Processing Systems},
  volume={32},
  year={2019}
}

@article{bartlett2002rademacher,
  title={Rademacher and Gaussian complexities: Risk bounds and structural results},
  author={Bartlett, Peter L and Mendelson, Shahar},
  journal={Journal of Machine Learning Research},
  volume={3},
  number={Nov},
  pages={463--482},
  year={2002}
}

@inproceedings{neyshabur2015norm,
  title={Norm-based capacity control in neural networks},
  author={Neyshabur, Behnam and Tomioka, Ryota and Srebro, Nathan},
  booktitle={Conference on Learning Theory},
  pages={1376--1401},
  year={2015}
}

@article{bartlett2017spectrally,
  title={Spectrally-normalized margin bounds for neural networks},
  author={Bartlett, Peter L and Foster, Dylan J and Telgarsky, Matus J},
  journal={Advances in Neural Information Processing Systems},
  volume={30},
  year={2017}
}

@article{lyu2023dichotomy,
  title={Dichotomy of early and late phase implicit biases can provably induce grokking},
  author={Lyu, Kaifeng and Jin, Jikai and Li, Zhiyuan and Du, Simon S and Lee, Jason D and Hu, Wei},
  journal={The Twelfth International Conference on Learning Representations},
  year={2024}
}

@inproceedings{mohamadi2024you,
  title={Why Do You Grok? A Theoretical Analysis on Grokking Modular Addition},
  author={Mohamadi, Mohamad Amin and Li, Zhiyuan and Wu, Lei and Sutherland, Danica J},
  booktitle={International Conference on Machine Learning},
  pages={35934--35967},
  year={2024}
}

@article{power2022grokking,
  title={Grokking: Generalization beyond overfitting on small algorithmic datasets},
  author={Power, Alethea and Burda, Yuri and Edwards, Harri and Babuschkin, Igor and Misra, Vedant},
  journal={arXiv preprint arXiv:2201.02177},
  year={2022}
}

@article{xu2023benign,
  title={Benign Overfitting and Grokking in ReLU Networks for XOR Cluster Data},
  author={Xu, Zhiwei and Wang, Yutong and Frei, Spencer and Vardi, Gal and Hu, Wei},
  journal={The Twelfth International Conference on Learning Representations},
  year={2024}
}

@inproceedings{
    kumar2024grokking,
    title={Grokking as the transition from lazy to rich training dynamics},
    author={Tanishq Kumar and Blake Bordelon and Samuel J. Gershman and Cengiz Pehlevan},
    booktitle={The Twelfth International Conference on Learning Representations},
    year={2024}
}

@article{merrill2023tale,
  title={A tale of two circuits: Grokking as competition of sparse and dense subnetworks},
  author={Merrill, William and Tsilivis, Nikolaos and Shukla, Aman},
  journal={arXiv preprint arXiv:2303.11873},
  year={2023}
}

@inproceedings{humayun2024deep,
  title={Deep Networks Always Grok and Here is Why},
  author={Humayun, Ahmed Imtiaz and Balestriero, Randall and Baraniuk, Richard},
  booktitle={International Conference on Machine Learning},
  pages={20722--20745},
  year={2024}
}

@inproceedings{chughtai2023toy,
  title={A toy model of universality: Reverse engineering how networks learn group operations},
  author={Chughtai, Bilal and Chan, Lawrence and Nanda, Neel},
  booktitle={International Conference on Machine Learning},
  pages={6243--6267},
  year={2023}
}

@article{tan2023understanding,
  title={Understanding grokking through a robustness viewpoint},
  author={Tan, Zhiquan and Huang, Weiran},
  journal={arXiv preprint arXiv:2311.06597},
  year={2023}
}

@article{notsawo2023predicting,
  title={Predicting grokking long before it happens: A look into the loss landscape of models which grok},
  author={Notsawo Jr, Pascal and Zhou, Hattie and Pezeshki, Mohammad and Rish, Irina and Dumas, Guillaume and others},
  journal={arXiv preprint arXiv:2306.13253},
  year={2023}
}

@article{belkin2019reconciling,
  title={Reconciling modern machine-learning practice and the classical bias--variance trade-off},
  author={Belkin, Mikhail and Hsu, Daniel and Ma, Siyuan and Mandal, Soumik},
  journal={Proceedings of the National Academy of Sciences},
  volume={116},
  number={32},
  pages={15849--15854},
  year={2019},
  publisher={National Academy of Sciences}
}

@article{boursier2025theoretical,
  title={A Theoretical Framework for Grokking: Interpolation followed by Riemannian Norm Minimisation},
  author={Boursier, Etienne and Pesme, Scott and Dragomir, Radu-Alexandru},
  journal={Advances in Neural Information Processing Systems},
  year={2025}
}

@article{fan2024deep,
  title={Deep grokking: Would deep neural networks generalize better?},
  author={Fan, Simin and Pascanu, Razvan and Jaggi, Martin},
  journal={arXiv preprint arXiv:2405.19454},
  year={2024}
}

@article{zhu2024critical,
  title={Critical data size of language models from a grokking perspective},
  author={Zhu, Xuekai and Fu, Yao and Zhou, Bowen and Lin, Zhouhan},
  journal={arXiv preprint arXiv:2401.10463},
  year={2024}
}

@article{wang2024grokking,
  title={Grokking of implicit reasoning in transformers: A mechanistic journey to the edge of generalization},
  author={Wang, Boshi and Yue, Xiang and Su, Yu and Sun, Huan},
  journal={Advances in Neural Information Processing Systems},
  volume={37},
  pages={95238--95265},
  year={2024}
}

@article{xulet,
  title={Let Me Grok for You: Accelerating Grokking via Embedding Transfer from a Weaker Model},
  author={Xu, Zhiwei and Ni, Zhiyu and Wang, Yixin and Hu, Wei},
  journal={The Thirteenth International Conference on Learning Representations},
  year={2025}
}

@inproceedings{cui2019class,
  title={Class-balanced loss based on effective number of samples},
  author={Cui, Yin and Jia, Menglin and Lin, Tsung-Yi and Song, Yang and Belongie, Serge},
  booktitle={Proceedings of the IEEE/CVF conference on computer vision and pattern recognition},
  pages={9268--9277},
  year={2019}
}

@article{cao2019learning,
  title={Learning imbalanced datasets with label-distribution-aware margin loss},
  author={Cao, Kaidi and Wei, Colin and Gaidon, Adrien and Arechiga, Nikos and Ma, Tengyu},
  journal={Advances in neural information processing systems},
  volume={32},
  year={2019}
}

@article{menon2020long,
  title={Long-tail learning via logit adjustment},
  author={Menon, Aditya Krishna and Jayasumana, Sadeep and Rawat, Ankit Singh and Jain, Himanshu and Veit, Andreas and Kumar, Sanjiv},
  journal={arXiv preprint arXiv:2007.07314},
  year={2020}
}

@article{li2021autobalance,
  title={Autobalance: Optimized loss functions for imbalanced data},
  author={Li, Mingchen and Zhang, Xuechen and Thrampoulidis, Christos and Chen, Jiasi and Oymak, Samet},
  journal={Advances in Neural Information Processing Systems},
  volume={34},
  pages={3163--3177},
  year={2021}
}

@inproceedings{byrd2019effect,
  title={What is the effect of importance weighting in deep learning?},
  author={Byrd, Jonathon and Lipton, Zachary},
  booktitle={International conference on machine learning},
  pages={872--881},
  year={2019},
  organization={PMLR}
}

@inproceedings{sagawa2020investigation,
  title={An investigation of why overparameterization exacerbates spurious correlations},
  author={Sagawa, Shiori and Raghunathan, Aditi and Koh, Pang Wei and Liang, Percy},
  booktitle={International Conference on Machine Learning},
  pages={8346--8356},
  year={2020},
  organization={PMLR}
}

@article{kini2021label,
  title={Label-imbalanced and group-sensitive classification under overparameterization},
  author={Kini, Ganesh Ramachandra and Paraskevas, Orestis and Oymak, Samet and Thrampoulidis, Christos},
  journal={Advances in Neural Information Processing Systems},
  volume={34},
  pages={18970--18983},
  year={2021}
}

@article{xu2021understanding,
  title={Understanding the role of importance weighting for deep learning},
  author={Xu, Da and Ye, Yuting and Ruan, Chuanwei},
  journal={arXiv preprint arXiv:2103.15209},
  year={2021}
}

@inproceedings{lin2017focal,
  title={Focal loss for dense object detection},
  author={Lin, Tsung-Yi and Goyal, Priya and Girshick, Ross and He, Kaiming and Doll{\'a}r, Piotr},
  booktitle={Proceedings of the IEEE international conference on computer vision},
  pages={2980--2988},
  year={2017}
}

@inproceedings{park2021influence,
  title={Influence-balanced loss for imbalanced visual classification},
  author={Park, Seulki and Lim, Jongin and Jeon, Younghan and Choi, Jin Young},
  booktitle={Proceedings of the IEEE/CVF international conference on computer vision},
  pages={735--744},
  year={2021}
}

@article{cimpean2011hyers,
  title={Hyers--Ulam stability of Euler’s equation},
  author={Cimpean, Dalia Sabina and Popa, Dorian},
  journal={Applied Mathematics Letters},
  volume={24},
  number={9},
  pages={1539--1543},
  year={2011},
  publisher={Elsevier}
}

@misc{wang2026convergencegradientdescentgeneral,
      title={Convergence of Gradient Descent for General Neural Network Architectures Beyond the NTK Regime}, 
      author={Yuqing Wang},
      year={2026},
      eprint={2606.23364},
      archivePrefix={arXiv},
      primaryClass={cs.LG},
      url={https://arxiv.org/abs/2606.23364}, 
}

@inproceedings{
chen2023which,
title={Which Layer is Learning Faster? A Systematic Exploration of Layer-wise Convergence Rate for Deep Neural Networks},
author={Yixiong Chen and Alan Yuille and Zongwei Zhou},
booktitle={The Eleventh International Conference on Learning Representations },
year={2023},
url={https://openreview.net/forum?id=wlMDF1jQF86}
}

@article{raghu2017svcca,
  title={Svcca: Singular vector canonical correlation analysis for deep learning dynamics and interpretability},
  author={Raghu, Maithra and Gilmer, Justin and Yosinski, Jason and Sohl-Dickstein, Jascha},
  journal={Advances in neural information processing systems},
  volume={30},
  year={2017}
}

@article{xu2026grok,
  title={To Grok Grokking: Provable Grokking in Ridge Regression},
  author={Xu, Mingyue and Vardi, Gal and Safran, Itay},
  journal={arXiv preprint arXiv:2601.19791},
  year={2026}
}

@inproceedings{arora2019fine,
  author    = {Arora, Sanjeev and Du, Simon S. and Hu, Wei and Li, Zhiyuan and Wang, Ruosong},
  title     = {{Fine-Grained} {Analysis} of {Optimization} and {Generalization} for {Overparameterized} {Two-Layer} {Neural} {Networks}},
  booktitle = {International Conference on Machine Learning},
  pages     = {322--332},
  year      = {2019}
}

@article{mcallester1999pac,
  author  = {McAllester, David A.},
  title   = {Some {PAC-Bayesian} {Theorems}},
  journal = {Machine Learning},
  volume  = {37},
  number  = {3},
  pages   = {355--363},
  year    = {1999}
}

@inproceedings{dziugaite2017computing,
  author    = {Dziugaite, Gintare Karolina and Roy, Daniel M.},
  title     = {{Computing} {Nonvacuous} {Generalization} {Bounds} for {Deep} ({Stochastic}) {Neural} {Networks} with {Many} {More} {Parameters} than {Training} {Data}},
  booktitle = {Uncertainty in Artificial Intelligence},
  year      = {2017}
}

@article{bartlett1998sample,
  author  = {Bartlett, Peter L.},
  title   = {The {Sample} {Complexity} of {Pattern} {Classification} with {Neural} {Networks}: {The} {Size} of the {Weights} is {More} {Important} than the {Size} of the {Network}},
  journal = {IEEE Transactions on Information Theory},
  volume  = {44},
  number  = {2},
  pages   = {525--536},
  year    = {1998}
}

@article{bartlett2020benign,
  author  = {Bartlett, Peter L. and Long, Philip M. and Lugosi, G{\'a}bor and Tsigler, Alexander},
  title   = {{Benign} {Overfitting} in {Linear} {Regression}},
  journal = {Proceedings of the National Academy of Sciences},
  volume  = {117},
  number  = {48},
  pages   = {30063--30070},
  year    = {2020}
}

@article{bousquet2002stability,
  author  = {Bousquet, Olivier and Elisseeff, Andr{\'e}},
  title   = {{Stability} and {Generalization}},
  journal = {Journal of Machine Learning Research},
  volume  = {2},
  pages   = {499--526},
  year    = {2002}
}

@article{dudley1967sizes,
  author  = {Dudley, Richard M.},
  title   = {The {Sizes} of {Compact} {Subsets} of {Hilbert} {Space} and {Continuity} of {Gaussian} {Processes}},
  journal = {Journal of Functional Analysis},
  volume  = {1},
  number  = {3},
  pages   = {290--330},
  year    = {1967}
}

@inproceedings{hardt2016train,
  author    = {Hardt, Moritz and Recht, Benjamin and Singer, Yoram},
  title     = {{Train} {Faster}, {Generalize} {Better}: {Stability} of {Stochastic} {Gradient} {Descent}},
  booktitle = {International Conference on Machine Learning},
  pages     = {1225--1234},
  year      = {2016}
}

@article{tsigler2023benign,
  author  = {Tsigler, Alexander and Bartlett, Peter L.},
  title   = {{Benign} {Overfitting} in {Ridge} {Regression}},
  journal = {Journal of Machine Learning Research},
  volume  = {24},
  number  = {123},
  pages   = {1--76},
  year    = {2023}
}

@article{vapnik1971uniform,
  author  = {Vapnik, Vladimir N. and Chervonenkis, Alexey Ya.},
  title   = {On the {Uniform} {Convergence} of {Relative} {Frequencies} of {Events} to {Their} {Probabilities}},
  journal = {Theory of Probability and Its Applications},
  volume  = {16},
  number  = {2},
  pages   = {264--280},
  year    = {1971}
}

@inproceedings{zhang2017understanding,
  author    = {Zhang, Chiyuan and Bengio, Samy and Hardt, Moritz and Recht, Benjamin and Vinyals, Oriol},
  title     = {{Understanding} {Deep} {Learning} {Requires} {Rethinking} {Generalization}},
  booktitle = {International Conference on Learning Representations},
  year      = {2017}
}

@article{hastie2022surprises,
  author  = {Hastie, Trevor and Montanari, Andrea and Rosset, Saharon and Tibshirani, Ryan J.},
  title   = {{Surprises} in {High-Dimensional} {Ridgeless} {Least} {Squares} {Interpolation}},
  journal = {Annals of Statistics},
  volume  = {50},
  number  = {2},
  pages   = {949--986},
  year    = {2022}
}

@inproceedings{nagarajan2019uniform,
  author    = {Nagarajan, Vaishnavh and Kolter, J. Zico},
  title     = {{Uniform} {Convergence} {May} {Be} {Unable} to {Explain} {Generalization} in {Deep} {Learning}},
  booktitle = {Advances in Neural Information Processing Systems},
  year      = {2019}
}

@inproceedings{nakkiran2020deep,
  author    = {Nakkiran, Preetum and Kaplun, Gal and Bansal, Yamini and Yang, Tristan and Barak, Boaz and Sutskever, Ilya},
  title     = {{Deep} {Double} {Descent}: {Where} {Bigger} {Models} and {More} {Data} {Hurt}},
  booktitle = {International Conference on Learning Representations},
  year      = {2020}
}

@inproceedings{rahimi2008uniform,
  title={Uniform approximation of functions with random bases},
  author={Rahimi, Ali and Recht, Benjamin},
  booktitle={2008 46th annual allerton conference on communication, control, and computing},
  pages={555--561},
  year={2008},
  organization={IEEE}
}

@article{ba2022high,
  title={High-dimensional asymptotics of feature learning: How one gradient step improves the representation},
  author={Ba, Jimmy and Erdogdu, Murat A and Suzuki, Taiji and Wang, Zhichao and Wu, Denny and Yang, Greg},
  journal={Advances in Neural Information Processing Systems},
  volume={35},
  pages={37932--37946},
  year={2022}
}

@article{arous2021online,
  title={Online stochastic gradient descent on non-convex losses from high-dimensional inference},
  author={Arous, Gerard Ben and Gheissari, Reza and Jagannath, Aukosh},
  journal={Journal of Machine Learning Research},
  volume={22},
  number={106},
  pages={1--51},
  year={2021}
}

@article{damian2024computational,
  title={Computational-statistical gaps in gaussian single-index models},
  author={Damian, Alex and Pillaud-Vivien, Loucas and Lee, Jason D and Bruna, Joan},
  journal={arXiv preprint arXiv:2403.05529},
  year={2024}
}

@article{defilippis2026noise,
  title={A Noise Sensitivity Exponent Controls Large Statistical-to-Computational Gaps in Single-and Multi-Index Models},
  author={Defilippis, Leonardo and Krzakala, Florent and Loureiro, Bruno and Maillard, Antoine},
  journal={arXiv preprint arXiv:2603.17896},
  year={2026}
}

@article{Bou02, author = {Bousquet, Olivier and Elisseeff, Andre}, title = {Stability and Generalization}, journal = {Journal of Machine Learning Research}, year = {2002} }

@inproceedings{Gol17, author = {Golowich, Noah and Rakhlin, Alexander and Shamir, Ohad}, title = {Size-Independent Sample Complexity of Neural Networks}, booktitle = {Conference on Learning Theory}, year = {2017} }

@article{Ney17b, author = {Neyshabur, Behnam and others}, title = {A PAC-Bayesian Approach to Spectrally-Normalized Margin Bounds for Neural Networks}, journal = {arXiv preprint arXiv}, year = {2017} }

@inproceedings{Aro18, author = {Arora, Sanjeev and others}, title = {Stronger Generalization Bounds for Deep Nets via a Compression Approach}, booktitle = {International Conference on Machine Learning}, year = {2018} }

@article{Per20, author = {P{'e}rez, Guillermo Valle and Louis, Annie}, title = {Generalization Bounds for Deep Learning}, journal = {arXiv preprint arXiv}, year = {2020} }

@article{Bie19, author = {Bietti, Alberto and Mairal, Julien}, title = {On the Inductive Bias of Neural Tangent Kernels}, journal = {arXiv preprint arXiv}, year = {2019} }

@inproceedings{Bor20, author = {Bordelon, Blake and Canatar, Abdulkadir and Pehlevan, Cengiz}, title = {Spectrum Dependent Learning Curves in Kernel Regression and Wide Neural Networks}, booktitle = {International Conference on Machine Learning}, year = {2020} }

@inproceedings{Cao19b, author = {Cao, Yuan and Gu, Quanquan}, title = {Generalization Bounds of Stochastic Gradient Descent for Wide and Deep Neural Networks}, booktitle = {Advances in Neural Information Processing Systems}, year = {2019} }

@article{Mei18, author = {Mei, Song and Montanari, Andrea and Nguyen, Phan-Minh}, title = {A Mean Field View of the Landscape of Two-Layer Neural Networks}, journal = {Proceedings of the National Academy of Sciences}, year = {2018} }

@article{Sir18, author = {Sirignano, Justin A. and Spiliopoulos, Konstantinos}, title = {Mean Field Analysis of Neural Networks: A Law of Large Numbers}, journal = {SIAM Journal on Applied Mathematics}, year = {2018} }

@inproceedings{Chi18, author = {Chizat, L{'e}na{"\i}c and Bach, Francis}, title = {On the Global Convergence of Gradient Descent for Over-parameterized Models using Optimal Transport}, booktitle = {Advances in Neural Information Processing Systems}, year = {2018} }

@inproceedings{Mei19b, author = {Mei, Song and Misiakiewicz, Theodor and Montanari, Andrea}, title = {Mean-field Theory of Two-layers Neural Networks: Dimension-free Bounds and Kernel Limit}, booktitle = {Conference on Learning Theory}, year = {2019} }

@article{Mei19, author = {Mei, Song and Montanari, Andrea}, title = {The Generalization Error of Random Features Regression: Precise Asymptotics and the Double Descent Curve}, journal = {Communications on Pure and Applied Mathematics}, year = {2019} }

@inproceedings{Tis15, author = {Tishby, Naftali and Zaslavsky, Noga}, title = {Deep Learning and the Information Bottleneck Principle}, booktitle = {IEEE Information Theory Workshop}, year = {2015} }

@article{Shw17, author = {Shwartz-Ziv, Ravid and Tishby, Naftali}, title = {Opening the Black Box of Deep Neural Networks via Information}, journal = {arXiv preprint arXiv}, year = {2017} }

@article{Xu17, author = {Xu, Aolin and Raginsky, Maxim}, title = {Information-theoretic Analysis of Generalization Capability of Learning Algorithms}, journal = {arXiv preprint arXiv}, year = {2017} }

\end{document}